\documentclass[oneside,11pt]{article}

\usepackage[T1]{fontenc}
\usepackage[utf8]{inputenc}
\usepackage{amsmath,amssymb,amsthm}
\usepackage{graphicx}
\usepackage{tikz}
\usepackage{psfrag}
\usepackage{xcolor}
\usepackage{stmaryrd}
\usepackage{booktabs}
\usepackage[margin=1in]{geometry}
\usepackage[
  scientific-notation=true,
  separate-uncertainty=true,
  bracket-ambiguous-numbers=true,
  propagate-math-font=true
]{siunitx}
\usepackage{varwidth}
\usepackage[round,authoryear]{natbib}
\usepackage{placeins}
\usepackage{enumerate}
\usepackage{hyperref}
\usepackage[capitalize]{cleveref}

\newtheorem{thm}{Theorem}
\newtheorem{prop}{Proposition}
\newtheorem{assumption}{Assumption}
\newtheorem{lem}[prop]{Lemma}
\newtheorem{coro}[prop]{Corollary}
\theoremstyle{definition}
\newtheorem{df}{Definition}
\newtheorem{rem}{Remark}
\newtheorem*{rem*}{Remark}
\newtheorem{ex}{Example}
\theoremstyle{remark}

\providecommand{\noopsort}[1]{}

\newcommand{\cH}{\mathcal{H}}
\newcommand{\cX}{\mathcal{X}}

\newcommand{\cL}{\mathcal{L}}
\newcommand{\cD}{\mathcal{D}}
\newcommand{\cE}{\mathcal{E}}
\newcommand{\cR}{\mathcal{R}}
\newcommand{\cF}{\mathcal{F}}
\newcommand{\cG}{\mathcal{G}}
\newcommand{\cC}{\mathcal{C}}
\newcommand{\cA}{\mathcal{A}}

\newcommand{\NN}{\mathbb{N}}
\newcommand{\ZZ}{\mathbb{Z}}
\newcommand{\PP}{\mathbb{P}}  
\newcommand{\RR}{\mathbb{R}}
\newcommand{\EE}{\mathbb{E}}

\newcommand{\Kbf}{\mathbf{K}}
\newcommand{\Abf}{\mathbf{A}}
\newcommand{\Bbf}{\mathbf{B}}
\newcommand{\Cbf}{\mathbf{C}}
\newcommand{\Jbf}{\mathbf{J}}
\newcommand{\Ybf}{\mathbf{Y}}

\newcommand{\Ah}{\widehat A}
\newcommand{\uhl}{\widehat u_\lambda}
\newcommand{\ul}{u_\lambda}
\newcommand{\uh}{\widehat u}
\newcommand{\Sigmah}{\widehat \Sigma}
\newcommand{\Int}[1]{\mathrm{Int} (#1)}
\newcommand{\Input}{\cX}

\newcommand{\KD}{K^\cD}
\newcommand{\Hper}{H_{\mathrm{per}}}

\newcommand{\lb}{\llbracket}
\newcommand{\rb}{\rrbracket}

\newcommand{\function}[5]{
 #1: \begin{array}{rcl}
	  #2 & \longrightarrow & #3 \\
     #4 & \longmapsto & #5 \end{array}
    }

\DeclareMathOperator*{\argmin}{arg\,min}
\DeclareMathOperator{\ran}{ran}
\DeclareMathOperator{\Tr}{T}
\DeclareMathOperator{\op}{op}
\DeclareMathOperator{\Span}{span}
\DeclareMathOperator{\Ker}{Ker}

\DeclareMathOperator{\supp}{supp}
\DeclareMathOperator{\Unif}{Unif}

\makeatletter
\renewcommand*\env@matrix[1][\arraystretch]{%
  \edef\arraystretch{#1}%
  \hskip -\arraycolsep
  \let\@ifnextchar\new@ifnextchar
  \array{*\c@MaxMatrixCols c}}
\makeatother

\crefname{assumption}{assumption}{assumptions}

\title{PIKS: Universal Physics-Informed Kernel Methods}

\author{%
  Joachim Bona-Pellissier$^{1}$, Giacomo Meanti$^{1}$,\\
  Matteo Santacesaria$^{2}$, Lorenzo Rosasco$^{1,3}$\\[0.6em]
  \small $^1$MaLGa Center, DIBRIS, Università degli Studi di Genova,
    Genoa, Italy\\
  \small $^2$MaLGa Center, DIMA, Università degli Studi di Genova,
    Genoa, Italy\\
  \small $^3$Istituto Italiano di Tecnologia, Genoa, Italy\\[0.4em]
  \small \texttt{joachim.bona@edu.unige.it},
    \texttt{giacomo.meanti@edu.unige.it}\\
  \small \texttt{matteo.santacesaria@unige.it},
    \texttt{lorenzo.rosasco@unige.it}%
}
\date{}

\begin{document}

\maketitle

\begin{abstract}
Physics-informed machine learning incorporates physical principles—often expressed via differential operators—into data-driven models. While physics-informed neural networks (PINNs) dominate empirical applications, the complexity of neural network architectures and optimization landscapes hinders the development of a corresponding learning theory. In turn, kernel methods offer an appealing alternative with closed-form solutions and analytical tractability, yet existing guarantees primarily cover the well-specified setting where the target belongs to the native Reproducing Kernel Hilbert Space (RKHS). This imposes unrealistic regularity assumptions that physical targets often fail to satisfy. In this paper, we introduce and analyze Physics-Informed Kernel methodS (PIKS). We establish the universal consistency of PIKS for linear differential constraints, proving that for universal kernels (such as Gaussian or Matérn), the estimator asymptotically learns the target while satisfying physical constraints. We further derive finite-sample bounds under suitable source conditions. Our analysis is based on extending classical operator-theoretic analysis of kernel methods to  physics-informed machine learning. Numerical experiments demonstrate that PIKS can be competitive with PINNs and traditional finite element methods.
\end{abstract}

\medskip
\noindent\textbf{Keywords:} physics-informed machine learning, kernel methods,
universal learning, statistical learning theory, scientific machine learning

\section{Introduction} \label{Introduction-sec}

Physics-informed machine learning (PIML) combines data-driven statistical learning with additional knowledge about a problem's physical properties. 
Indeed, in many applications governed by physical principles, the target function is constrained by relations involving its derivatives or, more generally, the action of differential operators~\citep{cuomo2022scientific}. 
These relations provide further information that can be incorporated into the learning problem as functional constraints, such as prescribed gradients, partial differential equation (PDE) residuals, or conservation laws~\citep{raissi2019physics, karniadakis2021physics, rackauckas2020universal}.

There can be several benefits in leveraging structural knowledge for a learning task. 
When the amount of labeled data is limited, for instance because measurements are challenging to collect, physical knowledge about the target function can help circumvent data scarcity.
Furthermore, simply approximating the target (as in classical regression) may not be enough in certain settings where the physical consistency of a solution can be as important as its accuracy.
For example even a small violation of the conservation of energy in the force-field function may crash a molecular dynamics simulation~\citep{fu2023forces}.  
Augmenting the learning problem with structural constraints that the unknown target function must satisfy has been shown to solve this challenge effectively, dramatically improving sample efficiency, stability, and out-of-distribution robustness~\citep{cuomo2022scientific,quarteroni25review}.

Starting with the introduction of physics-informed neural networks (PINNs)~\citep{raissi2019physics}, the use of machine learning --- in particular of neural networks --- for tackling forward and inverse physics-informed problems has surged in popularity. Engineering and applied aspects of the problem have taken center stage: improved neural network architectures and new ways of framing the problems have resulted in numerous success stories across domains \citep[see e.g.][among many examples]{tocano25pikans,zhao2023pinnsformer}.
As is common with deep learning, however, theoretical understanding has lagged behind empirical progress.
In this paper we propose to tackle a basic yet fundamental question concerning the \emph{universal consistency} of learning methods for linear physics-informed problems. Given a model which is fixed a-priori, can we guarantee that for any target function, the model will learn the target as the data grows to infinity?
This question has been studied in standard supervised learning for many families of models including neural networks and kernel methods, and relies on the \emph{universality} of the corresponding function class, i.e. when a hypothesis class is rich enough to approximate any function, typically in a $L^2$ or $L^\infty$ sense. However, as we will show, standard notions of universality fall short in the physics-informed setting, as the additional structure requires a more careful analysis.

In this paper, we study the universal consistency of kernel methods for physics-informed machine learning. Kernel methods provide a tractable framework for obtaining learning guarantees \citep[see e.g.][]{smale2007learning,caponnetto2007optimal,steinwart2008support,blanchard2018optimal}.
They are also naturally suited to incorporating linear functional constraints in the learning objective, a property which has been used at least since the 1970s~\citep{kimeldorf1971some} and has been applied to Hermite-Birkhoff interpolation \citep{kimeldorf1971some}, meshless methods for PDEs \citep{fasshauer1996solving,wendland2004scattered}, or self-supervised learning with manifold constraints \citep{belkin2006manifold} among others.
We review these works extensively in \Cref{sec:related-settings}, but to the best of our knowledge, certain fundamental properties such as universal consistency have not yet been proven in the physics-informed case. Indeed, the error analyses in the literature usually hold in the \emph{well-specified} setting, in which the target function $u^*$ belongs to the same reproducing kernel Hilbert space (RKHS) $\cH$ from which the estimator is chosen. 
Instead, universal consistency is a property that is relevant in the \emph{misspecified} setting, i.e.~the case $u^*\notin \cH$. This setting is important in practice: for example a Sobolev space $H^s(\Omega)$ is an RKHS only when $s>d/2$ but the target function for real-world problems in high dimensions may not be smooth enough (i.e., it may belong to a Sobolev space with exponent $s \leq d/2$).
Analyses in the misspecified setting allow us to obtain learning guarantees even when the target is not smooth enough to belong to any RKHS, or when perfect knowledge about the physical process is not available, leading to an inaccurate choice of kernel (and thus of space $\cH$).

We focus on a specific kernel-based algorithm for physics-informed machine learning which we refer to as PIKS (for \emph{Physics Informed Kernel methodS}). In \Cref{sec:prob} we introduce the physics-informed learning problem, which amounts to regression with an additional structural constraint.
Like in PINNs, the constraint is enforced by adding a physics-informed term to the training loss.
Throughout the paper we emphasize that our general formalism can apply to diverse physics-informed settings.
In particular, we show in \Cref{sec:PDE-setting} that it covers the setting of pure PDE solving, where measurements of the target values are only available on the boundary of the domain (while the PDE provides information inside the domain). 
But in general, measurements of the target can be available anywhere in the domain, as PIKS covers any problem of learning with linear constraints. 
In \Cref{sec:PIKS}, we precisely formalize the PIKS algorithm and in \Cref{sec:theory} proceed with its theoretical analysis. In the latter, after listing the working assumptions, we state our main result, \Cref{main-theorem}, which shows consistency of PIKS. While \Cref{main-theorem} is formulated for general linear operators, in the following subsection, we apply this result in the context of differential operators on Sobolev spaces, which is very common in applications. We then focus more specifically on the setting of PDE solving (already described in \Cref{sec:PDE-setting}), where measurement data is only available on the boundary of the domain. Finally, we show that under additional assumptions, we can go further than the consistency, and obtain finite-sample bounds, as stated in \Cref{thm:finite-sample-rates-main}. 
In \Cref{sec:experiments} we study the empirical performance of PIKS in different settings, and show that it compares well to other methods in the literature. In summary, the main contributions of this work are:

\begin{enumerate}
    \item We establish a set of general assumptions under which the PIKS estimator is a \emph{universal learner}: it is able to both learn the target and approximate the physical constraint asymptotically, even in the \emph{misspecified setting} in which the target $u^*$ does not belong to the native RKHS of the estimator. We analyze in more detail the typical case of differential operators on Sobolev spaces.
    \item We make this general result concrete in two specific settings: one in which target measurements are available on the domain boundary (typical of PDE problems) and one in which they are available inside the domain. In the case of elliptic differential operators we show how to use the PDE's regularity with our main theorem to obtain stronger convergence results. 
    \item We provide convergence rates under stronger assumptions on the target, which take the form of a source condition defined as powers of an integral operator. We illustrate the rates with an example of the Laplacian on periodic functions. 
    \item We demonstrate empirically that PIKS is competitive with recent kernel-based methods as well as physics-informed neural approaches on representative PDE solving tasks. We further study in practice the impact of model misspecification on the performance of the algorithm. When the target function is rougher than the base kernel, the convergence is slower, but remains competitive with classical FEM solvers. 
\end{enumerate}

\section{Physics-informed statistical learning}\label{sec:prob}

We consider learning problems with an additional constraint in the form of a linear operator which encodes the problem's physics.
More precisely, let $u^*\in\cF$ where $\cF$ is a space of maps from $\cX \subseteq \RR^d$ to $\RR$. Let $\cG$ be another space of maps from $\cX$ to $\RR$ and let $\cD : \cF \to \cG$ be a linear operator.
Consider input data random variables $X\sim \rho_X$ and $Z\sim\rho_Z$ in $\cX$ as well as noise random variables $\epsilon$ and $\eta$ (centered and independent from $X$ and $Z$ respectively), and let
\begin{equation}\label{eq:data-generation}
    Y = u^*(X) + \epsilon, \qquad  W = \cD u^* (Z) + \eta.
\end{equation}
Two datasets $(x_i, y_i)_{i=1}^n$ and $(z_j, w_j)_{j=1}^m$ can be obtained by sampling i.i.d.~copies of $(X, Y)$ and of $(Z, W)$ respectively.
The objective is to minimize the expected risk
\begin{equation}\label{eq:exp-rsk}
    \min_{u\in\cF} \cR(u), \quad \cR(u) = 
    \mathbb{E}\big[(u(X) - Y)^2\big] + \mathbb{E}\big[(\cD u(Z) - W)^2\big]
\end{equation}
using the $n + m$ training samples. Note that the target function $u^*$ is a minimizer of $\cR$, and the problem can equivalently be viewed as learning $u^*$ satisfying \eqref{eq:data-generation}. We will see later how this point of view is useful when considering PDEs.

The PIML approach approximates the expectations in~\eqref{eq:exp-rsk} with empirical estimates and restricts the hypothesis space to a smaller $\cH\subset \cF$ resulting in the following empirical risk minimization problem
\begin{equation}\label{eq:erm}
    \min_{u\in\cH} \widehat{\cR}(u), \quad \widehat{\cR}(u) = \frac{1}{n}\sum_{i=1}^n (u(x_i) - y_i)^2 + \frac{1}{m}\sum_{j=1}^m (\cD u(z_j) - w_j)^2.
\end{equation}

In the following, we focus in particular on the setting where $\cD$ is a known linear differential operator (e.g., divergence, Laplacian). The function space $\cF$ is assumed to be a Sobolev space $H^s(\cX)$ with $s>0$ such that $\cD$ is well defined from $H^s(\cX)$ to $L^2(\cX)$. 
Note that the probability distributions of $X$ and $Z$ may not have the same support and, more generally, we need not observe $u^*$ and $\cD u^*$ at the same points. This allows us to connect the framework directly to PDE solving.

\subsection{Solving Partial Differential Equations with Machine Learning}\label{sec:PDE-setting}
PDEs are among the most important problem classes addressed by physics-informed learning, with applications ranging from climate modeling \citep{kashinath2021physics} to the cardiovascular system \citep{kissas2020machine} to permanent magnets \citep{kovacs2022conditional}. In particular, boundary value problems are ubiquitous in applied mathematics: they combine differential equations governing the interior of a domain with conditions imposed on its boundary.
This setting fits in our framework by allowing the laws of $X$ and $Z$ to be concentrated on different subsets: $X$ is sampled from the boundary and $Z$ is sampled from the interior.
More precisely, consider the goal of estimating the solution $u^*$ of a boundary value problem on a bounded Lipschitz domain $\Omega \subset \RR^d$
\begin{equation}\label{eq:PDE}
    \begin{cases}
    \cD u^*(x) = q(x) & x \in \Omega \\
    u^*(x) = h(x) & x \in \partial \Omega,
    \end{cases}
\end{equation}
where we have access to $\cD, q$ and $h$.
Then, let \(\cX=\overline{\Omega}\) and consider 
\(\rho_X,\rho_Z\) such that $\rho_X(\partial\Omega)=1$ and $\rho_Z(\Omega)=1$.
Thus, samples of \(X\) are boundary points whereas samples of \(Z\) are interior
collocation points. 
We can then proceed as in the previous section, possibly considering the noiseless case ($\epsilon, \eta=0$) if we have perfect knowledge of $q$ and $h$.
In this view, minimizing~\eqref{eq:erm} can be interpreted as numerically solving the PDE~\eqref{eq:PDE} using a random discretization.
This choice of sampling is standard in machine learning theory and can be contrasted to classic deterministic discretizations in PDE,  such as 
meshes or collocation points \citep[see e.g.][and references therein]{wendland2004scattered}.

A simple example of a boundary-value problem is the Poisson equation on a domain $\Omega$:
\begin{equation}
    \begin{cases}
        \Delta u(x) = q(x) & x\in\Omega \\
        u(x) = h(x) & x \in \partial\Omega.
    \end{cases}
\end{equation}
If $\Omega$, $q$ and $h$ are regular enough, we can consider $\cF = H^2(\Omega)$ and take $\cD=\Delta$, where $\Delta : H^2(\Omega)\to L^2(\Omega)$ is the Laplace operator \citep{evans2010partial}.

Before describing our approach, we provide an in-depth overview of related problems and results in the literature.

\subsection{Related Settings}\label{sec:related-settings}

Here we give an overview of the settings which are related to the one introduced above and have appeared in different contexts.

\paragraph{Scientific Machine Learning and PINNs.} 
Machine learning methods for solving PDEs have attracted a lot of interest under the umbrella terms physics-informed machine learning or scientific machine learning~\citep{rackauckas2020universal}, especially using methods such as physics-informed neural networks (PINNs)~\citep{raissi2017physics} and Sobolev training~\citep{czarnecki2017sobolev}.
PINNs have become a common approach for both forward and inverse PDE problems from fluid dynamics~\citep{cai2021physics}, to geophysics~\citep{rasht2022physics}, and medical sciences \citep{sahli2020physics}.
In our setting, PINNs correspond to considering~\cref{eq:erm} while  choosing neural networks as the hypothesis space $\cH$.
Recent work has made significant progress in establishing learning-theoretic guarantees and error estimates for PINNs \citep{shin2020convergence,mishra2023estimates,de2024error,zeinhofer2025unified,doumeche25convergence}. 
However, fully characterizing the behavior of neural network based estimators remains challenging due to the highly non-convex nature of the underlying optimization problem.
Such non-convexity is non trivial and can lead to failures in practice, as has been well observed in the literature \citep{krishnapriyan2021characterizing,wang2021understanding,rathore2024challenges}. 

\paragraph{Operator Learning.} A growing body of work focuses on operator learning, where the objective is to approximate infinite-dimensional mappings between function spaces (e.g., mapping parametric PDE coefficients). Examples include Deep Operator Networks (DeepONets)~\citep{lu2021learning} and Fourier Neural Operators (FNOs)~\citep{li2021fourier}. However, these operator-based methods require massive offline datasets of pre-computed, high-fidelity PDE solutions to learn the underlying physical mapping. In contrast, PIKS addresses the single-instance problem (analogously to standard PINNs) where the goal is to infer the solution of a specific PDE using the governing equations, boundary conditions, and sparse empirical measurements, and thus requires only small datasets.

\paragraph{Gaussian processes and kernel methods.}
In parallel with the development of PINNs, Gaussian Process (GP) based methods for 
scientific machine learning have also been proposed \citep{owhadi2015bayesian,raissi2017gps,raissi2018numerical}. 
These approaches estimate solutions, and potentially their uncertainty, by minimizing the marginal log-likelihood (instead of the empirical risk~\cref{eq:erm}).
Some extensions to non-linear PDEs have also  been proposed using the Gauss-Newton algorithm~\citep{chen2021solving} although convergence can only be locally guaranteed~\citep{batlle2025error}. 
Recently \citet{baptista_solving_2025} used the same algorithm to solve PDEs in the weak form with a non-smooth forcing term. \citet{doumeche2024physics} showed that empirical risk minimization with linear differential constraints is equivalent to defining a new, physics-informed kernel and then performing kernel ridge regression, yielding theoretical evidence that physical constraints can improve convergence rates, characterized through the effective dimension of the new kernel. Such a kernel is defined by a continuous constraint (instead of pointwise evaluations as is the case in the present work) and is a priori not available in closed form, which is why subsequent works focus on practical approximations~\citep{doumeche2025physics} and fast implementations~\citep{doumeche2025fast}.

\paragraph{Hermite-Birkhoff interpolation.}
Despite the recent revival under the hat of \emph{physics-informed learning}, 
the problem of learning a function from its values and its derivatives can be traced back to \citet{hermite1878formule} and \citet{birkhoff1906general}. 
It was formulated as a splines problem in the 1D case by \citet{kimeldorf1971some}, who considered the problem $\min_{u\in\cH} \sum_{j=1}^m(\cL_j u - w_j)^2$ with $\cL_j$ a linear functional defined on an RKHS (thus including the setting of differential operators $\cL_j u = \cD u(z_j)$). 
It was also considered in a multivariate setting using Radial Basis Functions (RBFs) \citep{zongmin1992hermite}.
More recently \citet{shi2010hermite} derived learning rates for the case of $\cL_j$ the gradient operator, showing that a weaker source condition is sufficient compared to standard regression.

\paragraph{Meshless methods for PDEs.}
Such methods have been widely used to approximate PDE solutions in a meshless way (i.e. without the strict mesh requirements of finite elements solvers). \citet{kansa1990multiquadrics,fasshauer1996solving} among others studied how to combine pointwise evaluations of differential constraints and boundary conditions using kernel bases.
This forms part of the more general literature on scattered data approximation~\citep{wendland2004scattered}, which is mainly concerned with the interpolation problem: the constraints on the data are to be enforced exactly instead of weakly as in \cref{eq:erm}.
The interpolant is decomposed into a basis generated by a type of kernel known as Radial Basis Function (RBF). 
On the theoretical side, the analysis typically considers settings where the hypothesis space $\cH$ (which is an RKHS) is the same space in which the true PDE solution lives \citep{franke1998convergence,franke1998solving}. Convergence to this solution is proved via fill-distance techniques~\citep{wendland2004scattered}.

\paragraph{Regularization with differential operators.}
When $W=0$ (or equivalently $q=0$ in \cref{eq:PDE}), the second term of \eqref{eq:exp-rsk} can be interpreted as a regularizer. Indeed, regularizers of the form $\int (\cD u(z))^2 dz$ have been explored in the spline smoothing literature~\citep{wahba1990spline}, in inverse problems \citep{hanke1992regularization,engl1996regularization,arridge2019solving} as well as in machine learning~\citep{poggio90splines,smola98splines}. When $\cD$ is simple, the estimator can be computed in closed form using Green's functions. The non-homogeneous case ($q \neq 0$) with more general PDEs is more complex and was explored more recently as spatial regression with PDE regularization~\citep{azzimonti2015blood,sangalli2021spatial,arnone2022some}. In this case the closed form is usually not available, but the estimator can be computed using finite element methods.

\paragraph{Manifold Regularization.}
In semi-supervised learning, the differential regularizer is unknown and must be approximated from available data~\citep{zhu2003semi,zhou2005regularization}. This leads to \emph{manifold regularization} approaches. For instance, \citet{belkin2006manifold} use a gradient-based regularizer of the form $\| \nabla_M u\|^2$ to penalize deviations of $u$ from the data manifold $M$, which itself is estimated from data using the graph Laplacian. See also \citet{slepcev2019analysis,cabannes2021overcoming} for recent results in this direction.

\paragraph{Misspecified kernel methods.}
The approximation capabilities of kernel methods depend on the size of the RKHS $\cH$, which itself depends on the kernel choice. Kernels for which $\cH$ is dense in the space $\cF$ in which the target lives (typically an $L^p$ space) are called universal. Universality of many common kernels has been well studied \citep{micchelli2006universal,sriperumbudur2011universality,simon2018kernel}, and for kernel ridge regression, universality implies asymptotic consistency in the misspecified setting \citep{de2005learning}.
In the case of Sobolev RKHS, finer results exist in the misspecified setting, with convergence rates that depend on the exact smoothness of the target, characterized with a source condition \citep{steinwart2009optimal,lin2020optimal,fischer2020sobolev,zhang2023optimality}. Convergence results have also been established for Gaussian processes in misspecified settings \citep{wynne2021convergence,wang2022gaussian}. In the physics-informed machine learning context, the study of the misspecified setting is much more limited. Building on the work of \citet{narcowich2006escape} who analyze the error of RBF interpolation in the misspecified setting, \citet{schrader2012extended} apply these results to a PDE context, and establish error estimates when the solution is less smooth than the considered hypothesis space. Similarly to all the works that belong to the RBF interpolation literature, the results are established for a deterministic and noiseless set of input points, and the estimates are expressed in terms of fill-distance.
Furthermore, the target function can be outside of the considered RKHS, but must still belong to a Sobolev RKHS, i.e.~a space $H^s(\Omega)$ with $s > d/2$. 
In contrast, the present work considers random design with noise, and the space $\cF$ in which the target lives need not be an RKHS. 
More recently, \citet{baptista_solving_2025} studied a kernel-based collocation method for PDEs with rough solutions, which could be outside the RKHS, but the method considers the PDE in a weak form, and relies on test functions, which is a different setting than the present one. 

Provided with the above discussion we next describe the approach we consider and analyze.

\section{Physics-informed kernel methods (PIKS)}\label{sec:PIKS}
In this section, we introduce the PIKS estimator, specify some initial assumptions on $\cD$, and  introduce the regularized problem and  the closed-form expression for PIKS.

We consider the hypothesis space $\cH$ to be a reproducing kernel Hilbert space (RKHS) of functions from $\cX$ to $\RR$. An RKHS is defined by a kernel function $K: \cX\times\cX\to\RR$ such that for all $x\in\cX$, the function $K_x: y\mapsto K(x, y)$ belongs to $\cH$ and satisfies the reproducing property
\begin{equation}\label{eq:reproducing-property}
    \forall u \in \cH,\quad \langle u , K_x \rangle_\cH = u(x).
\end{equation}
For all the computations in this section, we need the following assumption.
\begin{assumption}\label{ass:diff-eval} For all $x \in \cX$, the map $\cH \to \RR, u \mapsto\cD u(x)$
 is a bounded functional.
\end{assumption}
\noindent Under \Cref{ass:diff-eval}, the Riesz representation theorem guarantees that for all $x\in\cX$ there exists a representer $\KD_x\in\cH$ such that
\begin{equation}\label{eq:diff-reproducing-property} \forall u \in \cH, \quad \langle u , \KD_x \rangle_\cH \ = \ \cD u(x).\end{equation}
\Cref{ass:diff-eval} is formulated in a general setting where $\cD$ can be any linear operator. We can however make it more concrete in the case of differential operators, for which the reproducing property is well studied \citep{zhou2008derivative}. In such a case, the smoothness of the kernel is enough to guarantee that \Cref{ass:diff-eval} holds, as shown by the following result.
\begin{lem}\label{lem:smooth-kernel-implies-bounded-diff-eval}
Assume that $\cX$ is compact and satisfies $\overline{\Int{\cX}} = \cX$, and consider a linear differential operator of the form 
\begin{equation}\label{eq:diff-op-order-s}
    \cD = \sum_{|\alpha| \leq s} c_\alpha \partial^\alpha,\end{equation}
 for some integer $s\geq1$ and where the $c_\alpha \colon \cX \rightarrow \RR$ are continuous coefficient functions.
    If $K \in C^{2s}(\cX \times \cX)$, then \Cref{ass:diff-eval} holds. Furthermore the representer $\KD_x$ is available in closed form:
\begin{equation}\label{eq:PDE-representer-expression}
    \KD_x: y \mapsto \mathcal{D}_1 K(x,y). 
\end{equation}
\end{lem}
\Cref{lem:smooth-kernel-implies-bounded-diff-eval} is a simple consequence of a result from \citet[Theorem 1]{zhou2008derivative}, and is proved in the appendix (see \Cref{lem:smooth-kernel-implies-bounded-diff-eval-apdx} in \Cref{sec:rkhs-basics}). Then to verify \Cref{ass:diff-eval}, we simply need a smooth-enough kernel. For example the Gaussian kernel $K(x,y) = \exp(- \frac{\|x - y\|_2^2}{2 \sigma^2})$ is $C^{\infty}$ over $\mathbb{R}^d \times \mathbb{R}^d$, while the Matérn kernel $K_\nu(x,y)$ of order $\nu>0$ is $C^{2s}$ over $\mathbb{R}^d \times \mathbb{R}^d$ as soon as $\nu > s$. \Cref{lem:smooth-kernel-implies-bounded-diff-eval} also provides the expression of the representer, which is key for computing the estimator in closed form. 

If the  kernel $K$ satisfies \Cref{ass:diff-eval}, we define the PIKS estimator $\widehat u_\lambda$ as the minimizer of the regularized physics-informed empirical risk:
\begin{equation}\label{eq:PIKS-estimator}
\widehat u_\lambda : = \argmin_{u \in \cH} \widehat R_\lambda(u), \qquad \widehat R_\lambda(u) =\frac{1}{n}\sum_{i=1}^n ( u (x_i) - y_i)^2 + \frac{1}{m}\sum_{j=1}^m (\cD u(z_j) - w_j)^2 + \lambda \|u \|^2_\cH.
\end{equation}
For any $\lambda > 0$, the regularized loss \eqref{eq:PIKS-estimator} is strongly convex, so the PIKS estimator is well-defined and unique, and one can show (see \Cref{prop:closed-from-expression-apdx} in \Cref{sec:PIML-kernels-apdx}) that it can be decomposed as 
\begin{equation}\label{closed-form-expression-eq}
    \uhl = \sum_{i=1}^n \alpha_i K_{x_i} + \sum_{j=1}^m \beta_j \KD_{z_j},
\end{equation}
where $(\alpha, \beta) = \left( \Kbf + \lambda \Jbf  \right)^{-1} \Ybf \ \in \RR^{n+m}$. In this formulation $\Kbf$ is a $(n+m)\times (n+m)$ block kernel matrix, $\Jbf$ a diagonal regularization matrix, and $\Ybf$ the regressed variables:
\begin{equation}\label{block-kernel-matrix}
    \Kbf = 
    \begin{pmatrix}[2]
        \Abf & \Cbf \\ 
        \Cbf^\top & \Bbf 
    \end{pmatrix}
    \qquad
    \Jbf _{i,i} = \begin{cases} 
    n & \text{if } 1 \leq  i \leq n \\ 
    m & \text{if } n+1 \leq i \leq n+m
    \end{cases}
    \qquad
    \Ybf = \begin{pmatrix}
        y_1 \\
        \vdots \\
        y_n \\
        w_1 \\
        \vdots \\
        w_m
    \end{pmatrix}.
\end{equation}
The blocks of the kernel matrix are
\begin{align*}
& \Abf \in \RR^{n \times n}, & \quad & \Abf_{i,i'}  = \langle K_{x_i} , K_{x_{i'}}\rangle_\cH, & \quad & \forall i,i' \in \lb 1 , n \rb;\\
& \Bbf \in \RR^{m \times m}, & \quad & \Bbf_{j,j'} = \langle K^\cD_{z_j}, K^\cD_{z_{j'}} \rangle_\cH, & \quad & \forall j,j' \in \lb 1 , m \rb; \\
& \Cbf  \in \RR^{n \times m}, & \quad & \Cbf_{i,j'} = \langle K_{x_i} , K^\cD_{z_{j'}} \rangle_\cH , & \quad & \forall i \in \lb 1 , n \rb, \forall j' \in \lb 1, m \rb.
\end{align*}
where we denoted $\lb 1 , n \rb = \{1,\dots,n\}$. Note that $\langle K_{x_i} , K_{x_{i'}}\rangle_\cH = K(x_i, x_{i'})$, and in the case of differential operators discussed in \Cref{lem:smooth-kernel-implies-bounded-diff-eval}, we also have $\langle K^\cD_{z_j}, K^\cD_{z_{j'}} \rangle_\cH =  \cD_1 \cD_2 K(z_j, z_{j'}) $ and $\langle K_{x_i} , K^\cD_{z_{j'}} \rangle_\cH = \cD_2 K(x_i, z_{j'})$ where $\cD_1$ denotes the differential operator with respect to the first variable of the kernel and $\cD_2$ with respect to the second one.

The structure of $\Kbf$ arises from the presence of two types of data, $(x_i, y_i)_{i=1}^n$ and $(z_j, w_j)_{j=1}^m$. Diagonal blocks $\Abf$ and $\Bbf$ are the kernel matrices of the two separate datasets, and $\Cbf$ represents the cross terms.
Computationally the PIKS estimator \eqref{closed-form-expression-eq} requires storing and later inverting an $(n+m)\times(n+m)$ matrix with a cost of $O((n+m)^2)$ space and  $O((n+m)^3)$ time units. While this cost is high in general, there exists a rich literature on approximations which greatly reduce the computational complexity of kernel methods without compromising on accuracy~\citep{rahimi2007random,rudi2015less}. In the particular case of structured matrices involving derivatives such as~\eqref{block-kernel-matrix}, we can notably cite the works of \citet{eriksson2018scaling,padidar2021scaling,de2021high} in the case of Hermite-Birkhoff interpolation and of \citet{chen2025sparse} in the case of PDEs.

\begin{rem}
As we described in \Cref{sec:related-settings}, hybrid regression settings in RKHS mixing different types of linear functionals (and in particular differential operators) were studied at least since \cite{kimeldorf1971some}. Solving such problems leads to block matrices of the form \eqref{block-kernel-matrix}, which were observed in Hermite-Birkhoff problems \citep{zongmin1992hermite}, in RBF collocation methods to approximate PDE solutions~\citep{fasshauer1996solving,franke1998solving, wendland2004scattered}, as well as in recent uses of Gaussian processes for PDEs~\citep{raissi2017gps,chen2021solving,chen2025sparse}.
\end{rem}

\section{Theoretical analysis of the PIKS estimator}\label{sec:theory}

In this section we derive theoretical results about the asymptotic convergence of the PIKS estimator. Usual kernel analyses rely on the assumption that the model is well-specified, that is, that the target function $u^*$ belongs to the same RKHS from which the estimator is taken. However, in practice, one does not always know the exact regularity of the target function. Even worse, the target function might lack the smoothness required to belong to any RKHS. For instance, it is well known that a Sobolev space $H^s(\Omega)$ is an RKHS if and only if $s > d/2$, a condition that is harder to satisfy in high dimensions. Such a situation arises naturally in PDE problems: for a second-order
elliptic equation with sufficiently regular coefficients and $L^2$ data,
elliptic regularity theory guarantees that the solution belongs to
$H^2(\Omega)$ \citep{evans2010partial}, but in general to no smoother
Sobolev space. Since $H^2(\Omega)$ is an RKHS only when $d \leq 3$, the
natural regularity class of the solution fails to be an RKHS as soon as
$d \geq 4$. This motivates studying the misspecified setting where $u^* \not\in \cH$.

Next, we first introduce the main technical assumptions in \Cref{sec:functional-assumptions}, then we state the main result in \Cref{sec:main-result}, which we illustrate in the case of differential operators and Sobolev spaces in \Cref{sec:Sobolev-setting}, and in the particular case of PDE settings in \Cref{sec:boundary-setting-PDEs}. Finally, in \Cref{sec:rates-src-condition}, we derive convergence rates, which we illustrate in the example of the Laplacian on Sobolev spaces of periodic functions.

\subsection{Functional assumptions}\label{sec:functional-assumptions}
We collect here the functional assumptions used to prove our main result in \Cref{main-theorem}. They are formulated for general Hilbert spaces of functions $\cF$ and  $\cG$. Later, in \Cref{sec:Sobolev-setting}, we specialize these assumptions to Sobolev spaces and give concrete sufficient conditions under which they can be verified.

We start by providing further details on the setting introduced in \Cref{sec:prob}. 
Let us consider a compact domain $\cX \subset \RR^d$ and a target $u^* \in \cF$ where $\cF$ is a Hilbert space of functions embedded in $L^2(\cX)$ (endowed with the Lebesgue measure of $\RR^d$). We consider a bounded linear operator $\cD : \cF \to \cG$, where $\cG$ is another Hilbert space of functions embedded in $L^2(\cX)$. Since the elements of $\cF$ and $\cG$ are only defined almost everywhere, we need the following compatibility condition:
\begin{assumption}\label{ass:embeddings}
    The embeddings $\cF \hookrightarrow L^2(\rho_X)$ and $\cG \hookrightarrow L^2(\rho_Z)$ are well-defined and bounded.
\end{assumption}
\noindent \Cref{ass:embeddings} ensures that the random variables $u^*(X)$ and $\cD u^* (Z)$ introduced in \Cref{sec:prob} are well defined. Indeed, suppose for example that $u^* \in L^2(\cX)$. Such a function is only defined almost everywhere with respect to the Lebesgue measure, so for instance if $\rho_X$ is equal to the Dirac distribution $\delta_x$ (that is, we only sample at a single, deterministic location $x \in \cX$), then the variable $u^*(X)$ is not well defined, as the value of $u^*(x)$ is not uniquely defined a priori. A similar observation holds for $\cD u^*(Z)$. If instead, for this same $u^*$, $\rho_X$ and $\rho_Z$ are absolutely continuous with respect to the Lebesgue measure on $\cX$, with bounded densities, we easily see that \Cref{ass:embeddings} holds. Note that \Cref{ass:embeddings} also covers less trivial settings. One important example is the one discussed in \Cref{sec:PDE-setting,sec:boundary-setting-PDEs} where $\rho_X$ is only supported on the boundary of $\cX$. Then, the  part of \Cref{ass:embeddings} concerning $\rho_X$ is provided not by the embedding of $L^2(\cX)$ into $L^2(\rho_X)$ but rather by the trace theorems~\citep{evans2010partial}, as we discuss in \Cref{sec:boundary-setting-PDEs}.

We now introduce the universality assumption, which is central to our study.
\begin{assumption}[Universality]\label{ass:universality}
  The RKHS $\cH$ is densely embedded in $\cF$.
\end{assumption}
\noindent This assumption is standard in the kernel literature when $\cF = L^2(\cX)$, and is satisfied for a broad class of kernels which are referred to as universal~\citep{micchelli2006universal}. It guarantees sufficient flexibility to learn the target function in misspecified settings, that is when $u^* \not\in \cH$.
However, in the physics-informed setting, we consider more structured spaces $\cF$, so the density of $\cH$ in $L^2$ is not enough, and \Cref{ass:universality} is thus a stronger requirement than classical universality. For Sobolev spaces $\cF = H^s(\cX)$, a sufficient condition for \Cref{ass:universality} to hold is $C_0^s$-universality, as we discuss in \Cref{sec:Sobolev-setting}, and it is satisfied for instance by Gaussian and Matérn kernels.

We now make two boundedness assumptions on the data-generating process and on the features in $\cH$ which are used to guarantee the validity of concentration inequalities.

\begin{assumption}[Bounded data]\label{ass:bounded-data}
We have $u^* \in L^\infty(\rho_X)$ and $\cD u^* \in L^\infty(\rho_Z)$. Furthermore, both noise random variables $\epsilon$ and $\eta$ are bounded almost surely. 
\end{assumption}
\Cref{ass:bounded-data} restricts the class of admissible targets and noise distributions. Its role is to ensure that the empirical quantities appearing in our analysis are uniformly bounded, so that Hoeffding-type concentration inequalities apply. Such boundedness assumptions are standard in convergence analyses of kernel methods. They can in principle be relaxed --- for example, by imposing tail or moment conditions and using Bernstein-type inequalities, truncation, or other refined concentration tools --- but pursuing these extensions is beyond the scope of the present work.

\begin{assumption}[Bounded features]\label{ass:bounded-features}
There exist $\kappa, \kappa_\cD > 0$ such that for all $u \in \cH$, for all $x \in \cX$, 
\begin{equation}\label{eq:bounded-features}
   |u(x)| \leq  \kappa \|u\|_\cH; \qquad |\cD u(x)|\leq \kappa_\cD \|u\|_\cH .
\end{equation}
\end{assumption}
    Under \Cref{ass:diff-eval}, we can define the kernel $K^\cD(x,y) := \langle K^\cD_x , K^\cD_y \rangle_\cH$, and \eqref{eq:bounded-features} can be reformulated as having, for all $x \in \cX$,
    \begin{equation}\label{eq:bounded-features-var}
        K(x,x) \leq \kappa^2; \qquad K^\cD(x,x) \leq \kappa_\cD^2. 
    \end{equation}
    The formulation \eqref{eq:bounded-features-var} might be more familiar to the reader. \Cref{ass:bounded-features} can typically be obtained as a consequence of the smoothness of the kernels and the compactness of $\cX$. It is in particular the case in the setting of \Cref{lem:smooth-kernel-implies-bounded-diff-eval} (for more details, see \Cref{lem:C^2s-on-compact-implies-bounded-features} in \Cref{sec:Sobolev-setting-apdx}).

\subsection{Main result}\label{sec:main-result}

Consider the \emph{physics-informed} setting described in \Cref{sec:prob}, as well as in the previous section.  Consider the PIKS estimator $\widehat u_\lambda \in \cH$ defined in \eqref{eq:PIKS-estimator}. The following result characterizes the asymptotic behavior of the estimator as the dataset sizes increase.

\begin{thm}\label{main-theorem}
    Under \Cref{ass:diff-eval,ass:embeddings,ass:universality,ass:bounded-data,ass:bounded-features}, the PIKS estimator is a universal learner: for any regularizing sequence $(\lambda_{n,m})$ such that
    \begin{equation}\label{eq:condition-lambda}
    \lambda_{n,m}\to 0,
    \qquad 
    \frac{\log N}{\lambda_{n,m}^3 N}\to 0
    \qquad \text{as } n,m\to\infty ,
\end{equation}
    where $N = \min(n,m)$, almost surely, the estimator $\widehat{u}_{\lambda_{n,m}}$ satisfies
\begin{equation}\label{eq:main-thm-cases}
\begin{cases}
\| \widehat u_{\lambda_{n,m}} - u^* \|_{L^2(\rho_X)} \ \underset{n,m \rightarrow \infty}{\longrightarrow} \ 0  \\
\| \cD \widehat u_{\lambda_{n,m}}  - \cD u^*\|_{L^2(\rho_Z)} \ \underset{n,m \rightarrow \infty}{\longrightarrow}  \ 0.
\end{cases}
\end{equation}
\end{thm}

This theorem, which is proved as \Cref{main-thm-apdx} in the appendix, shows that the PIKS estimator is able to learn the function's values and satisfy the physical constraint at the same time. Here, the main technical difficulty lies in showing that it holds under weak assumptions on $u^*$. 
As discussed in \Cref{sec:related-settings}, theoretical analyses usually assume $u^*$ belongs to the same RKHS $\cH$ from which $\widehat u_\lambda$ is taken. Indeed, in the well-specified setting, one can typically prove the stronger convergence $\|\widehat u_\lambda - u^*\|_\cH \to 0$, from which we can derive both convergences $\| \widehat u_\lambda - u^* \|_{L^2(\rho_X)} \to 0$ and $\| \cD \widehat u_\lambda  - \cD u^*\|_{L^2(\rho_Z)} \to 0$ as direct consequences.
In the misspecified setting instead, the standard analysis \citep[see e.g.][]{de2005learning} only guarantees the $L^2$ convergence $\|\widehat u_\lambda - u^*\|_{L^2(\rho_X)}\to 0$ for the classical KRR estimator, which does not imply the convergence $\| \cD \widehat u_\lambda  - \cD u^*\|_{L^2(\rho_Z)}\to 0$. 
\Cref{main-theorem} extends the analysis to the physics-informed setting, by showing that universality guarantees both convergences at the same time, without requiring $u^* \in \cH$.

\paragraph{Remark: Convergence with $n\to\infty$ and $m$ finite.}
An observation which could arise from looking at \Cref{eq:main-thm-cases} is that both dataset sizes ($n$ and $m$) must tend to infinity to guarantee convergence, even for the function values alone; in other words, the two convergences are not decoupled. It is possible to decouple them by introducing a scaling parameter to the physics term. If this parameter were set to decay to $0$ as $n \to \infty$ while $m$ remains finite, the PIKS estimator would become asymptotically equivalent to a classical KRR estimator, thereby recovering the standard $L^2$ convergence on the values, but losing the physical consistency. However, this is not the focus of the present paper, and scaling parameters were left out to maintain a simplified analysis.

\subsection{The case of differential operators on Sobolev spaces}\label{sec:Sobolev-setting}

The assumptions in \Cref{sec:functional-assumptions} and the result in \Cref{sec:main-result} are stated for a general linear operator $\cD$ and function spaces $\cF, \cG$. Here, we specialize them to the case where $\cD$ is a differential operator defined on a Sobolev space. In particular, we discuss the implications of the universality assumption, and analyse more concretely the theorem's applications in this setting.

Consider $\cX=\overline{\Omega}$ with $\Omega$ a bounded Lipschitz domain, $u^*\in \cF = H^s(\Omega)$, for some integer $s>0$, and $\cG = L^2(\Omega)$. Assume that $\cD$ is a linear differential operator of order $s$ as defined in \cref{eq:diff-op-order-s}. We first check that the operator $\cD$ indeed defines a bounded operator $H^s(\Omega) \to L^2(\Omega)$ (see \Cref{lem:diff-op-is-bounded-from-H^s-to-L^2} in \Cref{sec:sobolev-setting-generalities}). For \Cref{main-theorem} to apply, we further need to check \Cref{ass:diff-eval,ass:embeddings,ass:universality,ass:bounded-data,ass:bounded-features}. \Cref{ass:bounded-data} is a standard boundedness assumption that we take independently from the setting. All other assumptions can be verified as consequences of the setting and the kernel choice, as we see below.

Assume that $K : \cX \times \cX \to \RR$ is a $C^{2s}$ kernel. We already established in \Cref{sec:PIKS} that in such a case, \Cref{ass:diff-eval} holds. It is also straightforward to check that \Cref{ass:bounded-features} holds, as a consequence of the compactness of $\cX$ (see \Cref{lem:C^2s-on-compact-implies-bounded-features} in \Cref{sec:sobolev-setting-generalities}). 
\Cref{ass:universality} instead requires more care, as it requires the RKHS to be rich enough to approximate not only functions, but also their derivatives up to order $s$ --- a requirement that is stronger than the standard $L^2$ universality.
A sufficient condition, which holds for common kernels such as the Gaussian one, is $C_0^s$ universality.

\begin{df}[$C^s_0$ universality]
    Let $U$ be an open subset of $\RR^d$ and let $C_0^s(U)$ be the set of continuous functions $f : U \to \RR$ such that $f$ and all its derivatives up to order $s$ tend to $0$ at infinity (i.e. for any $\epsilon > 0$ and $\alpha \in \NN^d$, $|\alpha|\leq s$, there exists a compact $K_\epsilon \subset U$ such that for all $x \in U \backslash K_\epsilon$, $|\partial^\alpha f(x)| \leq \epsilon$).  We say that a kernel $K$ is $C_0^s$-universal on $U$ if for any $f \in C_0^s(U)$, for any $\epsilon > 0$, there exists $f_\cH$ in the RKHS $\cH$ associated to $K$ such that for any $\alpha \in \NN^d$ satisfying $|\alpha| \leq s$,
    \[\sup_{x \in U} |\partial^\alpha f (x) - \partial^\alpha f_\cH(x)| \leq \epsilon.\]
\end{df}

The following lemma shows that if the kernel considered is $C_0^s$-universal on $\RR^d$, then \Cref{ass:universality} holds.
\begin{lem}\label{lem:C_0^s-universality-implies-ass:universality}
    Assume that $\cX = \overline{\Omega}$, with $\Omega$ a Lipschitz domain, and that $\cF = H^s(\cX) = H^s(\Omega)$. If $K$ is the restriction to $\cX$ of a $C_0^s$-universal kernel on $\RR^d$, then \Cref{ass:universality} holds.
\end{lem}
We prove \Cref{lem:C_0^s-universality-implies-ass:universality} in \Cref{sec:sobolev-setting-generalities}, where it is restated as \Cref{lem:C_0^s-universal-implies-sobolev-density}. The proof proceeds by extending a given function $u \in H^s(\Omega)$ to $\RR^d$, approximating the extension by sufficiently regular compactly supported functions, and then invoking $C_0^s$ universality. Restricting the resulting approximants to $\cX$ yields the desired approximation of $u$.

 $C_0^s$ universality was discussed by \citet{simon2018kernel}, in a paper that characterizes it in particular for translation-invariant kernels on $\RR^d$ (which are of the form $K(x,y) = \Phi(x-y)$). Such kernels are $C^s_0$-universal if and only if they are $C^{s,s}$ (which is slightly weaker than the $C^{2s}$ requirement of \Cref{lem:smooth-kernel-implies-bounded-diff-eval}, that we consider in the present work --- we refer the reader to \citet{simon2018kernel} for the precise definition) and the Fourier transform of $\Phi$ has full support. For instance, this is verified for the Gaussian kernel and the Matérn kernel of index $\nu > s$.

\paragraph{In-domain sampling.}
The only missing element to apply \Cref{main-theorem} is \Cref{ass:embeddings}. Such assumption depends on $\rho_X,\rho_Z$ and thus on the sampling setting. A natural choice is uniform sampling in $\Omega$, as we study here; we will discuss another interesting choice in the next section. Assume now that $\rho_X = \rho_Z = \Unif(\Omega)$. Then, we have $L^2(\rho_X) = L^2(\rho_Z) \sim L^2(\Omega)$. The embedding of $\cG$ into $L^2(\rho_Z)$ is then trivial, and the embedding of $\cF$ into $L^2(\rho_X)$ is a consequence of the canonical embedding of $H^s(\Omega)$ into $L^2(\Omega)$. We thus see that \Cref{ass:embeddings} is satisfied. Note that, more generally, it would be satisfied for any distributions with bounded density on $\Omega$--- and that \Cref{main-theorem} applies.

\Cref{main-theorem} tells us that as the data grows to infinity, we get the convergences $\widehat u \ \overset{L^2(\rho_X)}{\longrightarrow} \ u^*$ and $\cD \widehat u \ \overset{L^2(\rho_Z)}{\longrightarrow} \ \cD u^*$, which are (up to normalization) convergences in $L^2(\Omega)$ since $\rho_X$ and $\rho_Z$ are both uniform.
The benefit of the PIKS estimator compared to standard kernel regression in this setting lies in the ability to learn with two kinds of data, and to guarantee that we learn both $u^*$ and $\cD u^*$ simultaneously, as the latter does not follow from the former in general.
For instance, if we know that the target satisfies a PDE $\cD u^* = q$, then, with enough data, with PIKS we will have $\cD \widehat u \approx q$ (in the $L^2$ sense),
i.e. the estimator is physically consistent. In contrast, with kernel ridge regression one can only establish the $L^2$ convergence of $\widehat u$ to $u^*$, without a corresponding guarantee on physical consistency. The benefit of these two simultaneous convergences becomes even clearer in some cases where we have stability of the PDE, as we study in the next paragraph. 

This setting is related to the earlier fundamental work of \citet{doumeche2024physics,doumeche2025physics} on physics-informed learning. The authors study a physics-informed constraint of the form $\|\cD u\|_{L^2}$ in addition to the standard Tikhonov regularization, in a Sobolev RKHS. In our setting, this corresponds to the homogeneous case, $\cD u^*=0$, together with an
infinite amount of physical information, namely $m\to\infty$. They show that their estimator is equivalent to classical KRR with a modified,
physics-informed kernel. Unlike the present work, they consider the well-specified setting $u^*\in\cH$, in which the asymptotic convergence $\|\widehat u-u^*\|_{\cH}\to0$ is straightforward to establish; the two convergence results in \Cref{main-theorem} then follow as consequences. Their
focus is different, however: they study how incorporating the physics accelerates the convergence of $\widehat u$ to $u^*$ in the $L^2$ sense, which is beyond the scope of the present work.

\paragraph{Stronger convergence results in the elliptic case.}
For some differential operators, the convergence
$\cD\widehat u\to\cD u^*$ provided by \Cref{main-theorem} can be used to
deduce convergence of $\widehat u$ to $u^*$ in a stronger topology than that
of $L^2$. This relies on PDE regularity estimates and depends strongly on the
operator $\cD$. We illustrate this in the elliptic case.

We retain the Sobolev setting considered above, with uniform sampling in
$\Omega$. We further assume that $s=2$ and that $\cD$ is a uniformly elliptic
operator satisfying the regularity assumptions of
\Cref{sec:in-domain-stronger-cvg-elliptic-apdx}. Ellipticity guarantees us \citep[Section 6.3]{evans2010partial} that for any $q \in L^2(\Omega)$ and any weak solution $u$ of the PDE $\cD u = q$ on $\Omega$, we have the estimate
    \begin{equation}\label{elliptic-regularity-domain-eq-main}
        \| u \|_{H^{2}(V)} \leq C_V ( \| u \|_{L^2(\Omega)} + \| q \|_{L^2(\Omega)}),\end{equation}
for any open $V$ such that $\overline{V} \subset \Omega$. The regularity estimate \eqref{elliptic-regularity-domain-eq-main} allows us to translate the two $L^2$ convergences of \Cref{main-theorem} into a convergence of $\widehat u$ to the target $u^*$ in a stronger, Sobolev sense, as we see in the following result.

\begin{coro}\label{Coro:cvg-H^2(V)-main}
   Under the preceding assumptions, for any sequence $(\lambda_{n,m})$ satisfying \eqref{eq:condition-lambda}, we have that almost surely, for any open $V$ such that $\overline{V} \subset \Omega$,
       \[\|\uh_{\lambda_{n,m}} - u^* \|_{H^2(V)} \ \underset{n,m \rightarrow \infty}{\longrightarrow} \ 0. \]
\end{coro}
\Cref{Coro:cvg-H^2(V)-main} is restated as \Cref{Coro:cvg-H^2(V)-apdx} and proved in \Cref{sec:in-domain-sampling-apdx}. \Cref{Coro:cvg-H^2(V)-main} shows a benefit of physics-informed learning in the elliptic setting: while for a classical KRR estimator (using only the classical data $(x_i,y_i)$), theory typically only guarantees $L^2$ convergence, here we obtain the asymptotic convergence of the PIKS estimator to the target in the $H^2$ sense on all compactly embedded domains.

\subsection{Data on the boundary: solving PDEs}\label{sec:boundary-setting-PDEs}

In the previous section, we focused on uniform sampling in $\Omega$ for both $X$ and $Z$. Another important sampling setting is the one described in \Cref{sec:PDE-setting}, where $(x_i)_{i=1}^n$ are sampled on the domain's boundary and $(z_j)_{j=1}^m$ in its interior. This happens when instead of having data measurements of $u^*$ in the domain, we have a boundary condition at our disposal. This gets us closer to typical PDE settings, where we are trying to solve a boundary value problem
\begin{equation}\label{eq:boundary-value-pb}
\begin{cases}
    \cD u^*(x) = q(x) & x \in \Omega, \\
    u^*(x) = h(x) & x \in \partial \Omega.
\end{cases}\end{equation}

Let us indeed consider, same as \Cref{sec:Sobolev-setting}, that $\Omega$ is a bounded Lipschitz domain, $u^* \in \cF = H^s(\Omega)$, $\cG = L^2(\Omega)$, and $\cD : \cF \to \cG$ is a differential operator defined as in~\eqref{eq:diff-op-order-s}. Consider this time that $\rho_X = \Unif(\partial \Omega)$ and $\rho_Z = \Unif(\Omega)$. Let $K\in C^{2s}(\cX\times\cX)$ be the restriction of a $C_0^s$-universal kernel on $\RR^d$. The only difference is the sampling setting, and we know from the previous section that \Cref{ass:diff-eval,ass:universality,ass:bounded-features} hold, and again we can assume that \Cref{ass:bounded-data} holds independently. In order to apply \Cref{main-theorem}, we only need to check \Cref{ass:embeddings}.

For that, we observe that while the first equality in \eqref{eq:boundary-value-pb} must be understood as holding almost everywhere, with respect to the Lebesgue measure, the second equality does not make sense a priori since $\partial \Omega$ has Lebesgue measure $0$ and $u^*$ is only defined Lebesgue almost everywhere. Nevertheless it can be made meaningful with the trace operator \citep[Sec. 5.5]{evans2010partial} which is a bounded operator
\[\Tr:H^1(\Omega) \to H^{1/2}(\partial \Omega)\]
that coincides with the restriction to the boundary $u \mapsto u|_{\partial \Omega}$ when $u $ is a continuous function over $\overline{\Omega}$. The trace operator is the standard way of defining boundary conditions in PDE theory, as soon as we consider weak solutions. Composing $\Tr$ with the canonical Sobolev embeddings $H^s(\Omega) \hookrightarrow H^1(\Omega)$, $H^{1/2}(\partial \Omega) \hookrightarrow L^2(\partial \Omega)$, and with the bounded map $L^2(\partial \Omega) \to L^2(\rho_X)$, we see that \Cref{ass:embeddings} holds. Hence, \Cref{main-theorem} applies.

The conclusions of \Cref{main-theorem} have a different meaning than for in-domain sampling in \Cref{sec:Sobolev-setting}. Since $\rho_X$ is supported on $\partial \Omega$, the convergence $\widehat u \overset{L^2(\rho_X)}{\longrightarrow} u^* $ is merely a convergence on the boundary; by itself, it tells us nothing about what happens in $\Omega$. The two convergences of \Cref{main-theorem} mean that the PIKS estimator approximately satisfies the PDE \eqref{eq:boundary-value-pb} when the data grow to infinity. For a strong convergence to $u^*$ on $\Omega$, one further needs an appropriate stability estimate for the boundary-value problem, which is PDE-dependent. We give an example of stronger convergence results at the end of the current section.

This setting has received considerably more attention in the literature, as it corresponds to a classical PDE formulation. In particular, the use of kernels to numerically solve linear PDEs by prescribing function values at given boundary points and the values of the differential operator at given domain points is not new and has been studied in the framework of interpolation with RBF functions~\citep{wendland2004scattered}. In such literature, the points at which the differential operator is evaluated are deterministic, referred to as collocation points. The convergence analysis of these methods is then carried out using the notion of fill-distance, that is, the biggest distance that exists between a point and its closest neighbor. The analysis typically takes place in a noiseless setting, with the true solution of the PDE belonging to the RKHS defined by the chosen RBF.

Our analysis addresses a different and complementary regime, building on KRR theory in misspecified settings. We assume that the sampling points $x_i$ and $z_j$ are drawn randomly from prescribed distributions, allow the observations to be noisy, and do not require the target $u^*$ to belong to the RKHS $\cH$. This makes it possible to derive statistical guarantees that account jointly for sampling variability, observation noise, and model misspecification. In turn, the resulting bounds are distribution-dependent and typically control an average error, rather than providing deterministic guarantees tied to the geometric coverage of a particular set of collocation points.

 The more recent work \citep{chen2021solving} studies nonlinear operators. They formulate the problem as a nested optimization and propose an iterative method (with a Gauss-Newton algorithm). On the theoretical side, they consider the well-specified setting as well, and prove convergence of the second member of the PDE. Similar to existing results on PINNs, they must assume the convergence of the optimization process, as it is difficult to obtain guarantees in the nonlinear case.

 \paragraph{Stronger convergence results.} Assume now that $s=2$ and $\cD$ is uniformly elliptic with smooth coefficients over $\overline{\Omega}$ (see the beginning of \Cref{sec:boundary-setting-strong-cvg-apdx} for the precise assumptions). Further assume that the boundary $\partial \Omega$ is a smooth $(d-1)$-dimensional manifold, $\Omega$ being locally on one side of $\partial \Omega$. We assume that $0$ is not a Dirichlet eigenvalue for the operator $\mathcal{D}$  in $\Omega$, so that using Fredholm alternative \citep{evans2010partial},
  we can guarantee that for any $(q,h) \in L^2(\Omega) \times H^{3/2}(\partial \Omega)$, there exists a unique solution $u \in H^2(\Omega)$ to \eqref{eq:boundary-value-pb}. This is the case for instance for the Laplacian $\cD = - \Delta$.

Then, ellipticity and the regularity of the boundary guarantee us \citep[see][]{lions2012non} that for any $(q,h) \in L^2(\Omega) \times H^{3/2}(\partial \Omega)$, for any $u \in H^2(\Omega)$ solution of \eqref{eq:boundary-value-pb}, we have the estimate
\begin{equation}\label{lions-magenes-boundary-regularity-eq-main}
    \|u\|_{H^{1/2}(\Omega)} \leq c \left(  \| h \|_{L^2(\partial \Omega)} + \|q \|_{\Xi^{-3/2}(\Omega)} \right), \end{equation}
where $c > 0$ and $\| \cdot \|_{\Xi^{-3/2}(\Omega)}$ is a norm weaker than $\| \cdot \|_{L^2(\Omega)}$ \citep[see][for more details]{lions2012non}. The estimate \eqref{lions-magenes-boundary-regularity-eq-main} allows us to translate the two $L^2$ convergences of \Cref{main-theorem} into an $H^{1/2}$ convergence of $\uh_{\lambda_{n,m}}$ to the target $u^*$, as we see in the following result.

\begin{coro}\label{Coro:cvg-H^1/2(Omega)-main}
    Under the preceding assumptions, for any sequence $(\lambda_{n,m})$ satisfying \eqref{eq:condition-lambda}, almost surely,
      \begin{equation}\label{eq:cvg-H^1/2(Omega)-main}
          \|\uh_{\lambda_{n,m}} - u^* \|_{H^{1/2}(\Omega)} \ \underset{n,m \rightarrow \infty}{\longrightarrow} \ 0, 
      \end{equation}
and, for every open $V$ such that $\overline{V} \subset \Omega$, almost surely,
       \begin{equation}\label{eq:cvg-H^2(V)bis-main}
           \|\uh_{\lambda_{n,m}} - u^* \|_{H^2(V)} \ \underset{n,m \rightarrow \infty}{\longrightarrow} \ 0. 
       \end{equation}
\end{coro}

\Cref{Coro:cvg-H^1/2(Omega)-main} is proved in \Cref{sec:boundary-setting-strong-cvg-apdx} under the form of two results, \Cref{Coro:cvg-H^1/2(Omega)-apdx} and \Cref{coro:cvg-H^2(V)bis-apdx}. Let us compare this result with \Cref{main-theorem}: there, since $\rho_X$ is supported on $\partial \Omega$, the $L^2$ convergence of $\widehat u_{\lambda_{n,m}}$ to $u^*$ only happens on the boundary. Instead, \eqref{eq:cvg-H^1/2(Omega)-main} shows that, in the elliptic case with smooth boundary and uniqueness of the solution, we have in fact $H^{1/2}$ convergence on the whole domain $\Omega$. Then, analogously to \Cref{Coro:cvg-H^2(V)-main}, \eqref{eq:cvg-H^2(V)bis-main} shows that we have $H^2$ convergence on all compactly embedded domains.

\subsection{Convergence rates}\label{sec:rates-src-condition}

\Cref{main-theorem} establishes asymptotic convergence, without any indication on the convergence speed. It is actually possible to get convergence rates if one further assumes a source condition, following other works in the KRR literature \citep[see e.g.][]{de2005model,steinwart2009optimal,blanchard2018optimal,lin2020optimal,fischer2020sobolev,zhang2023optimality}. To state the source condition, let us introduce some notation. \Cref{ass:embeddings} allows us to consider the bounded operator
\begin{equation}\label{def:cA_rho}
    \function{\cA_\rho}{\cF}{L^2(\rho_X) \times L^2(\rho_Z)}{u}{(u, \cD u),}
\end{equation}
where we endow $L^2(\rho_X) \times L^2(\rho_Z)$ with the product norm defined by $\|(f,g) \|_\rho^2 := \|f\|_{L^2(\rho_X)}^2 + \|g\|_{L^2(\rho_Z)}^2 $.
Let us denote $(h,q) := \cA_\rho u^*$. In the boundary-value setting of \Cref{sec:PDE-setting} this is consistent with the notation, since $h = u^*|_{\partial \Omega}$ and $q = \cD u^*$; in the in-domain setting $h$ is simply $u^*$ viewed in $L^2(\rho_X)$. We can compose $\cA_\rho$ with the embedding $i : \cH \to \cF$ to define 
\begin{equation}\label{def:A_rho}
A_\rho = \cA_\rho \circ i,
\end{equation}
which allows us to write for any $u \in \cH$,
\[ \|A_\rho u - (h,q) \|_\rho^2 \ = \ \| u - h \|_{L^2(\rho_X)}^2 + \|\cD u - q \|_{L^2(\rho_Z)}^2.\]
 We can now define
\begin{equation}
    L : = A_\rho A_\rho^*: L^2(\rho_X) \times L^2(\rho_Z) \to L^2(\rho_X) \times L^2(\rho_Z).
\end{equation}
The operator $L$ is a composite integral operator: for $(f,g) \in L^2(\rho_X) \times L^2(\rho_Z)$ and $(x,z) \in \cX^2$, it can be expressed as
\begin{equation*}
    \begin{split} L (f,g)(x,z) = \left( \int_{\Input} \langle K_y , K_x \rangle_\cH \, f(y) d \rho_X(y) + \int_{\Input} \langle K^\cD_t , K_x \rangle_\cH \, g(t) d\rho_Z (t)\ , \right. \\
    \left. \int_{\Input} \langle K_y , K^\cD_z \rangle_\cH \, f(y) d \rho_X(y) + \int_{\Input} \langle K^\cD_t , K^\cD_z \rangle_\cH \, g(t) d\rho_Z(t)\right).
\end{split}
\end{equation*}
We now have all the elements to introduce a source condition adapted to this physics informed setting. 
\begin{assumption}[Physics-informed source condition]\label{ass:src}
    There exists $r \in (0,1]$ such that $(h, q) \in \ran L^r$, i.e. there exists $(\widetilde f , \widetilde g) \in L^2(\rho_X) \times L^2(\rho_Z)$ such that $(h, q) = L^r(\widetilde f , \widetilde g) $.
\end{assumption}

Given the above assumption we can derive explicit convergence rates.

\begin{thm}\label{thm:finite-sample-rates-main}
  Let $\delta \in (0,\frac12)$. Under \Cref{ass:embeddings,ass:diff-eval,ass:universality,ass:bounded-features,ass:bounded-data,ass:src}, if $\lambda \geq C \max \left( \frac{1}{n}\log \frac{n}{\delta}, \frac{1}{m}\log \frac{m}{\delta} \right) $, then with probability at least $1 - \delta$, the PIKS estimator satisfies
\begin{equation}\label{eq:composite-rates}
   \| \uhl - h \|_{L^2(\rho_X)} + \|\cD \uhl - q \|_{L^2(\rho_Z)}
\quad \lesssim \quad \frac{\sqrt{\ln(1/\delta)}}{\lambda^{1-\tilde{r}} \sqrt{n}}
+ \frac{\sqrt{\ln(1/\delta)}}{\lambda^{1-\tilde{r}} \sqrt{m}} + \lambda^r,
\end{equation}
where $\tilde{r} = \min(r,1/2)$.
\end{thm}

\Cref{thm:finite-sample-rates-main} is restated and proved as \Cref{coro:rates-apdx} in \Cref{sec:estimation-error-apdx}. If we consider the minimum between $n$ and $m$ and select $\lambda$ accordingly, the following corollary gives a simplified rate.
\begin{coro}\label{coro:simplified-rate-main}
    Let us set $\lambda = N^{-1/2}$ if $r \leq 1/2$, and $\lambda = N^{-\frac{1}{2r+1}}$ if $r > 1/2$, where $N := \min(n,m)$. Under \Cref{ass:src}, if $\delta \in (0,\frac12)$ is such that $\lambda \geq C \max \left( \frac{1}{n}\log \frac{n}{\delta}, \frac{1}{m}\log \frac{m}{\delta} \right) $, then with probability at least $1 - \delta$, we have
        \begin{align}\label{eq:simplified-rates}
        \| \uh_\lambda - h \|_{L^2(\rho_X)} + \|\cD \uh_\lambda - q \|_{L^2(\rho_Z)} \quad & \lesssim \quad \sqrt{\ln(1/\delta)} N^{-r/2} & \text{if } \ r \leq 1/2, \\
        \| \uh_\lambda - h \|_{L^2(\rho_X)} + \|\cD \uh_\lambda - q \|_{L^2(\rho_Z)} \quad & \lesssim \quad \sqrt{\ln(1/\delta)} N^{-\frac{r}{2r + 1}} & \text{if } \ r > 1/2.
    \end{align}
\end{coro}
\noindent \Cref{coro:simplified-rate-main} is restated and proved as \Cref{coro:simplified-rates-apdx} in \Cref{sec:estimation-error-apdx}. 

Using integral operator techniques and source conditions like \Cref{ass:src} is classical in KRR analysis; the distinctive feature here is that the rates depend on two indices $n$ and $m$. For the convergence to be guaranteed, both need to tend to infinity. The practicality of the rates provided in \eqref{eq:composite-rates} and \eqref{eq:simplified-rates} depends on the ability to interpret the source condition, which itself depends on the operator $L$. Such an operator is a composite integral operator which is more complex than the corresponding object in classical KRR. We provide below an example where the source condition is interpretable, while leaving a more general analysis for future work.

\begin{ex}[Laplacian on periodic functions]
Consider $\cX = [0,1]^d$, and denote, for any $\tau > 0$,
\[\Hper^\tau(\cX) \subset H^\tau(\cX)\]
the Sobolev space of order $\tau$ of periodic functions on $\cX$. Such spaces can be described with Fourier series, as we detail in \Cref{sec:laplacian-periodic-fct-rates}. Let us consider $\cF = \Hper^2(\cX)$, $\cG = L^2(\cX)$ and $\cD = \Delta = \sum_{i=1}^d \frac{\partial^2}{(\partial x_i)^2}$ the Laplacian, which defines a continuous operator from $\cF$ to $\cG$. Let $\cH = \Hper^\tau(\cX)$, with $\tau > d/2 + 2$. Let us consider $\rho_X = \rho_Z = \Unif([0,1]^d)$ (which corresponds to the in-domain sampling setting of \Cref{sec:Sobolev-setting}). In particular $L^2(\rho_X) = L^2(\rho_Z) = L^2([0,1]^d)$. One can show that \Cref{ass:diff-eval,ass:embeddings,ass:universality,ass:bounded-features} hold (see \Cref{lem:laplacian-periodic-assumptions-check} in \Cref{sec:laplacian-periodic-fct-rates}). Consider furthermore a target $u^* \in \cF$, and a noise model satisfying \Cref{ass:bounded-data}. Then \Cref{main-theorem} guarantees the asymptotic convergence. Furthermore, in such a setting, the source condition is interpretable in terms of smoothness of the target function $u^*$, as shown by the following proposition.
\begin{prop}\label{Ex:Laplacian-on-the-torus}
    Let $u^* \in \cF$, and $(h,q) = \cA_\rho u^*$. Then, \Cref{ass:src} is satisfied, with $r \in (0,1]$ if and only if $u^* \in \Hper^{\sigma_r}(\cX)$, with $\sigma_r = 2 + 2r(\tau-2)$. In particular, for $r = \frac{1}{2}$, \Cref{ass:src} is equivalent to $u^* \in \cH$, which corresponds to the well-specified setting.
\end{prop}
\Cref{Ex:Laplacian-on-the-torus} is proved in \Cref{sec:laplacian-periodic-fct-rates}. It is classical in KRR to characterize source conditions in Sobolev RKHS as Sobolev smoothness of the target $u^*$ of a given order \citep[see e.g.][]{fischer2020sobolev}. We see with \Cref{Ex:Laplacian-on-the-torus} that the same is possible in the case of the Laplacian on periodic functions. This gives a clear interpretation to \Cref{ass:src}.
\end{ex}

\section{Empirical verification}\label{sec:experiments}
In this section we demonstrate the performance of PIKS in two scenarios: 
the first one is a supervised learning problem in which we show that gradient information helps improve the underlying model's accuracy, decreasing the error on both the function values and derivatives. Such a setting is sometimes referred to as \emph{Sobolev training} or \emph{Hermite learning}.
In the second scenario, we use PIKS as a \emph{PDE solver} and compare it to FEM solvers, PINNs and alternative kernel-based PDE solvers on three linear PDEs. As discussed earlier, PIKS in this setting is close to classical meshless methods with kernels~\citep{wendland2004scattered}, and our results emphasize that there are many cases in which kernel regression with a RBF kernel is competitive with more complex methods involving neural networks or kernel functions.
We can furthermore show (see \cref{fig:smooth-funcs}) that indeed it is possible to use kernel methods for learning in the misspecified setting.

\subsection{Learning with gradient data} \label{sec:expgrad}
In traditional machine learning settings, where the goal is to learn a function $f: \cX\to\mathcal{Y}$ given a dataset of pairs $(x, y) \in \cX\times\mathcal{Y}$, PIKS can be used whenever additional information about linear transformations of $f$ is known.
We demonstrate the effect of incorporating such information with a simple example of a smooth 2D function $f(x) = \sin(\pi x_1) \cdot \sin(\pi x_2) + 4 \sin(4\pi x_1) \cdot \sin(4\pi x_2)$ with additive Gaussian noise. 
We compare the generalization error of KRR with that of PIKS which has additional access to derivative information. The number of training samples from $f$ is fixed to $n$, and the additional samples from $\partial_{x_1} f$ and $\partial_{x_2} f$ are also fixed to $n$.
The RMSE obtained as a function of $n$ and of the available data is shown in \cref{fig:imp-learn}. Gradient data is especially useful for this function: having access to $n$ function plus $n$ gradient points is better than having access to $2n$ function points. 

\begin{figure}
    \centering
    \begin{minipage}[t]{0.44\textwidth}
        \centering
        \includegraphics[width=0.89\textwidth]{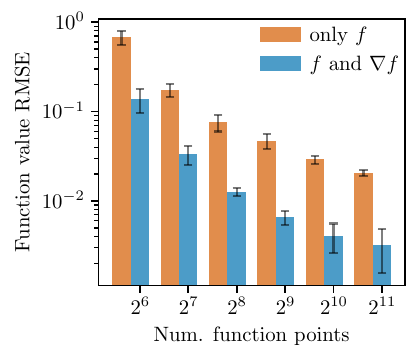} 
        \caption{Improved learning accuracy with derivative data. A dataset of function values (in red) is augmented by gradient (in blue) information to improve accuracy.}
        \label{fig:imp-learn}
    \end{minipage}\hfill
    \begin{minipage}[t]{0.54\textwidth}
        \centering
        \includegraphics[width=\textwidth]{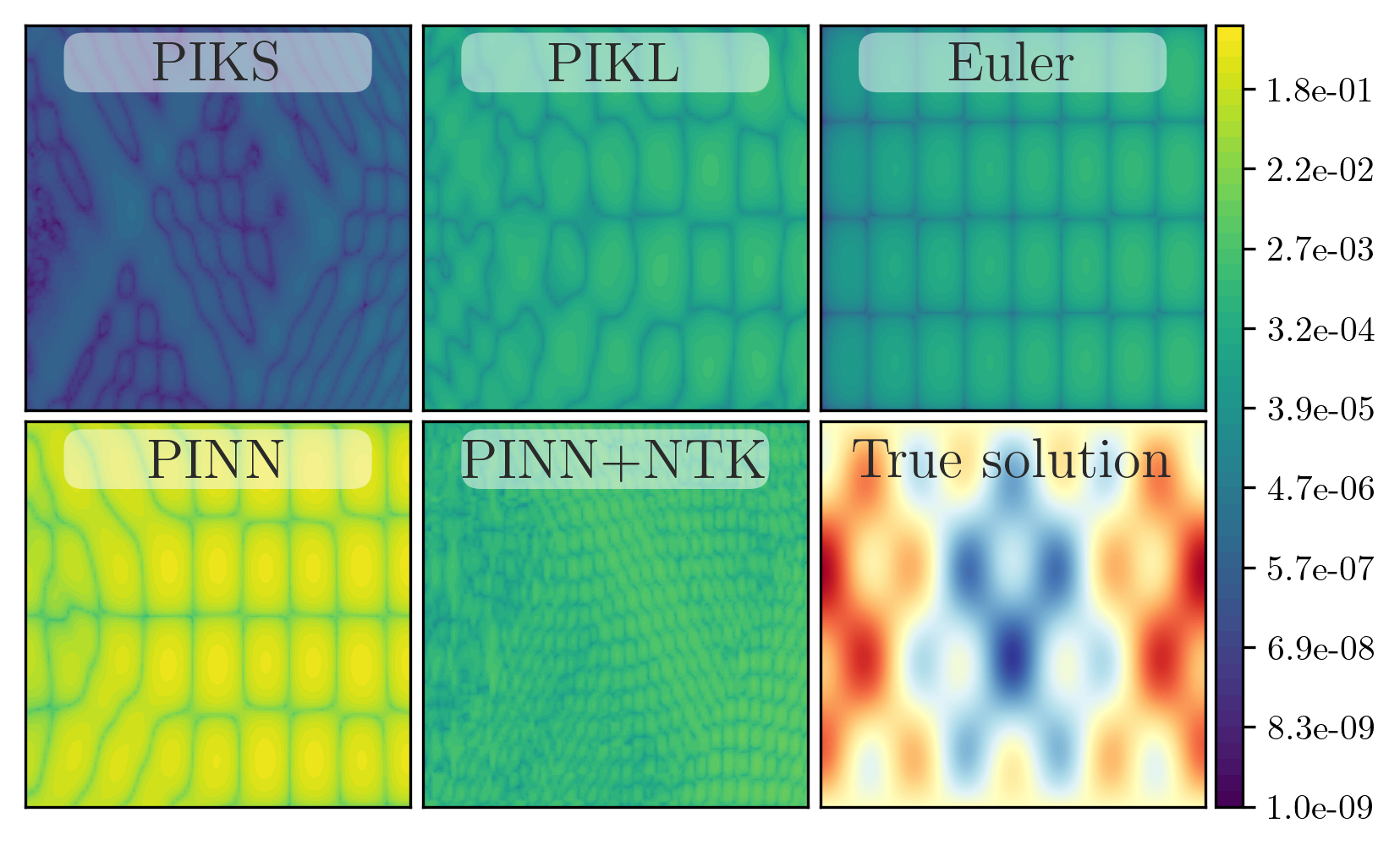} 
        \caption{Errors on the 1D wave equation. The color-scale is logarithmic and at this scale of errors only the standard PINN produces a solution which looks qualitatively incorrect.}
        \label{fig:1dwave-clean}
    \end{minipage}
\end{figure}

\subsection{PDE settings}
We showcase PIKS for solving PDEs with three examples. The first two focus on a well-known failure mode of PINNs: they struggle to learn functions with a large range of frequency components~\citep{rahaman2019spectral}, comparing also against other kernel estimators. In the last we focus on misspecification and compare against a FEM solver. 

\paragraph{Convection equation}
Following the setup of~\citet{krishnapriyan2021characterizing}, we take the 1D convection equation with smoothness controlled by parameter $\beta$. This PDE, whose solution is $f(t, x) = \sin(x - \beta t)$, is defined as
\begin{equation}
    \begin{cases}
\forall x \in [0, 2\pi], t\in[0, 1], &\partial_t f(t, x) + \beta \partial_x f(t, x) = 0  \\
\forall x \in [0, 2\pi], &f(0, x) = \sin(x) \\
\forall t \in [0, 1], &f(t, 0) = f(t, 2\pi).
    \end{cases}
\end{equation}
\citet{krishnapriyan2021characterizing} show that PINN performance degrades with increasing $\beta$ and propose \emph{curriculum training} to improve on this, while \citet{doumeche2024physics} show that kernel-based \emph{PIKL} is accurate to numerical precision. 
In \cref{tab:convection}, we show that, even with a standard RBF kernel and the same number of boundary points (100), PIKS achieves a relative RMSE nearly four orders of magnitude below that of curriculum-trained PINNs and approaches the accuracy of PIKL, which retains an edge in this experiment.
Note that the periodic boundary condition needed for this problem can be implemented as an extra linear constraint on the estimator i.e., $u(\cdot) = \ldots + \sum_{i=1}^n \gamma_i (K(x_i, \cdot) - K(\overline{x_i}, \cdot))$ where $\overline{x_i}$ is simply $x_i$ on the other side of the boundary.

\begin{table}
    \centering
    \begin{minipage}[c]{0.42\textwidth}
        \caption{Estimator comparison on the convection equation. Data for PINNs taken from \citet{krishnapriyan2021characterizing}, for PIKL from \citet{doumeche2024physics}.}
        \label{tab:convection}
        \resizebox{\textwidth}{!}{
        \begin{tabular}{@{}lc@{}}
            \toprule
             Method & Rel. RMSE ($\beta = 30$) \\
             \midrule
             Vanilla PINN  & $8.97 \times 10^{-1}$ \\
             Curriculum PINN & $2.02 \times 10^{-2}$ \\
             PIKL estimator & $0.91 \times 10^{-7}$ \\
             PIKS estimator & $2.18 \times 10^{-6}$ \\
             \bottomrule
        \end{tabular}}
    \end{minipage}\hfill
    \begin{minipage}[c]{0.55\textwidth}
        \caption{Solver comparison on the 1D wave equation. Noisy initial conditions use $\sigma=0.1$. Standard deviation is computed on 10 repetitions.}
        \label{tab:1dwave}
        \resizebox{\textwidth}{!}{
        \begin{tabular}[b]{@{}lcc@{}}
            \toprule
            Method & Rel. RMSE & Rel. RMSE (noisy IC) \\
            \midrule
            Euler & \num{6e-4} & \num{1.3(1)e-1} \\
            CN &    \num{6.4e-2} & \num{6.9(5)e-2}  \\
            PINN &  \num{4.1(3)e-1} & \num{4.1(2)e-1} \\
            PINN+NTK & \num{4.9(36)e-3} & \num{1.4(5)e-2} \\
            PIKL &  \num{8.7(1)e-4} & \num{2.8(12)e-2} \\
            PIKS &  \num{1.0(4)e-6} & \num{1.4(2)e-2} \\
            \bottomrule
        \end{tabular}}
    \end{minipage}
\end{table}

\paragraph{1D wave equation.}
Another case where PINNs struggle to have good accuracy is the high frequency 1D wave equation described in~\citet{wang22pinns}:
\begin{equation}
    \begin{cases}
      \forall x\in[0,1], t\in[0, 1], &\partial_{tt} f - c^2 \partial_{xx} f = 0 \\
      \forall x \in [0, 1], &f(0, x) = \sin(\pi x) + \sin(4\pi x) / 2 \\
      \forall x \in [0, 1], &\partial_t f(0, x) = 0 \\
      \forall t \in [0, 1], &f(t, 0) = f(t, 1) = 0.
    \end{cases}
\end{equation}
Under both noiseless and noisy conditions we compare PIKS with standard PDE solvers (Euler and Crank–Nicolson), standard PINNs, PINNs augmented with the NTK kernel~\citep{wang22pinns} and PIKL~\citep{doumeche2024physics}. 
The results presented in \cref{tab:1dwave} show that PIKS performs better than all other methods on both clean and noisy data. Note that this is with a RBF kernel which only required small-scale tuning of regularization and length-scale. For this experiment PIKS used \num[{scientific-notation = false}]{22000} data points in total, PIKL and the classic solvers used $\sim\num[{scientific-notation = false}]{100000}$ while PINNs used $\sim\num[{scientific-notation = false}]{24000000}$ points. \Cref{fig:1dwave-clean} compares the error of the different methods on a logarithmic scale.

\paragraph{Poisson equation.}
The final comparison is against a FEM solver~\citep{dolfinx} with piecewise linear elements, with a focus on how accuracy scales with the training-set size and with the amount of noise.
By taking multiple Poisson PDEs with decreasing levels of smoothness (more details about the definitions available in \cref{sec:exp-appendix}) we can additionally evaluate the performance of PIKS in a misspecified setting. 
We take functions from the Matérn family which are parametrized by $\nu$. A $\nu$-Matérn function is $\lceil\nu\rceil - 1$ times differentiable and since PIKS is used with a Gaussian kernel, only the $\nu = \infty$ target matches
the smoothness of the Gaussian kernel, with smaller $\nu$ signifying stronger misspecification.
Meanwhile the FEM solver uses linear elements and as can be seen in \cref{fig:fem-incdata} its performance does not depend on the function smoothness.
PIKS performance instead gradually degrades based on the level of misspecification, from the well-specified setting where PIKS reaches perfect accuracy with as few as 1000 points (an RMSE difference with FEM of 10 orders of magnitude) down to $\nu=0.5$ where kernel learning performs essentially on par with FEM. We note that FEM scales much better with dataset size, but even with ten times as much data its performance is far from kernel regression.
Similar reasoning follows with noisy data (see \cref{fig:fem-incnoise}): when noise is small, PIKS greatly outperforms FEM (for a moderately misspecified problem, $\nu=1.5$); as noise increases the two algorithms attain equally accurate estimates.

\begin{figure}
    \centering
    \begin{minipage}[t]{0.485\textwidth}
        \includegraphics[width=\linewidth]{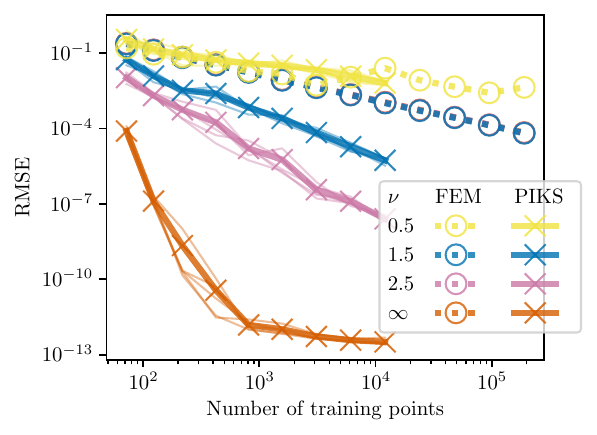}
        \caption{FEM vs. PIKS on noiseless data of different smoothness. All FEM results apart from $\nu=0.5$ overlap at the top of the plot. PIKS with $\nu = \infty$ plateaus at numerical precision.}
        \label{fig:fem-incdata}
    \end{minipage}\hfill
    \begin{minipage}[t]{0.485\textwidth}
        \includegraphics[width=\linewidth]{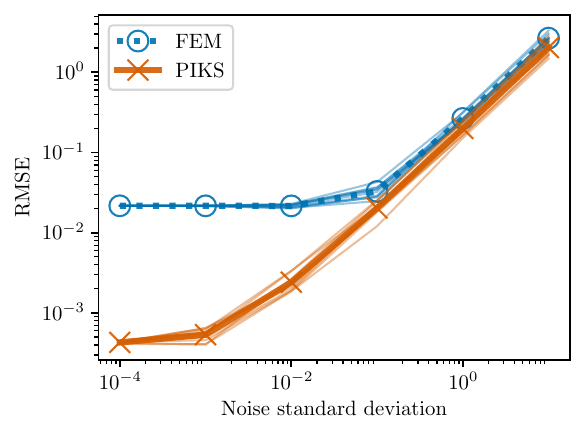}
        \caption{FEM vs. PIKS with increasing noise on the boundary conditions. 500 data points were used; the true function comes from a Matérn 3/2 kernel.}
        \label{fig:fem-incnoise}
    \end{minipage}
\end{figure}

\section{Conclusions}

We formulated and analyzed PIKS, a kernel-based framework for physics-informed learning that incorporates linear differential constraints into the standard regression objective. The resulting estimator can be computed in closed form by solving a linear system. Our principal contribution is to establish learning guarantees in the misspecified setting: PIKS is universally consistent, jointly recovering $u^*$ and $\cD u^*$ asymptotically even when the target function $u^*$ lies outside the native RKHS. We verify the required assumptions in Sobolev-space settings under two sampling schemes. Under suitable source conditions, we also derive convergence rates for estimating both function values and derivatives. Small-scale experiments illustrate these results, showing that PIKS can learn even under severe misspecification and can be competitive with existing methods.

Several questions remain open. The experiment of \Cref{sec:expgrad} suggests that derivative observations improve the estimation of the function values themselves ($n$ value plus $n$ gradient points outperform $2n$ value points) but our rate does not capture this effect. It would be interesting to establish theoretically when the two data sources influence each other, in either direction. On the algorithmic side, introducing separate weights for the value and differential loss terms could accommodate differences in noise levels, sample sizes, and physical scales. From a computational perspective, PIKS shares the principal limitation of standard kernel methods: solving the resulting linear system scales poorly with the dataset size. Developing efficient approximations is therefore an important direction for enabling large-scale applications. Extending our method to nonlinear PDEs would be an important and challenging research direction, as well as considering inverse problems for PDEs with kernel methods.

\section*{Acknowledgements}
The authors thank Rayan Autones for his contributions to the early experimental stages of the project, his exploration of the relevant literature, and many helpful discussions.
This material is based upon work supported by the Air Force Office of Scientific Research under award number FA8655-23-1-7083. The research was supported in part by the MIUR Excellence Department Project awarded to Dipartimento di Matematica, Università di Genova, CUP D33C23001110001.

\appendix

\section{Notation and RKHS basics}
\subsection{Notation}
We denote by $\NN$ the set of nonnegative integers including $0$. For functions $\RR^d \rightarrow \RR$, where $d, m \in \NN$, and for $\alpha = (\alpha_1, \dots , \alpha_d) \in \NN_0^d$ we denote by $\partial^\alpha$ the multi-index derivative
\[\partial^\alpha = \frac{\partial^{\alpha_1}}{\partial x_1^{\alpha_1}}  \cdots \frac{\partial^{\alpha_d}}{\partial x_d^{\alpha_d}}.\]
We also define the modulus $|\alpha| = \alpha_1 + \cdots + \alpha_d$. For a function of two variables $G(x,y)$, defined over $\RR^{d} \times \RR^d$, we denote by $\partial^\alpha_1 G(x,y)$ and $\partial^\alpha_2 G(x,y)$  the derivative $\partial^\alpha$ applied to $G(x,y)$ as a function of $x$ and $y$, respectively. For $\alpha, \beta \in \NN_0^d$, we denote by 
\begin{equation}\label{df:bilateral-derivative}
\partial^{\alpha, \beta} = \partial^\alpha_1 \partial^\beta_2
\end{equation}
the bilateral multi-index derivative. In particular, if we see a function of two $d$-dimensional variables as a function of one $2d$-dimensional variable, we can write
\[\partial^{\alpha, \beta} = \frac{\partial^{\alpha_1}}{\partial x_1^{\alpha_1}} \cdots \frac{\partial^{\alpha_d}}{\partial x_d^{\alpha_d}} \cdot \frac{\partial^{\beta_1}}{\partial x_{d+1}^{\beta_1}}  \cdots \frac{\partial^{\beta_d}}{\partial x_{2d}^{\beta_d}}.\]
For $ a \leq b \in \NN$, we denote by $\lb a , b \rb$ the set $\{ n \in \NN, a \leq n \leq b \}$.

\subsection{RKHS basics}\label{sec:rkhs-basics}
We consider a compact set $\Input \subset \RR^d$, for some integer $d \geq 1$, such that $\Input$ is the closure of its nonempty interior. We consider a continuous kernel $K : \cX \times \cX \to \RR$. We denote by $\cH$ the RKHS associated to $K$, which is in particular a set of functions $\Input \rightarrow \RR$.

For $x \in \Input$, we define 
\[ \function{K_x}{\Input}{\RR}{y}{K(x,y).}\]
It is well-known that the elements $K_x$ belong to the RKHS $\cH$ (in fact, the span of all the $K_x$ for $x \in \Input$ is dense in $\cH$). They satisfy the \emph{reproducing property}:
\begin{equation}\label{eq:eq:reproducing-property2}\forall f \in \cH, \forall x \in \Input, \quad  \langle f , K_x \rangle_\cH = f(x) .\end{equation}

\begin{df}
 Assume that $\cX \subset \RR^d$ is compact and is equal to the closure of its interior. For $s \in \NN$, we define $C^s(\cX)$ as the space of functions $u : \Input \to \RR$ such that $u \in C^s(\Int \Input)$ and for all multi-indices $\alpha \in \NN_0^d$ with $|\alpha| \leq s$, the derivative $\partial^\alpha u : \Int \Input \to \RR$ can be extended to a continuous function of $\Input$.
\end{df}

\begin{prop}[Reproducing property for the derivatives] \label{prop:derivatives-reproducing-prop}
    Assume that $\cX \subset \RR^d$ is compact and is equal to the closure of its interior, and assume that $K \in C^{2s}(\Input \times \Input)$. Then, any function $u$ in $\cH$ is $s$ times continuously differentiable on $\Input$, i.e. we have a natural embedding $\cH \hookrightarrow C^s(\Input)$.
     
     Furthermore, for any index $\alpha \in \NN_0^d$ such that $|\alpha| \leq s$, for any $x \in \Input$, the function 
     \begin{equation}\label{eq:definition-of-J^alpha-apdx}
         K_x^\alpha : y \mapsto \partial^\alpha_1 K(x,y)
     \end{equation} belongs to $\cH$ and we have 
    \begin{equation}\label{eq:derivatives-representer}\langle u, K_x^\alpha \rangle_\cH = \partial^\alpha u(x).
    \end{equation}
\end{prop}
\noindent \Cref{prop:derivatives-reproducing-prop} is proved in \cite[Theorem 1]{zhou2008derivative}.

We now restate and prove \Cref{lem:smooth-kernel-implies-bounded-diff-eval} from \Cref{sec:PIKS}.
\begin{lem}\label{lem:smooth-kernel-implies-bounded-diff-eval-apdx}
Assume that $\cX$ is compact and satisfies $\overline{\Int{\cX}} = \cX$, and consider a linear differential operator of the form 
\begin{equation}\label{eq:diff-op-order-s-apdx}
    \cD = \sum_{|\alpha| \leq s} c_\alpha \partial^\alpha,\end{equation}
 for some integer $s\geq1$ and where the $c_\alpha \colon \cX \rightarrow \RR$ are continuous coefficient functions.
    If $K \in C^{2s}(\cX \times \cX)$, then \Cref{ass:diff-eval} holds. Furthermore the representer $\KD_x$ is available in closed form:
\begin{equation}\label{eq:PDE-representer-expression-apdx}
    \KD_x: y \mapsto \mathcal{D}_1 K(x,y),
\end{equation}
where $\cD_1$ means that we applied the operator $\cD$ with respect to the first variable of $(x,y) \mapsto K(x,y)$.
\end{lem}

\begin{proof}
    For any $x \in \Input$, $\cD u(x) = \sum_{|\alpha| \leq s} c_\alpha(x) \partial^\alpha u(x)$ is a linear combination of the functionals $u \mapsto \partial^\alpha u(x)$, which are bounded according to \Cref{prop:derivatives-reproducing-prop}, and it is thus a bounded functional. We obtain the representer by linear combination of the representers $K_x^\alpha$:
    \[K^\cD_x =  \sum_{|\alpha| \leq s} c_\alpha(x) K_x^\alpha.\] 
\end{proof}

\section{PIML with kernels}\label{sec:PIML-kernels-apdx}

In this section, we work under \Cref{ass:diff-eval}. We define the \emph{sampling operator}  
\begin{equation}\label{definition-sampling-op-apdx}\function{\Ah}{\cH}{\RR^n \times \RR^m}{u}{\left(u(x_i)_{i=1}^n , \cD u(z_j)_{j=1}^m \right),}\end{equation}
where $\RR^n \times \RR^m$ is equipped with the normalized Euclidean norm 
\[\|(a,b)\|_{n,m}^2 = \|a\|_n^2 + \|b\|_m^2 = \frac{1}{n}\sum_{i=1}^n a_i^2 + \frac{1}{m}\sum_{j=1}^m b_j^2.\]
The sampling operator $\widehat A$ is a random bounded operator. Indeed, it depends on the random samples $x_i, 1 \leq i \leq n$ and $z_j , 1 \leq j \leq m$, and it is bounded thanks to \eqref{eq:reproducing-property} and \eqref{eq:diff-reproducing-property}.

Let us compute the adjoint $\Ah^*$. Let $u \in \cH$ and $(a,b) \in \RR^n \times \RR^m$. 
\begin{align*}
    \langle \Ah^* (a,b), u \rangle_\cH = \langle (a,b) , \Ah u \rangle_{n,m} & =  \frac{1}{n} \sum_{i=1}^n a_i u(x_i) + \frac{1}{m} \sum_{j=1}^m b_j \cD u(z_j)  \\
    & =  \frac{1}{n} \sum_{i=1}^n \langle a_i K_{x_i}, u \rangle_\cH + \frac{1}{m} \sum_{j=1}^m \langle b_j K^\cD_{z_j} , u \rangle_\cH \\
    & = \left\langle \frac{1}{n} \sum_{i=1}^n  a_i K_{x_i} + \frac{1}{m} \sum_{j=1}^m  b_j K^\cD_{z_j} \ , \ u \right\rangle_\cH .
\end{align*}
This proves that
\begin{equation}\label{eq:expression-of-Ahat*-apdx}
    \Ah^* (a,b) = \frac{1}{n} \sum_{i=1}^n  a_i K_{x_i} + \frac{1}{m} \sum_{j=1}^m  b_j K^\cD_{z_j}.
\end{equation}

Let us now denote
\[Y = \left((y_i)_{i=1}^n, (w_j)_{j=1}^m\right) \in \RR^n\times \RR^m.\]
We can reformulate the physics-informed empirical risk \eqref{eq:PIKS-estimator} defined in \Cref{sec:PIKS} as
\begin{equation}\label{reformulated-loss-apdx}
    \widehat R_\lambda(u) = \| \Ah u - Y \|_{n,m}^2 + \lambda \| u \|_\cH^2.
\end{equation}
For any $\lambda > 0$, the empirical risk \eqref{reformulated-loss-apdx} admits a unique minimizer. Furthermore, this minimizer can be expressed in closed form as
\begin{equation}\label{expression-of-u-lambdahat:eq}
    \uhl : = \left(\Ah^* \Ah + \lambda I\right)^{-1} \Ah^*Y. 
\end{equation}
    Indeed, the function $\cL_\lambda$ is strongly convex on the RKHS $\cH$. By solving $\nabla \cL_\lambda(u) = 0$, we find $\uhl$ as the unique solution, which shows that it is the unique minimizer. This proves that the PIKS estimator \eqref{eq:PIKS-estimator} is well-defined.

\begin{prop}\label{prop:closed-from-expression-apdx}
The estimator $\uhl$ can be decomposed as
    \begin{equation}\label{closed-form-expression-eq:apdx}
        \uhl = \sum_{i=1}^n \alpha_i K_{x_i} + \sum_{j=1}^m \beta_j K^\cD_{z_j},
    \end{equation}
    where $(\alpha, \beta) = \left( \Kbf + \lambda \Jbf  \right)^{-1} Y \ \in \RR^{n+m}$, with $\Jbf$ the diagonal $(n+m) \times (n+m)$ matrix satisfying 
\[\Jbf _{i,i} = \begin{cases} n & \text{if } 1 \leq  i \leq n \\ m & \text{if } n+1 \leq i \leq n+m\end{cases}\]
and $\Kbf$ the kernel matrix, which is a $(n+m) \times (n+m)$ block matrix:

\[\Kbf = \begin{pmatrix}[2] \Abf & \Cbf \\ 
    \Cbf^\top & \Bbf \end{pmatrix} , \]
where the blocks are the following:
\begin{align*}
& \Abf \in \RR^{n \times n}, \quad \Abf_{i,i'}  = \langle K_{x_i} , K_{x_{i'}}\rangle_\cH, \quad \forall i,i' \in \lb 1 , n \rb;\\
& \Bbf \in \RR^{m \times m}, \quad  \Bbf_{j,j'} = \langle K^\cD_{z_j}, K^\cD_{z_{j'}} \rangle_\cH, \quad \forall j,j' \in \lb 1 , m \rb; \\
& \Cbf  \in \RR^{n \times m}, \quad  \Cbf_{i,j'} = \langle K_{x_i} , K^\cD_{z_{j'}} \rangle_\cH , \quad \forall i \in \lb 1 , n \rb, \forall j' \in \lb 1, m \rb.
\end{align*}
\end{prop}

\begin{proof} First observe with \eqref{definition-sampling-op-apdx} and \eqref{eq:expression-of-Ahat*-apdx} that $\Ah \Ah^* = \Kbf \Jbf^{-1}$. Then, using \eqref{expression-of-u-lambdahat:eq} and the push-through identity, we have
    \begin{align*}
        \uhl & = \left(\Ah^* \Ah + \lambda I\right)^{-1} \Ah^*Y \\
        & = \Ah^* \left( \Ah \Ah^* + \lambda I \right)^{-1} Y \\
        & = \sum_{i=1}^n \alpha_i K_{x_i} + \sum_{j=1}^m \beta_j K^\cD_{z_j},
    \end{align*}
where we see with \eqref{eq:expression-of-Ahat*-apdx} that 
\[(\alpha, \beta) = \Jbf^{-1} \left( \Ah \Ah^* + \lambda I \right)^{-1} Y = \Jbf^{-1} \left( \Kbf \Jbf^{-1}+ \lambda I \right)^{-1} Y =  \left(  \Kbf+ \lambda \Jbf \right)^{-1} Y.\] 
\end{proof}

\section{Theoretical analysis}\label{sec:theoretical-analysis-apdx}
\subsection{Assumptions}\label{sec:theory-assumptions-apdx}

In this section, we make a slight change in the sampling setting with respect to \Cref{sec:prob}, by temporarily forgetting about the target $u^*$ and replacing it with two decoupled constraints. This leads to a more general setting, which we will later specialize back to the setting of this paper to obtain the main results. Consider indeed the same sampling setting as in \Cref{sec:prob}, with a compact set $\cX \subset \RR^d$ and two variables $X$ and $Z$ over $\cX$ with respective probability distributions $\rho_X$ and $\rho_Z$. We then consider two functions $h \in L^2(\rho_X) \cap L^\infty(\rho_X)$, $q \in L^2(\rho_Z) \cap L^\infty(\rho_Z)$, and instead of \eqref{eq:data-generation}, we now consider
\begin{equation}
        Y = h(X) + \epsilon, \qquad  W = q (Z) + \eta.
\end{equation}
The datasets $(x_i,y_i)_{i=1}^n$ and $(z_j,w_j)_{j=1}^m$ are then defined as i.i.d.~copies of $(X,Y)$ and $(Z,W)$ respectively, and the PIKS estimator $\widehat u_\lambda$ is defined --- as in \Cref{sec:prob} --- as:
\begin{equation}\label{eq:PIKS-estimator-apdx}
\widehat u_\lambda : = \argmin_{u \in \cH} \widehat R_\lambda(u), \qquad \widehat R_\lambda(u) =\frac{1}{n}\sum_{i=1}^n ( u (x_i) - y_i)^2 + \frac{1}{m}\sum_{j=1}^m (\cD u(z_j) - w_j)^2 + \lambda \|u \|^2_\cH.
\end{equation}
As we will see later, under our working assumptions, for any $u \in \cH$, we have $u \in L^2(\rho_X)$ and $\cD u \in L^2(\rho_Z)$ which allows us to study the error
\[\mathcal{E}(u) : =  \sqrt{\| u - h \|_{L^2(\rho_X)}^2 + \| \cD u - q \|_{L^2(\rho_Z)}^2} .\]

Throughout \Cref{sec:theoretical-analysis-apdx}, we work under \Cref{ass:diff-eval,ass:bounded-features,ass:bounded-data}. In particular, \Cref{ass:bounded-data} guarantees that there exist $M_y, M_w > 0$ such that almost surely we have $|y_i| \leq M_y$ and $|w_j| \leq M_w$.

\subsection{Error decomposition}

As is standard in theory of kernel methods, we introduce in this section an intermediate function $u_\lambda \in \cH$, which allows us to decompose the error between an estimation term and an approximation term.

\begin{prop}
The two linear operators
\[\function{S_{\rho_X}}{\cH}{L^2(\rho_X)}{u}{[u]_{\rho_X},}\]
and
\[\function{D_{\rho_Z}}{\cH}{L^2(\rho_Z)}{u}{[\cD u]_{\rho_Z}.}\]
are well-defined and bounded.
\end{prop}
\begin{proof}
For any $u \in \cH$, for any $x \in \supp(\rho_X)$, we find using \Cref{ass:bounded-features} that
\[|u(x)| = | \langle u , K_x \rangle_\cH | \leq \| u \|_\cH \| K_x \|_\cH \leq \kappa \| u \|_\cH ,\]
and thus we have
\[\|u\|_{L^2(\rho_X)}^2= \int_{\Input} |u(x)|^2 d\rho_X(x) \leq \kappa^2 \|u\|_\cH^2.\]
This shows that $[u]_{\rho_X}$ belongs to $L^2(\rho_X)$, i.e. $S_{\rho_X}$ is well defined, and $S_{\rho_X}$ is bounded with $ \|S_{\rho_X} \|_{\op} \leq \kappa$. In a similar way, for any $u \in \cH$, for any $z \in \supp(\rho_Z)$, we find using again \Cref{ass:bounded-features} that
\[|\cD u (z)| = | \langle u , K^\cD_z \rangle_\cH | \leq \| u \|_\cH \| K^\cD_z \|_\cH \leq \kappa_\cD \| u \|_\cH .\]
We thus have
\[\|D_{\rho_Z}u\|_{L^2(\rho_Z)}^2 = \int_{\Input} |\cD u(z)|^2 d\rho_Z(z) \leq \kappa_\cD^2 \|u\|_\cH^2,\]
which shows that $[\cD u]_{\rho_Z}$ belongs to $L^2(\rho_Z)$, i.e. $D_{\rho_Z}$ is well defined, and $D_{\rho_Z}$ is bounded with $ \|D_{\rho_Z} \|_{\op} \leq \kappa_\cD$.
\end{proof}

We can then define the operator
\begin{equation}\label{eq:def-operator-A_rho-apdx}
    \function{A_\rho}{\cH}{L^2(\rho_X) \times L^2(\rho_Z)}{u}{(S_{\rho_X}u, D_{\rho_Z}u),}
\end{equation}
where $L^2(\rho_X) \times L^2(\rho_Z)$ is equipped with the standard (Hilbertian) product norm
\begin{equation}\label{rho-norm-def}\|(f_1, f_2)\|_\rho^2 = \|f_1\|_{L^2(\rho_X)}^2 + \|f_2\|_{L^2(\rho_Z)}^2 = \int_{\Input} |f_1(x)|^2 d\rho_X(x) + \int_{\Input} |f_2(z)|^2 d\rho_Z(z),\end{equation}
where we also denote by $\langle \cdot , \cdot \rangle_\rho$ the corresponding scalar product. 
Since $S_{\rho_X}$ and $D_{\rho_Z}$ are bounded, $A_\rho$ is bounded with norm 
\begin{equation}\label{operator:norm:of:A}
\|A_\rho \|_{\op} \leq \sqrt{\|S_{\rho_X} \|_{\op} ^2 +  \|D_{\rho_Z} \|_{\op}^2}\leq \sqrt{\kappa^2 + \kappa_\cD^2}.\end{equation}

\begin{prop}
For $\lambda > 0$, we define the \emph{population estimator} as
\[u_\lambda = \argmin_{u \in \cH} \| A_\rho u - (h,q) \|_\rho^2 + \lambda \|u\|_\cH^2.\]
Its expression is
\begin{equation}\label{expression-of-u-lambda}
    u_\lambda = \left( A_\rho^*A_\rho + \lambda I\right)^{-1} A_\rho^* (h,q).
\end{equation}
\end{prop}
\begin{proof}
    The regularized population risk $R_\lambda (u) = \| A_\rho u - (h,q) \|_\rho^2 + \lambda \|u\|_\cH^2$ is strongly convex on the RKHS $\cH$. By solving $\nabla R_\lambda(u) = 0$, we find $\ul$ as the unique solution, which shows that it is the unique minimizer.
\end{proof}

We can now decompose the error as 
\begin{equation}\label{eq:error-decomposition} \cE(u) \ = \ \|A_\rho \widehat u_\lambda - (h,q) \|_\rho \quad \leq \quad \underbrace{\| A_\rho \widehat u_\lambda - A_\rho u_\lambda \|_\rho}_{\text{estimation error}} \ + \ \underbrace{\| A_\rho u_\lambda - (h,q) \|_\rho}_{\text{approximation error}}.\end{equation}
The \emph{approximation error} is studied in \Cref{sec:approx-error-apdx} while the \emph{estimation error} is studied in \Cref{sec:estimation-error-apdx}.

\subsection{Operator properties}\label{sec:operator-properties-apdx}

In this section we define and study the integral and covariance operators that are key in the theoretical proofs. 

Let us define the \emph{integral operator} as
\begin{equation}\label{eq:integral-operator-definition}
    L = A_\rho A_\rho^* : L^2(\rho_X) \times L^2(\rho_Z) \rightarrow L^2(\rho_X) \times L^2(\rho_Z).
\end{equation}
Integral operators are common objects in kernel ridge regression analysis, where they can be used to obtain convergence results. We use the integral operator \eqref{eq:integral-operator-definition} for the same purposes. The following proposition establishes some of its basic properties. 
\begin{prop}\label{properties-of-integral-operator-apdx}
For $\rho_X$-almost all $x$ and $\rho_Z$-almost all $z$, we have
\begin{equation*}
    \begin{split} L (f,g)(x,z) = \left( \int_{\Input} \langle K_y , K_x \rangle_\cH \, f(y) d \rho_X(y) + \int_{\Input} \langle K^\cD_t , K_x \rangle_\cH \, g(t) d\rho_Z (t)\ , \right. \\
    \left. \int_{\Input} \langle K_y , K^\cD_z \rangle_\cH \, f(y) d \rho_X(y) + \int_{\Input} \langle K^\cD_t , K^\cD_z \rangle_\cH \, g(t) d\rho_Z(t)\right).
\end{split}
\end{equation*}

Furthermore, $L$ is positive and trace class, and it admits a decomposition in an orthonormal system of eigenvectors $(f_i, g_i)_{i \in I}$.
\begin{equation}\label{decomposition:integral:op:ONS}
L(f,g) = \sum_{i \in I} \mu_i \langle (f,g) , (f_i,g_i) \rangle_\rho (f_i,g_i),
\end{equation}
where for all $i \in I$, $\mu_i > 0$.
\end{prop}

Before proving \Cref{properties-of-integral-operator-apdx}, let us first prove the following Lemma which gives the expression of $A_\rho^*$.
\begin{lem}For any $(f,g) \in L^2(\rho_X) \times L^2(\rho_Z) $, we have
    \begin{equation}\label{expression-of-A*-apdx}
         A_\rho^*(f,g) =  \int_{\Input} f(y) K_y d\rho_X(y) + \int_{\Input} g(t) K^\cD_t  d\rho_Z(t).
    \end{equation}
\end{lem}
\begin{proof}
Let $(f,g) \in L^2(\rho_X) \times L^2(\rho_Z) $. We have
\begin{equation*}
    \begin{aligned}
        \langle A_\rho^*(f,g) \, , \, \varphi \rangle_\cH & = \langle (f,g) \, , \, A_\rho \varphi \rangle_\rho \\
        & = \int_{\Input} f(y) \varphi(y) d\rho_X(y) + \int_{\Input} g(t) \cD \varphi(t) d\rho_Z(t) \\
        & = \int_{\Input} f(y) \langle K_y , \varphi \rangle_\cH d\rho_X(y) + \int_{\Input} g(t) \langle K^\cD_t , \varphi \rangle_\cH d\rho_Z(t) \\
        & = \left\langle \ \int_{\Input} f(y) K_y d\rho_X(y) + \int_{\Input} g(t) K^\cD_t  d\rho_Z(t) \ \ , \ \ \varphi \  \right\rangle_\cH, \\
    \end{aligned}
\end{equation*}
where we used the reproducing properties \eqref{eq:reproducing-property} and \eqref{eq:diff-reproducing-property}, and the last equality is justified by the fact that $y \mapsto f(y) K_y$ is Bochner integrable with respect to $\rho_X$ and $t \mapsto g(t) K^\cD_t$ is Bochner integrable with respect to $\rho_Z$. This proves formula \eqref{expression-of-A*-apdx}.
\end{proof}

\begin{proof}[Proof of \Cref{properties-of-integral-operator-apdx}]
Let us first establish the expression of $L$. Composing \eqref{expression-of-A*-apdx} with the expression \eqref{eq:def-operator-A_rho-apdx} of $A_\rho$, we obtain that for $\rho_X$-almost all $x$ and $\rho_Z$-almost all $z$, we have
\begin{equation*}
    \begin{split}
    L (f,g)(x,z) & = \left( \int_{\Input} f(y) K_y(x)  d \rho_X(y) + \int_{\Input} g(t) K^\cD_t(x) d\rho_Z (t)\ , \right. \\
   & \left. \cD \left[\int_{\Input} f(y) K_y  d \rho_X(y) \right](z) + \cD \left[\int_{\Input} g(t) K^\cD_t d\rho_Z(t)\right](z)\right) \\
   & = \left( \int_{\Input} \langle K_y , K_x \rangle_\cH \, f(y) d \rho_X(y) + \int_{\Input} \langle K^\cD_t , K_x \rangle_\cH \, g(t) d\rho_Z (t)\ , \right. \\
   & \left. \int_{\Input} \langle K_y , K^\cD_z \rangle_\cH \, f(y) d \rho_X(y) + \int_{\Input} \langle K^\cD_t , K^\cD_z \rangle_\cH \, g(t) d\rho_Z(t)\right).
   \end{split}
\end{equation*}
The second equality is the desired expression.

By definition, $L = A_\rho A_\rho^*$ is self-adjoint. Let us show that $A_\rho$ is Hilbert-Schmidt. Let $(e_i)_{i \in \NN}$ be an orthonormal basis of $\cH$, and let us prove that
\[\|A_\rho \|_{HS}^2 := \sum_{i\in \NN} \| A_\rho e_i \|_{\rho}^2 < + \infty. \]
Let us first observe that 
\begin{align*} \sum_{i\in \NN} \| A_\rho e_i \|_{\rho}^2  & = \sum_{i\in \NN} \left(\int_\cX |e_i(x)|^2 d\rho_X(x) + \int_\cX |\cD e_i(z)|^2 d \rho_Z(z)\right) \\
    & = \int_\cX \sum_{i\in \NN} |e_i(x)|^2 d\rho_X(x) + \int_\cX \sum_{i\in \NN} |\cD e_i(z)|^2 d \rho_Z(z) \\
& = \int_\cX \sum_{i\in \NN} |\langle e_i, K_x \rangle_\cH|^2 d\rho_X(x) + \int_\cX \sum_{i\in \NN} |\langle e_i, K^\cD_z \rangle_\cH|^2 d \rho_Z(z) \\
& \leq \kappa^2 + \kappa_\cD^2,\end{align*}
where the second equality is obtained by monotone convergence, and the final inequality comes from \Cref{ass:bounded-features}. This shows that $A_\rho$ is Hilbert-Schmidt, and as a consequence, $L$ is trace class.

To conclude the proof, trace class operators of separable Hilbert spaces (such as $L^2(\rho_X) \times L^2(\rho_Z)$) are compact, and since $L$ is also self-adjoint and positive, the spectral theorem guarantees that there exists an orthonormal system $(f_i,g_i)_{i \in I}$, there exist $\mu_i > 0$, $i \in I$ such that for any $(f,g) \in L^2(\rho_X) \times L^2(\rho_Z)$,
\begin{equation*}L(f,g) = \sum_{i \in I} \mu_i \langle (f,g) , (f_i,g_i) \rangle_\rho (f_i,g_i). \qedhere
\end{equation*}
\end{proof}

\begin{lem}\label{bound-hilbert-schmidt-norm-apdx}
   For any $x \in \cX$, we have $\| K_x \otimes K_x \|_{HS} \leq \kappa^2$ and $\| K^\cD_x \otimes K^\cD_x \|_{HS} \leq \kappa_\cD^2$.
\end{lem}
\begin{proof}
Recall that we work under \Cref{ass:bounded-features}. Since $\cH$ is a separable Hilbert space, we can consider a Hilbert basis $(e_i)_{i\in \NN}$ of $\cH$. We have by definition
\[K_x \otimes K_x (e_i) = \langle K_x, e_i \rangle_\cH K_x,\]
so
\[\sum_{i \in \NN} \|K_x \otimes K_x (e_i) \|_\cH^2 = \| K_x \|_\cH^2 \sum_{i \in \NN} \langle K_x, e_i \rangle_\cH^2 = \| K_x \|_\cH^4 \leq \kappa^4. \]
This proves that $K_x \otimes K_x$ is a Hilbert-Schmidt operator, with Hilbert-Schmidt norm $ \|K_x \otimes K_x \|_{HS} \leq \kappa^2$.

The proof of the second point is identical, replacing $K_x$ by $K^\cD_x$ and $\kappa$ by $\kappa_\cD$.
\end{proof}
Let us now define the \emph{covariance operators} 
\begin{equation}\label{eq:def-of-Sigma-and-Sigmah}\Sigmah = \Ah^* \Ah,  \qquad \qquad  \Sigma = A_\rho^*A_\rho.\end{equation}

\begin{lem}\label{covariance-operators-decomposition-apdx}
    We have
\begin{equation}\label{empirical-covariance-estimator-apdx}
\widehat \Sigma = \frac{1}{n} \sum_{i=1}^n K_{x_i} \otimes K_{x_i} + \frac{1}{m} \sum_{j=1}^m K^\cD_{z_j} \otimes K^\cD_{z_j},
\end{equation}
and
\begin{equation}\label{covariance-estimator-apdx}
    \Sigma = \EE_{x \sim \rho_X}[K_x \otimes K_x] + \EE_{z \sim \rho_Z}[K^\cD_z \otimes K^\cD_z ],
\end{equation}
where the expected values are defined as Bochner integrals in the space of Hilbert-Schmidt operators. 
\end{lem}
\begin{proof}
    Combining \eqref{definition-sampling-op-apdx} and \eqref{eq:expression-of-Ahat*-apdx}, we get that for all $u \in \cH$, we have 
    \begin{align*}\Sigmah u  =  \frac{1}{n} \sum_{i=1}^n u(x_i) K_{x_i} + \frac{1}{m} \sum_{j=1}^m \cD u(z_j) K^\cD_{z_j} & = \frac{1}{n} \sum_{i=1}^n \langle K_{x_i}, u \rangle_\cH K_{x_i} + \frac{1}{m} \sum_{j=1}^m \langle K^\cD_{z_j} , u \rangle_\cH K^\cD_{z_j} \\
    & = \left( \frac{1}{n} \sum_{i=1}^n K_{x_i} \otimes K_{x_i} + \frac{1}{m} \sum_{j=1}^m K^\cD_{z_j} \otimes K^\cD_{z_j}\right)u,
    \end{align*}
which proves \eqref{empirical-covariance-estimator-apdx}.

Combining \eqref{eq:def-operator-A_rho-apdx} and \eqref{expression-of-A*-apdx}, we get that for all $u \in \cH$, we have
    \begin{align*}\Sigma u = \int_{\Input} u(y) K_y d\rho_X(y) + \int_{\Input} \cD u(t) K^\cD_t  d\rho_Z(t) & =  \int_{\Input}( K_y \otimes K_y u )d\rho_X(y) + \int_{\Input}( K^\cD_t \otimes K^\cD_t u )d\rho_Z(t) \\
    & =  \left(\EE_{x \sim \rho_X}[K_x \otimes K_x] + \EE_{z \sim \rho_Z}[K^\cD_z \otimes K^\cD_z ]\right) u .
    \end{align*}
This proves \eqref{covariance-estimator-apdx}, provided we justify the last equality. Observe indeed that as proved in \Cref{bound-hilbert-schmidt-norm-apdx}, for any $x,z \in \cX$, $K_x \otimes K_x$ is a Hilbert-Schmidt operator of norm $\| K_x \otimes K_x\|_{HS} \leq \kappa^2 $ and $K^\cD_z \otimes K^\cD_z$ is a Hilbert-Schmidt operator of norm $\|K^\cD_z \otimes K^\cD_z\|_{HS} \leq \kappa_\cD^2 $. Since we have
\[\int_{\Input} \|K_y \otimes K_y\|_{HS} \ d\rho_X(y) \leq \kappa^2\]
and
\[\int_{\Input} \|K^\cD_t \otimes K^\cD_t \|_{HS} \ d\rho_Z(t) \leq \kappa_\cD^2,\]
the Bochner integrals $\EE_{x \sim \rho_X}[K_x \otimes K_x]$ and $\EE_{z \sim \rho_Z}[K^\cD_z \otimes K^\cD_z ]$ are well defined. Finally, we use the fact that the Bochner integral commutes with bounded operators to factorize the expression.
\end{proof}

\subsection{Approximation error}\label{sec:approx-error-apdx}

In this section, we focus on the approximation part of the decomposition \eqref{eq:error-decomposition}. In \Cref{sec:theory-assumptions-apdx}, we defined $(h,q) \in L^2(\rho_X) \times L^2(\rho_Z)$, let us now define $(h^\Pi, q^\Pi) \in L^2(\rho_X) \times L^2(\rho_Z)$ as the orthogonal projection of $(h, q)$ on $\overline{\ran A_\rho}$, which is a closed subspace of $L^2(\rho_X) \times L^2(\rho_Z)$. For the error bounds, we introduce a slightly modified version of \Cref{ass:src}, which applies to the projection $(h^\Pi, q^\Pi)$ of $(h,q)$ onto $\overline{\ran A_\rho}$. 
\begin{assumption}[Source condition, variant]\label{ass:src-var-apdx}
        There exists $r \in (0,1]$ such that $(h^\Pi, q^\Pi) \in \ran L^r$, i.e. there exists $(\widetilde f , \widetilde g) \in L^2(\rho_X) \times L^2(\rho_Z)$ such that $(h^\Pi, q^\Pi) = L^r(\widetilde f , \widetilde g) $.
\end{assumption}

\begin{prop}\label{approx:error:apdx}
    We have
    \[\| A_\rho u_\lambda - (h^\Pi, q^\Pi)\|_\rho \underset{\lambda \rightarrow 0}{\longrightarrow} 0.\]
    Furthermore, under \Cref{ass:src-var-apdx}, we have the rate
    \begin{equation}\label{approx-error-source-condition}
    \| A_\rho u_\lambda - (h^\Pi, q^\Pi)\|_\rho \leq \lambda^r \| (\widetilde f , \widetilde g) \|_\rho.
    \end{equation}
\end{prop}

\begin{proof}
Let us denote $L_\lambda = A_\rho A_\rho^* + \lambda I$. Recalling the definition \eqref{expression-of-u-lambda} of $u_\lambda$, we have
\begin{align*} u_\lambda & = (A_\rho^*A_\rho + \lambda I)^{-1} A_\rho^* (h, q)  \\
& = (A_\rho^*A_\rho + \lambda I)^{-1} A_\rho^* (h^\Pi, q^\Pi) \\
& = A_\rho^* L_\lambda^{-1} (h^\Pi, q^\Pi),
\end{align*}
where the second equality holds because $(h^\Pi, q^\Pi)$ is the orthogonal projection of $(h, q)$ onto $\overline{\ran{A_\rho}}$ and the third equality is the so-called \emph{push-through identity}.

We thus have
\begin{align} A_\rho u_\lambda - (h^\Pi, q^\Pi) & = A_\rho A_\rho^*L_\lambda^{-1} (h^\Pi, q^\Pi) - (h^\Pi, q^\Pi) \nonumber \\
    & = (L L_\lambda^{-1} - L_\lambda L_\lambda^{-1})(h^\Pi, q^\Pi) \nonumber \\
    & = - \lambda L_\lambda^{-1} (h^\Pi, q^\Pi). \label{approx-error-integral-op-formulation}
\end{align}
    Recall from \eqref{decomposition:integral:op:ONS} that there exists an orthonormal system of eigenvectors $(f_i, g_i)$ such that
\begin{equation*}
L(f,g) = \sum_{i \in I} \mu_i \langle (f,g) , (f_i,g_i) \rangle_\rho (f_i,g_i),
\end{equation*}
with $\mu_i>0$.
We thus see that $ (f_i , g_i)_{i\in I}$ is a Hilbert basis of $\overline{\ran{A_\rho}} = (\Ker L)^\perp$. Since by definition $(h^\Pi, q^\Pi) \in \overline{\ran A_\rho}$, we can decompose $(h^\Pi, q^\Pi)$ in the basis $ (f_i , g_i)_{i\in I}$, hence we can write
 \begin{equation*}
     \|A_\rho u_\lambda - (h^\Pi, q^\Pi) \|_\rho^2 = \|\lambda L_\lambda^{-1} (h^\Pi, q^\Pi) \|_\rho^2 = \sum_{i \in I} \left(\frac{\lambda}{\lambda + \mu_i}\right)^2 \langle (h^\Pi, q^\Pi) , (f_i , g_i) \rangle_\rho^2.
 \end{equation*}
 For all $ i \in I$, since $\mu_i> 0$, we have 
 \[\left(\frac{\lambda}{\lambda + \mu_i}\right)^2 \langle (h^\Pi, q^\Pi) , (f_i , g_i) \rangle_\rho^2 \underset{\lambda \rightarrow 0}{\longrightarrow} 0. \]
 By dominated convergence we thus have $\|A_\rho u_\lambda - (h^\Pi, q^\Pi) \|_\rho^2 \underset{\lambda \rightarrow 0}{\longrightarrow} 0 $.

  For the second point, using again the expression \eqref{approx-error-integral-op-formulation}, we have
    \begin{equation*}
        \begin{aligned}
    A_\rho \ul - (h^\Pi, q^\Pi) & = - \lambda L_\lambda^{-1} (h^\Pi, q^\Pi) \\
    & = - \lambda L_\lambda^{-1} L^r (\widetilde f,\widetilde g),
    \end{aligned}
    \end{equation*}
 and thus
\begin{equation*}
 \begin{aligned}
     \|A_\rho u_\lambda - (h^\Pi, q^\Pi) \|_\rho^2 & = \sum_{i\in I} \left(\frac{\lambda}{\lambda + \mu_i}\right)^2 \mu_i^{2r} \langle (\widetilde f,\widetilde g) , (f_i , g_i) \rangle_\rho^2 \\
      & = \sum_{i\in I} \left(\frac{\lambda^r\lambda^{1 - r} \mu_i^r}{\lambda + \mu_i}\right)^2 \langle (\widetilde f,\widetilde g) , (f_i , g_i) \rangle_\rho^2 \\
      & = \lambda^{2r}  \sum_{i\in I} \left(\frac{\lambda^{1 - r} \mu_i^r}{\lambda + \mu_i}\right)^2 \langle (\widetilde f,\widetilde g) , (f_i , g_i) \rangle_\rho^2 \\
      & \leq  \lambda^{2r}  \sum_{i\in I}  \langle (\widetilde f,\widetilde g) , (f_i , g_i) \rangle_\rho^2 \\
      & \leq  \lambda^{2r} \|(\widetilde f , \widetilde g )\|_\rho^2,
 \end{aligned}
 \end{equation*}
    where we observed that for any $r \in [0,1]$, we have either $\lambda^{1-r} \mu_i^r \leq \lambda $ or  $\lambda^{1-r} \mu_i^r \leq \mu_i $, so $\lambda^{1-r} \mu_i^r \leq \lambda + \mu_i .$
    \end{proof}

\subsection{Estimation error}\label{sec:estimation-error-apdx}

In this section, we want to bound the estimation part of the decomposition \eqref{eq:error-decomposition}, that is, $\| A_\rho\widehat u_\lambda - A_\rho u_\lambda \|_\rho$. To make proofs more compact, we will introduce the notation
\begin{equation}\label{eq:def-r-tilde}
    \tilde{r} = \begin{cases}
        0 & \text{if \Cref{ass:src-var-apdx} does not hold} \\
        r & \text{if \Cref{ass:src-var-apdx} holds with } r \in (0,1/2] \\
        1/2 & \text{if \Cref{ass:src-var-apdx} holds with } r \in (1/2,1]. 
    \end{cases}
\end{equation}
 When $\tilde{r} = 0$, we take as a convention $(\widetilde f , \widetilde g) := (h^\Pi, q^\Pi)$. Let us introduce the notation $\widehat \Sigma_\lambda = \Ah^* \Ah + \lambda I$ and $\Sigma_\lambda = A_\rho^*A_\rho + \lambda I$, which will be used at several moments and will help keeping the computations compact.

We begin by proving two technical lemmas.

\begin{lem}\label{lemma:population-cov-controls-emp-cov}
 Let us define the operators $\Sigma_X = \EE[K_x \otimes K_x]$ and $\widehat \Sigma_X = \frac{1}{n} \sum_{i=1}^n K_{x_i} \otimes K_{x_i}$, and define
    \begin{equation}
        B_X = \Sigma_\lambda^{-\frac{1}{2}} (\Sigma_X - \widehat \Sigma_X) \Sigma_\lambda^{-\frac{1}{2}}.
    \end{equation}
          Similarly, let $\Sigma_Z = \EE[K^\cD_z \otimes K^\cD_z]$ and $\widehat \Sigma_Z = \frac{1}{m} \sum_{j=1}^m K^\cD_{z_j} \otimes K^\cD_{z_j}$, and define
          \begin{equation}
            B_Z = \Sigma_\lambda^{-\frac{1}{2}} (\Sigma_Z - \widehat \Sigma_Z) \Sigma_\lambda^{-\frac{1}{2}}.
          \end{equation}
    Then, if $ \|B_X\|_{\op} \leq \frac{1}{4}$ and $\| B_Z \|_{\op} \leq \frac{1}{4}$, we have
    \[\| \Sigma_\lambda^{\frac{1}{2}} \widehat \Sigma_\lambda^{-\frac{1}{2}} \|_{\op} \leq \sqrt{2} \,.\]
\end{lem}
\begin{proof}
    Let us define the operator $B = \Sigma_\lambda^{-\frac{1}{2}} (\Sigma - \widehat \Sigma) \Sigma_\lambda^{-\frac{1}{2}}$. We observe that
\begin{equation*}
    \| \Sigma_\lambda^{\frac{1}{2}} \widehat \Sigma_\lambda^{-\frac{1}{2}} \|_{\op}^2 
    = \| \Sigma_\lambda^{\frac{1}{2}} \widehat \Sigma_\lambda^{-1} \Sigma_\lambda^{\frac{1}{2}} \|_{\op} 
    = \| (I - B)^{-1} \|_{\op} \,.
\end{equation*}
If $\| B \|_{\op} \le 1/2$, the Neumann series allows us to bound $\| (I - B)^{-1} \|_{\op} \le (1 - \|B\|_{\op})^{-1} \le 2$, which implies the statement of the lemma. 

We can write:
\begin{align*}
    B &= \Sigma_\lambda^{-\frac{1}{2}} (\Sigma_X - \widehat \Sigma_X) \Sigma_\lambda^{-\frac{1}{2}} + \Sigma_\lambda^{-\frac{1}{2}} (\Sigma_Z - \widehat \Sigma_Z) \Sigma_\lambda^{-\frac{1}{2}} \\
      &= B_X + B_Z \,.
\end{align*}
 Thus:
\begin{equation*}
    \| B \|_{\op} \le \| B_X \|_{\op} + \| B_Z \|_{\op} \le \frac{1}{4} + \frac{1}{4} = \frac{1}{2} \,. \qedhere
\end{equation*}
\end{proof}

\begin{lem}\label{Tropp-ineq-B_X-B_Z}
If we have 
\begin{equation}\label{Tropp-ineq-condition-lambda-X} 
    \lambda \geq \frac{112 \kappa^2}{3n} \log \left(\frac{n+4}{\delta}\right),\end{equation}
 then, with probability at least $1 - \delta$, we have
\[ \|B_X\|_{\op} \leq \frac{1}{4}.\]
Analogously, if we have 
\begin{equation}\label{Tropp-ineq-condition-lambda-Z}
    \lambda \geq \frac{112 \kappa_\cD^2}{3m} \log \left(\frac{m+4}{\delta}\right),\end{equation}
     then, with probability at least $1 - \delta$, we have  
\[\| B_Z \|_{\op} \leq \frac{1}{4}.\]
\end{lem}
\begin{proof}    
Let us first bound $B_X$. Let us denote $\Sigma_{X,\lambda} := \Sigma_X + \lambda I$. First note that
\[ \Sigma_\lambda \ = \ \Sigma + \lambda I \ = \ \Sigma_X + \Sigma_Z + \lambda I \ \succeq \ \Sigma_X + \lambda I,\]
so
\begin{align*} \| \Sigma_\lambda^{-\frac{1}{2}} \Sigma_X \Sigma_\lambda^{-\frac{1}{2}} \|_{\op} & \leq  \| \Sigma_\lambda^{-\frac{1}{2}}  \Sigma_{X,\lambda}^{\frac{1}{2}} \|_{\op} \| \Sigma_{X,\lambda}^{-\frac{1}{2}} \Sigma_X \Sigma_{X,\lambda}^{-\frac{1}{2}} \|_{\op} \| \Sigma_{X,\lambda}^{\frac{1}{2}}\Sigma_\lambda^{-\frac{1}{2}}  \|_{\op} \\
     & \leq \| \Sigma_{X,\lambda}^{-\frac{1}{2}} \Sigma_X \Sigma_{X,\lambda}^{-\frac{1}{2}} \|_{\op}. 
\end{align*}

The idea is now to apply Tropp's concentration inequality (\Cref{Tropp-inequality}) to the random variables $Z_i := U_i \otimes U_i$ with $U_i := \Sigma_{X,\lambda}^{-\frac{1}{2}} K_{x_i} $.
We have $\EE[Z_i] = \Sigma_{X,\lambda}^{-\frac{1}{2}} \Sigma_X \Sigma_{X,\lambda}^{-\frac{1}{2}} := T$. 
We have
\[ \| U_i \otimes U_i \|_{\op} 
    =  \| U_i \|_{\cH}^2 
    \le \frac{\|K_{x_i}\|_\cH^2}{\lambda} \leq \frac{\kappa^2}{\lambda} ,\]
so we can choose $R = \frac{\kappa^2}{\lambda}$. Now observe that
\[\EE[(U_i \otimes U_i - T)^2] = \EE[\|U_i\|_\cH^2 U_i \otimes U_i] - T^2 \preceq \EE[\|U_i\|_\cH^2 U_i \otimes U_i]  \preceq RT,\]
Let us define $S := RT$. We then have $\sigma^2 := \| S \|_{\op} \le \frac{\kappa^2}{\lambda} \| T \|_{\op} \le \frac{\kappa^2}{\lambda}$ (since $T \preceq I$). 
We define $\alpha = \frac{\|S\|_1}{\|S\|_{\op}} = \frac{(\|\Sigma_X\|_{\op} + \lambda)\|T\|_1}{\|\Sigma_X\|_{\op}}$. 
We can thus use \Cref{Tropp-inequality} to conclude that, with probability $1 - \delta$,
\[\|B_X\|_{\op} = \left\| \frac{1}{n} \sum_{i=1}^n Z_i - T \right\|_{\op} \leq \frac{\beta \kappa^2}{\lambda n} + \sqrt{\frac{3 \beta \kappa^2}{ \lambda n}}, \]
with $\beta = \frac{2}{3} \log \frac{4(\|\Sigma_X\|_{\op} + \lambda) \| T \|_1}{\|\Sigma_X \|_{\op} \delta}$. By taking $\lambda \geq \frac{56 \beta \kappa^2}{n}$ in the right member of the inequality, one can check that we obtain the desired $\|B_X\|_{\op} \leq \frac{1}{4}$. Such a condition on $\lambda$ is not explicit since $\beta$ itself depends on $\lambda$, but one can check that taking 
\begin{equation}
    \lambda \geq \frac{112 \kappa^2}{3n} \log \left(\frac{n+4}{\delta}\right)
\end{equation}
is enough to satisfy it, i.e. to have, with probability $1 - \delta$,
\[ \|B_X\|_{\op} \leq \frac{1}{4}.\]

We bound $B_Z$ by an identical argument for the second term. Utilizing the bound $\| K^\cD_z \|_{\cH} \le \kappa_{\cD}$, we obtain that if
\begin{equation}
    \lambda \geq \frac{112 \kappa_\cD^2}{3m} \log \left(\frac{m+4}{\delta}\right)
\end{equation}
then $\| B_Z \|_{\op} \le 1/4$ with probability $1 - \delta$.
\end{proof}

\begin{prop}\label{estimation-error-main-lemma} 
There exist four sequences of random variables $(U_n), (V_n), (W_m), (Z_m) \in \RR_+^\NN$ such that
\begin{itemize}
    \item[(i)]  for all $n,m \in \NN$, we have
    \[\| A_\rho\uhl - A_\rho\ul \|_\rho \quad \leq \quad  \| \Sigma_\lambda^{\frac{1}{2}} \widehat \Sigma_\lambda^{-\frac{1}{2}} \|_{\op} \left( \frac{1}{\lambda^{\frac{1}{2}}}U_n + \frac{1}{\lambda^{1-\tilde{r}}}V_n + \frac{1}{\lambda^{\frac{1}{2}}}W_m + \frac{1}{\lambda^{1-\tilde{r}}}Z_m \right);\]
    \item[(ii)]  for all $n \in \NN$, for all $\tau > 0$, with probability at least $1 - 2 e^{-\tau}$, we have
\begin{equation}\label{probability-bound-U_n}U_n \quad \leq \quad  C_U  \frac{\sqrt{\tau}}{ \sqrt{n}},
\end{equation}
where $C_U = 4  M_y \kappa$;
  \item[(iii)]  for all $n \in \NN$, for all $\tau > 0$, with probability at least $1 - 2 e^{-\tau}$, we have
\begin{equation}\label{probability-bound-V_n}V_n \quad \leq \quad C_V \frac{\sqrt{\tau}}{ \sqrt{n}},\end{equation}
where $C_V =  4 \kappa^2 (\kappa^2 + \kappa_\cD^2)^{r - \tilde r}\|(\widetilde f , \widetilde g)\|_\rho$;
\item[(iv)] for all $m \in \NN$, for all $\tau > 0$, with probability at least $1 - 2e^{-\tau}$, we have
\begin{equation}\label{probability-bound-W_m} W_m \quad \leq \quad C_W \frac{\sqrt{\tau}}{\sqrt{m}},
\end{equation}
where $C_W =  4 M_w \kappa_\cD$;
\item[(v)] for all $m \in \NN$, for all $\tau > 0$, with probability at least $1 - 2e^{-\tau}$, we have
\begin{equation}\label{probability-bound-Z_m} Z_m \quad \leq \quad   C_Z \frac{\sqrt{\tau}}{ \sqrt{m}},
\end{equation}
where $C_Z =  4 \kappa_\cD^2 (\kappa^2 + \kappa_\cD^2)^{r - \tilde r} \|(\widetilde f , \widetilde g)\|_\rho $.
\end{itemize}
\end{prop}
We prove \Cref{estimation-error-main-lemma} in this section. Let us denote 
\begin{equation}\label{eq:def-of-xi-and-xih}\widehat \xi = \Ah^* Y \qquad \qquad \xi = A_\rho^* (h, q) .\end{equation}
Recall the operators $\Sigmah = \Ah^* \Ah $ and $\Sigma = A_\rho^* A_\rho$ defined in \eqref{eq:def-of-Sigma-and-Sigmah} and the operators $\widehat \Sigma_\lambda = \Ah^* \Ah + \lambda I$ and $\Sigma_\lambda = A_\rho^*A_\rho + \lambda I$ introduced at the beginning of \Cref{sec:estimation-error-apdx}.

In order to prove \Cref{estimation-error-main-lemma}, we first state and prove the following decomposition lemma. 

\begin{lem}
    We have 
    \begin{equation} \widehat u_\lambda-u_\lambda = \widehat \Sigma_\lambda^{-1} \left[ (\widehat{\xi} - \xi) + (\Sigma-\widehat{\Sigma})\Sigma_\lambda^{-1} \xi \right].\end{equation}
\end{lem}
\begin{proof}
    We can rewrite the expressions \eqref{expression-of-u-lambdahat:eq} and \eqref{expression-of-u-lambda} as $\uhl = \widehat \Sigma_\lambda^{-1} \widehat \xi$ and $\ul = \Sigma_\lambda^{-1} \xi$. We thus have

\begin{align*}
        \widehat u_\lambda-u_\lambda &= \widehat \Sigma_\lambda^{-1}\widehat{\xi} - \Sigma_\lambda^{-1} \xi\\
        &=\widehat \Sigma_\lambda^{-1}\widehat{\xi} - \widehat \Sigma_\lambda^{-1} \xi+ \widehat \Sigma_\lambda^{-1} \xi - \Sigma_\lambda^{-1} \xi\\
        &=\widehat \Sigma_\lambda^{-1} (\widehat{\xi} - \xi) + (\widehat \Sigma_\lambda^{-1}-\Sigma_\lambda^{-1})\xi\\
        &=\widehat \Sigma_\lambda^{-1} (\widehat{\xi} - \xi)  + \widehat \Sigma_\lambda^{-1}(\Sigma_\lambda-\widehat{\Sigma}_\lambda)\Sigma_\lambda^{-1} \xi\\
        &=\widehat \Sigma_\lambda^{-1} (\widehat{\xi} - \xi)  + \widehat \Sigma_\lambda^{-1}(\Sigma-\widehat{\Sigma})\Sigma_\lambda^{-1} \xi\\
        &=\widehat \Sigma_\lambda^{-1} \left[ (\widehat{\xi} - \xi) + (\Sigma-\widehat{\Sigma})\Sigma_\lambda^{-1} \xi \right]. \qedhere
\end{align*}
\end{proof}

\begin{lem}\label{estimation-error-decomposition-lemma}
We have
    \begin{equation}\label{estimation-error-decomposition-src-eq} \|A_\rho (\uhl - \ul) \|_\rho \ \leq  \  \| \Sigma_\lambda^{\frac{1}{2}} \widehat \Sigma_\lambda^{-\frac{1}{2}} \|_{\op} \left( \frac{1}{\sqrt{\lambda}} \| \widehat{\xi} - \xi \|_\cH + \frac{\|\Sigma\|_{\op}^{r - \tilde r}}{\lambda^{1-\tilde{r}}} \|\Sigma-\widehat{\Sigma} \|_{\op} \ \|(\widetilde f , \widetilde g)\|_\rho \right). \end{equation}
\end{lem}

\begin{proof}
    First observe that 
    \[\|A_\rho (\uhl - \ul) \|_\rho = \|\Sigma^{\frac{1}{2}} (\widehat u_\lambda-u_\lambda)  \|_\cH .\]
Now, observe that 
\[ \xi = A_\rho^*(h, q) = A_\rho^* (h^\Pi, q^\Pi) = A_\rho^* L^r (\widetilde f, \widetilde g) = \Sigma^r A_\rho^* (\widetilde f, \widetilde g).\]
Thus,
\begin{align*}
    \Sigma^{\frac{1}{2}} (\widehat u_\lambda-u_\lambda) & =  \Sigma^{\frac{1}{2}} \widehat \Sigma_\lambda^{-1} \left[ (\widehat{\xi} - \xi) + (\Sigma-\widehat{\Sigma})\Sigma_\lambda^{-1} \xi \right]\\
    & = \Sigma^{\frac{1}{2}} \widehat \Sigma_\lambda^{-\frac{1}{2}} \widehat \Sigma_\lambda^{-\frac{1}{2}}  \left[ (\widehat{\xi} - \xi) + (\Sigma-\widehat{\Sigma})\Sigma_\lambda^{-1} \Sigma^r A_\rho^* (\widetilde f, \widetilde g) \right],
\end{align*}
thus
\begin{align*}
\| \Sigma^{\frac{1}{2}} (\widehat u_\lambda-u_\lambda) \|_\cH \ & \leq \ \| \Sigma^{\frac{1}{2}} \widehat \Sigma_\lambda^{-\frac{1}{2}} \|_{\op} \| \widehat \Sigma_\lambda^{-\frac{1}{2}}  \|_{\op} \\
& \hspace{7em} \cdot \left[ \| \widehat{\xi} - \xi \|_\cH + \|\Sigma-\widehat{\Sigma} \|_{\op} \|\Sigma_\lambda^{-\frac{1}{2}+\tilde{r}} \|_{\op} \| \Sigma_\lambda^{-\frac{1}{2}-\tilde{r}}  \Sigma^{\tilde{r}} A_\rho^* \|_{\rho \to \cH} \|(\widetilde f, \widetilde g) \|_\rho \right] \\
& \leq \ \frac{1}{\sqrt{\lambda}} \| \Sigma^{\frac{1}{2}} \widehat \Sigma_\lambda^{-\frac{1}{2}} \|_{\op}\left[ \| \widehat{\xi} - \xi \|_\cH + \frac{\|\Sigma\|_{\op}^{r - \tilde{r}}}{\lambda^{\frac{1}{2}-\tilde{r}}} \|\Sigma-\widehat{\Sigma} \|_{\op} \ \| \Sigma_\lambda^{-\frac{1}{2}-\tilde{r}}  \Sigma^{\tilde{r}} A_\rho^* \|_{\rho \to \cH} \|(\widetilde f, \widetilde g) \|_\rho \right],
\end{align*}
where $\| \cdot \|_{\op}$ denotes the standard operator norm on $(\cH, \|\cdot \|_\cH)$, and $\| \cdot\|_{\rho \to \cH}$ denotes the operator norm from $(L^2(\rho_X) \times L^2(\rho_Z), \| \cdot \|_\rho)$ to $(\cH, \| \cdot \|_\cH)$. We used the fact that $-\frac 12 + \tilde{r} \leq 0$ to bound $\|\Sigma_\lambda^{-\frac{1}{2}+\tilde{r}}\|_{\op}\leq \lambda^{-\frac{1}{2}+\tilde{r}}$.

In order to conclude, we need to bound $\| \Sigma^{\frac{1}{2}} \widehat \Sigma_\lambda^{-\frac{1}{2}} \|_{\op}$ and $ \| \Sigma_\lambda^{-\frac{1}{2}-\tilde{r}}  \Sigma^{\tilde{r}} A_\rho^* \|_{\rho \to \cH}$. First observe that 
\begin{align*} \| \Sigma_\lambda^{-\frac{1}{2}-r}  \Sigma^{\tilde{r}} A_\rho^* \|_{\rho \to \cH} & = \sqrt{ \left\|( \Sigma_\lambda^{-\frac{1}{2}-\tilde{r}}  \Sigma^{\tilde{r}} A_\rho^*) ( \Sigma_\lambda^{-\frac{1}{2}-\tilde{r}}  \Sigma^{\tilde{r}} A_\rho^*)^*\right\|_{\op}}\\
    & = \sqrt{ \left\|\Sigma_\lambda^{-\frac{1}{2}-\tilde{r}} \Sigma^{1 + 2\tilde{r}} \Sigma_\lambda^{-\frac{1}{2}-\tilde{r}}  \right\|_{\op}}.
\end{align*}
The eigenvalues of the operator $\Sigma_\lambda^{-\frac{1}{2}-\tilde{r}} \Sigma^{1 + 2\tilde{r}} \Sigma_\lambda^{-\frac{1}{2}-\tilde{r}}   = \Sigma_\lambda^{-1-2\tilde{r}}  \Sigma^{1 + 2\tilde{r}} $ are $\frac{\mu_i^{1 + 2\tilde{r}}}{(\mu_i + \lambda)^{1 + 2\tilde{r}}} \leq 1$, where $\mu_i$ are the same eigenvalues from \eqref{decomposition:integral:op:ONS} (we use that $\Sigma = A_\rho^* A_\rho$ and $L = A_\rho A_\rho^*$ have the same non-zero eigenvalues.)). As a consequence, its operator norm is bounded by $1$, and thus we have
\[  \| \Sigma_\lambda^{-\frac{1}{2}-\tilde{r}}  \Sigma^{\tilde{r}} A_\rho^* \|_{\rho \to \cH} \leq 1 .\]
Let us now bound $\| \Sigma^{\frac{1}{2}} \widehat \Sigma_\lambda^{-\frac{1}{2}} \|_{\op}$. We have
\begin{align*} \| \Sigma^{\frac{1}{2}} \widehat \Sigma_\lambda^{-\frac{1}{2}} \|_{\op} & \leq \| \Sigma^{\frac{1}{2}} \Sigma_\lambda^{-\frac{1}{2}} \|_{\op} \| \Sigma_\lambda^{\frac{1}{2}} \widehat \Sigma_\lambda^{-\frac{1}{2}} \|_{\op} \\
& \leq \| \Sigma_\lambda^{\frac{1}{2}} \widehat \Sigma_\lambda^{-\frac{1}{2}} \|_{\op} ,
\end{align*}
where we used $\| \Sigma^{\frac{1}{2}} \Sigma_\lambda^{-\frac{1}{2}} \|_{\op} \leq 1$ since the eigenvalues of $ \Sigma^{\frac{1}{2}} \Sigma_\lambda^{-\frac{1}{2}}$ are $\frac{\sqrt{\mu_i}}{\sqrt{\mu_i + \lambda}}$. 
\end{proof}

The following lemma allows us to express $\widehat \xi$ as an empirical mean and $\xi$ as an expectation.

\begin{lem}\label{expression-of-xi-and-xihat-apdx}
We have 
\begin{equation}\label{expression-of-xihat-apdx}\widehat \xi = \frac{1}{n} \sum_{i=1}^n y_i  K_{x_i}  + \frac{1}{m} \sum_{j=1}^m w_j  K^\cD_{z_j},\end{equation}
    and
\begin{equation}\label{expression-of-xi-apdx}\xi = \EE_{x \sim \rho_X} [h(x) K_x] + \EE_{z \sim \rho_Z}[q(z)  K^\cD_z].
\end{equation}
Furthermore, we have the bound
\begin{equation}\label{bound-on-xi-apdx} \| \xi \|_\cH \leq  \| h \|_{L^\infty(\rho_X)}\kappa + \| q\|_{L^\infty(\rho_Z)} \kappa_\cD.  \end{equation}
\end{lem}

\begin{proof}
    Let $u \in \cH$. We have
    \begin{equation*}
        \begin{aligned}
            \langle \widehat \xi, u \rangle_\cH & = \langle \Ah^* Y, u \rangle_\cH \\
            & = \langle Y, \Ah u \rangle_{n,m} \\
            & = \frac{1}{n} \sum_{i=1}^n y_i u(x_i) + \frac{1}{m} \sum_{j=1}^m w_j \cD u(z_j) \\
            & = \frac{1}{n} \sum_{i=1}^n y_i \langle K_{x_i}, u \rangle_\cH + \frac{1}{m} \sum_{j=1}^m w_j \langle K^\cD_{z_j} , u \rangle_\cH \\
            & = \left\langle \frac{1}{n} \sum_{i=1}^n y_i  K_{x_i}  + \frac{1}{m} \sum_{j=1}^m w_j  K^\cD_{z_j} \ , \ u \right\rangle_\cH.
        \end{aligned}
    \end{equation*}
    Since this is true for any $u \in \cH$, it proves \eqref{expression-of-xihat-apdx}.

We prove the second formula the same way. Let $u \in \cH$, we have
    \begin{equation*}
        \begin{aligned}
            \langle \xi, u \rangle_\cH & = \langle A_\rho^* (h, q), u \rangle_\cH \\
            & = \langle (h, q),  A_\rho u \rangle_\rho \\
            &  = \EE_{x \sim \rho_X} [h(x) u(x) ] + \EE_{z \sim \rho_Z} [ q(z) \cD u(z) ] \\
            &  = \EE_{x \sim \rho_X} [h(x) \langle K_x , u \rangle_\cH ] + \EE_{z \sim \rho_Z} [ q(z) \langle K^\cD_z , u \rangle_\cH ]  \\
            & = \left\langle \EE_{x \sim \rho_X} [h(x) K_x] + \EE_{z \sim \rho_Z}[q(z)  K^\cD_z] \ , \ u \right\rangle_\cH.
        \end{aligned}
    \end{equation*}
To prove \eqref{expression-of-xi-apdx}, observe first that $h(x)$ is $\rho_X$-almost surely bounded (as assumed at the beginning of \Cref{sec:theoretical-analysis-apdx}), and by \Cref{ass:bounded-features}, for all $x \in \Input$ we have $\|K_x\|_\cH \leq \kappa$, so $\| h(x)K_x\|_\cH$ is bounded $\rho_X$-almost surely. As a consequence, the expected value $\EE_{x \sim \rho_X} [h(x) K_x]$ is well defined as a Bochner integral with values in $\cH$. Analogously, $\|q(z)  K^\cD_z\|_\cH$ is bounded $\rho_Z$-almost surely which allows us to define the expected value $\EE_{z \sim \rho_Z}[q(z)  K^\cD_z]$. Finally, the Bochner integral commutes with bounded operators, so we can exchange the expectations and the scalar product $\langle \cdot , u \rangle_\cH$. Since this equality is true for any $u \in \cH$, it proves \eqref{expression-of-xi-apdx}. 

Let us now prove the bound. We have 
\[ \| \, \EE_{x \sim \rho_X} [h(x) K_x] \, \|_\cH \ \leq \ \EE_{x \sim \rho_X} [\|h(x) K_x\|_\cH] \ \leq \ \|h\|_{L^\infty(\rho_X)} \kappa. \]
Analogously, we have
\[ \| \, \EE_{z \sim \rho_Z}[q(z)  K^\cD_z] \, \|_\cH \ \leq \ \EE_{z \sim \rho_Z}[\| q(z)  K^\cD_z\|_\cH] \ \leq \ \|q\|_{L^\infty(\rho_Z)} \kappa_\cD. \]
Using the triangle inequality on \eqref{expression-of-xi-apdx} and the previous two inequalities yields \eqref{bound-on-xi-apdx}.
\end{proof}

\begin{proof}[Proof of Proposition~\ref{estimation-error-main-lemma}]
We know from \Cref{estimation-error-decomposition-lemma} that 
\begin{equation}\label{estimation-error-decomposition-src-eq-apdx}\|A_\rho (\uhl - \ul) \|_\rho \ \leq  \  \| \Sigma_\lambda^{\frac{1}{2}} \widehat \Sigma_\lambda^{-\frac{1}{2}} \|_{\op} \left( \frac{1}{\sqrt{\lambda}} \| \widehat{\xi} - \xi \|_\cH + \frac{\|\Sigma\|_{\op}^{r - \tilde r}}{\lambda^{1-\tilde{r}}} \|\Sigma-\widehat{\Sigma} \|_{\op} \ \|(\widetilde f , \widetilde g)\|_\rho \right).
\end{equation}
Let us prove \emph{(i)} first by decomposing both $\| \widehat \xi - \xi \|_\cH$ and $\| \Sigma - \widehat \Sigma \|_{\op}$. 

We begin with $\| \widehat \xi - \xi \|_\cH$. Using \Cref{expression-of-xi-and-xihat-apdx}, we have 
\[\widehat \xi =\frac{1}{n} \sum_{i=1}^n y_i  K_{x_i}  + \frac{1}{m} \sum_{j=1}^m w_j  K^\cD_{z_j},\]
and
\[\xi = \EE_{x \sim \rho_X} [h(x) K_x] + \EE_{z \sim \rho_Z}[q(z)  K^\cD_z].\]
Let us define 
\[U_n = \left\| \frac{1}{n} \sum_{i=1}^n y_i  K_{x_i} - \EE_{x \sim \rho_X} [h(x) K_x] \right\|_\cH\]
and
\[W_m =  \left\|\frac{1}{m} \sum_{j=1}^m w_j  K^\cD_{z_j} -  \EE_{z \sim \rho_Z}[q(z)  K^\cD_z] \right\|_\cH.\]
By the triangle inequality, we have
\begin{equation}\label{decomposition:hhat-h}
    \| \widehat \xi - \xi \|_\cH \ \leq \ U_n + W_m.
\end{equation}

Let us now proceed with $\| \Sigma - \widehat \Sigma \|_{\op} $.
Recall from \Cref{covariance-operators-decomposition-apdx} that we have
\[\widehat \Sigma = \frac{1}{n} \sum_{i=1}^n K_{x_i} \otimes K_{x_i} + \frac{1}{m} \sum_{j=1}^m K^\cD_{z_j} \otimes K^\cD_{z_j} = \widehat \Sigma_X + \widehat \Sigma_Z,\]
and
\[\Sigma = \EE_{x \sim \rho_X}[K_x \otimes K_x] + \EE_{z \sim \rho_Z}[K^\cD_z \otimes K^\cD_z ] = \Sigma_X + \Sigma_Z.\]
We define
\[V_n = \|(\widetilde f, \widetilde g)\|_\rho \left\| \widehat \Sigma_X -  \Sigma_X \right\|_{\op} ,\]
and 
\[Z_m = \|(\widetilde f, \widetilde g)\|_\rho \left\| \widehat \Sigma_Z - \Sigma_Z \right\|_{\op} .\]
By the triangle inequality, we have
\begin{equation}\label{decomposition:Sigmahat-Sigma}
  \|(\widetilde f, \widetilde g)\|_\rho  \left\|\Sigmah - \Sigma \right\|_{\op}  \ \leq \ V_n + Z_m.
\end{equation}
The inequalities \eqref{decomposition:hhat-h} and \eqref{decomposition:Sigmahat-Sigma} alongside with the decomposition \eqref{estimation-error-decomposition-src-eq-apdx} prove \emph{(i)}.

The four remaining points are all proved the same way, by using Hoeffding's inequality in Hilbert spaces. Let us prove \emph{(ii)} and \emph{(iv)} jointly, i.e. let us bound $U_n$ and $W_m$. We define the i.i.d. random variables
\[ e_i = y_i K_{x_i},\]
and the i.i.d. random variables
\[t_j = w_j K^\cD_{z_j}.\]
For $i \in \lb 1 , n \rb$, we have
\[\EE[e_i] = \EE\left[\EE[y_i K_{x_i}|x_i]\right] = \EE\left[ h(x_i) K_{x_i} \right],\]
and for $j \in \lb 1 , m \rb$, we have
\[\EE[t_j] = \EE\left[\EE[ w_jK^\cD_{z_j}| z_j]\right] = \EE\left[q(z_j)K^\cD_{z_j}\right],\]
so we can rewrite
\[U_n = \left\| \frac{1}{n} \left( \sum_{i=1}^n e_i   - \EE[e_i] \right) \right\|_\cH\]
and
\[ W_m =  \left\| \frac{1}{m} \left( \sum_{j=1}^m t_j - \EE [ t_j] \right) \right\|_\cH.\]

We want to apply Hoeffding's inequality in separable Hilbert spaces to bound $U_n$ and $W_m$. For completeness, we reproduced the inequality in this appendix as \Cref{Hoeffding-ineq-sep-hilbert-spaces}. For that, we need first to establish that the variables $e_i$ and $t_j$ are bounded. It is indeed the case since we have $\| e_i \|_\cH \leq |y_i| \| K_{x_i} \|_\cH \leq M_y \kappa$, where both $M_y$ and $\kappa$ are positive constants defined in \Cref{sec:theory-assumptions-apdx}, and where the second inequality holds almost surely. As a consequence, the zero-mean variable $e_i - \EE[e_i]$ satisfies almost surely 
\[\|e_i - \EE[e_i]\| \leq 2M_y \kappa.\]
Similarly, we have $\| t_j \|_\cH \leq |w_j| \| K^\cD_{z_j} \|_\cH \leq M_w \kappa_\cD$, where again $M_w$ and $\kappa_\cD$ are defined in \Cref{sec:theory-assumptions-apdx} and the second inequality holds almost surely. As a consequence, the zero-mean variable $t_j - \EE[t_j]$ satisfies almost surely 
\[\|t_j - \EE[t_j]\|_\cH \leq 2 M_w \kappa_\cD.\]
We can thus apply the Hoeffding inequality in separable Hilbert spaces (see \Cref{Hoeffding-ineq-sep-hilbert-spaces}) and get that, with probability at least $1 - 2e^{-\tau}$, we have
\begin{equation}\label{hoeffding-w}U_n \leq  4 M_y \kappa  \sqrt{\frac{\tau}{n}},\end{equation}
which proves \emph{(ii)}.
Similarly, we find that with probability at least $1 - 2e^{-\tau}$, we have
\begin{equation}\label{hoeffding-t} W_m \leq  4 M_w \kappa_\cD \sqrt{\frac{\tau}{m}},\end{equation}
which proves \emph{(iv)}.

Let us now prove \emph{(iii)} and \emph{(v)}. Bounding the operator norm by the Hilbert-Schmidt norm, we have
\begin{equation}\label{op:norm:bounded:by:HS:norm:V_n}V_n \leq \|(\widetilde f, \widetilde g)\|_\rho  \bigg\| \frac{1}{n} \sum_{i=1}^n K_{x_i} \otimes K_{x_i} -  \EE_x[K_x \otimes K_x] \bigg\|_{HS}  \end{equation}
and
\begin{equation}\label{op:norm:bounded:by:HS:norm:Z_m}Z_m \leq  \|(\widetilde f, \widetilde g)\|_\rho  \bigg\| \frac{1}{m} \sum_{j=1}^m K^\cD_{z_j} \otimes K^\cD_{z_j}-  \EE_z[K^\cD_z \otimes K^\cD_z ] \bigg\|_{HS}.\end{equation}
To apply again Hoeffding inequality, we need the variables $ K_{x_i} \otimes K_{x_i}$ and $ K^\cD_{z_j} \otimes K^\cD_{z_j}$ to be bounded, which is the case as proved in \Cref{bound-hilbert-schmidt-norm-apdx}. We can thus apply Hoeffding inequality in the separable Hilbert space of Hilbert-Schmidt operators (see \Cref{Hoeffding-ineq-sep-hilbert-spaces}) to find that with probability at least $1 - 2e^{-\tau}$, we have
\begin{equation}\label{hoeffding-kxk} \left\| \frac{1}{n} \sum_{i=1}^n K_{x_i} \otimes K_{x_i} -  \EE_x[K_x \otimes K_x] \right\|_{HS} \leq  4 \kappa^2 \sqrt{\frac{\tau}{n}}.
\end{equation}
Together, \eqref{op:norm:bounded:by:HS:norm:V_n}, \eqref{hoeffding-kxk} and \eqref{bound-on-xi-apdx} prove \emph{(iii)}.
Similarly,  with probability at least $1 - 2e^{-\tau}$, we have
\begin{equation}\label{hoeffding-jxj} \left\| \frac{1}{m} \sum_{j=1}^m K^\cD_{z_j} \otimes K^\cD_{z_j}-  \EE_z[K^\cD_z \otimes K^\cD_z ] \right\|_{HS} \leq 4 \kappa_\cD^2  \sqrt{\frac{\tau}{m}}.
\end{equation}
Together, \eqref{op:norm:bounded:by:HS:norm:Z_m}, \eqref{hoeffding-jxj} and \eqref{bound-on-xi-apdx} prove \emph{(v)}.
\end{proof}

\Cref{estimation-error-main-lemma} allows us to prove the following two results.

\begin{coro}\label{coro:estimation-error-apdx}
  Let $\delta \in (0,\frac12)$. There exists a constant $C > 0$ such that if $\lambda \geq C \max \left( \frac{1}{n}\log \frac{n}{\delta}, \frac{1}{m}\log \frac{m}{\delta} \right) $, and if $\lambda < \lambda_0$ for an arbitrary upper bound $\lambda_0 > 0$, then with probability at least $1 - \delta$, we have
\begin{equation}\label{eq:estimation-error-apdx}
    \| A_\rho \uhl - A_\rho\ul \|_\rho \quad \lesssim \quad \frac{\sqrt{\ln(1/\delta)}}{\lambda^{1-\tilde{r}} \sqrt{n}} + \frac{\sqrt{\ln(1/\delta)}}{\lambda^{1-\tilde{r}} \sqrt{m}},
\end{equation}
where the hidden constants in \eqref{eq:estimation-error-apdx} depend on $M_y, M_w, \kappa, \kappa_\cD$ and $\lambda_0$.
\end{coro}

\begin{proof}
We know from \Cref{estimation-error-main-lemma} that we can write 
    \[\| A_\rho\uhl - A_\rho\ul \|_\rho \quad \leq \quad  \| \Sigma_\lambda^{\frac{1}{2}} \widehat \Sigma_\lambda^{-\frac{1}{2}} \|_{\op} \left( \frac{1}{\lambda^{\frac{1}{2}}}U_n + \frac{1}{\lambda^{1-\tilde{r}}}V_n + \frac{1}{\lambda^{\frac{1}{2}}}W_m + \frac{1}{\lambda^{1-\tilde{r}}}Z_m \right).\]

    Let us first bound $\| \Sigma_\lambda^{\frac{1}{2}} \widehat \Sigma_\lambda^{-\frac{1}{2}} \|_{\op}$. We know that if we have
    \[\lambda \geq \frac{112 }{3} \max \left( \frac{\kappa^2}{n}\log \frac{n+4}{\widetilde \delta}, \frac{\kappa_\cD^2}{m}\log \frac{m+4}{\widetilde \delta} \right)\]
    then with probability $1-2\widetilde \delta$, we have simultaneously 
    \[\|B_X\|_{\op} \leq \frac{1}{4}, \qquad \|B_Z\|_{\op} \leq \frac{1}{4},\]
    so \Cref{lemma:population-cov-controls-emp-cov} implies
    \[ \| \Sigma_\lambda^{\frac{1}{2}} \widehat \Sigma_\lambda^{-\frac{1}{2}} \|_{\op} \le \sqrt{2} \,.\]
    Now using a union bound on this probability bound alongside with the bounds from \Cref{estimation-error-main-lemma}, where we pick $\tau = \ln(1/\widetilde \delta)$ we get that if $\lambda \geq \frac{112 }{3} \max \left( \frac{\kappa^2}{n}\log \frac{n+4}{\widetilde \delta}, \frac{\kappa_\cD^2}{m}\log \frac{m+4}{\widetilde \delta} \right)$, then with probability at least $1 - 10\widetilde \delta$, we have
     \begin{align*} \| A_\rho\uhl - A_\rho\ul \|_\rho \quad & \leq \quad \sqrt{2} \left( C_U\frac{\sqrt{\ln(1/\widetilde \delta)}}{\lambda^{\frac{1}{2}} \sqrt{n}} + C_V\frac{\sqrt{\ln(1/\widetilde \delta)}}{\lambda^{1-\tilde{r}} \sqrt{n}} + C_W \frac{\sqrt{\ln(1/\widetilde \delta)}}{\lambda^{\frac{1}{2}} \sqrt{m}} + C_Z \frac{\sqrt{\ln(1/\widetilde \delta)}}{\lambda^{1-\tilde{r}} \sqrt{m}} \right)\\
        & \leq \quad C' \left(\frac{\sqrt{\ln(1/\delta)}}{\lambda^{1-\tilde{r}} \sqrt{n}}  +  \frac{\sqrt{\ln(1/\delta)}}{\lambda^{1-\tilde{r}} \sqrt{m}} \right),
     \end{align*}
where we defined $ \delta := 10 \widetilde \delta$, and the second inequality is true provided $\lambda < \lambda_0$ for an arbitrary choice of $\lambda_0 > 0$, and where the constant $C'$ depends on $M_y, M_w, \kappa, \kappa_\cD, \lambda_0$ and $\|(\widetilde f , \widetilde g) \|_\rho$.
Using $\delta$, the condition on $\lambda$ can be written as
\begin{equation}
    \lambda \geq \frac{112 }{3} \max \left( \frac{\kappa^2}{n}\log \frac{10n+40}{\delta}, \frac{\kappa_\cD^2}{m}\log \frac{10m+40}{\delta} \right),
\end{equation}
which we can simplify (using that $\delta < 1/2$) as
\begin{equation*}
    \lambda \geq C \max \left( \frac{1}{n}\log \frac{n}{\delta}, \frac{1}{m}\log \frac{m}{\delta} \right).
\end{equation*}
\end{proof}

Recall that in the current section, we work under \Cref{ass:diff-eval,ass:bounded-features,ass:bounded-data}. Consider furthermore the following assumption.
\begin{assumption}\label{ass:data-in-the-closure}
    The pair $(h, q)$ is in the closure of the range of $A_\rho$ taken in $L^2(\rho_X) \times L^2(\rho_Z)$:
    \[(h, q) \in \overline{\ran A_\rho}.\]
\end{assumption}
\begin{rem}\label{rem:equivalence-assumptions-source}
    Since for any $r>0$, $\overline{\ran L^r} = \overline{\ran A \rho}$, we see that \Cref{ass:src} is equivalent to \Cref{ass:src-var-apdx,ass:data-in-the-closure} together. 
\end{rem}

\begin{coro}\label{coro:rates-apdx}
  Let $\delta \in (0,\frac12)$. Under \Cref{ass:src-var-apdx}, if $\lambda \geq C \max \left( \frac{1}{n}\log \frac{n}{\delta}, \frac{1}{m}\log \frac{m}{\delta} \right) $, and if $\lambda < \lambda_0$ for an arbitrary upper bound $\lambda_0 > 0$, then with probability at least $1 - \delta$, we have
\begin{equation}\label{eq:error-bound-projected-data-apdx}
    \| A_\rho \uhl - (h^\Pi, q^\Pi) \|_\rho \quad \lesssim \quad \frac{\sqrt{\ln(1/\delta)}}{\lambda^{1-\tilde{r}} \sqrt{n}} + \frac{\sqrt{\ln(1/\delta)}}{\lambda^{1-\tilde{r}} \sqrt{m}} + \lambda^r,
\end{equation}
where the hidden constants in \eqref{eq:error-bound-projected-data-apdx} depend on $M_y, M_w, \kappa, \kappa_\cD$ and $\lambda_0$.
Furthermore, under \Cref{ass:data-in-the-closure}, this simplifies to
\begin{equation}\label{eq:finite-sample-rates-data-closure}
    \| A_\rho \uhl - (h, q) \|_\rho \quad \lesssim \quad \frac{\sqrt{\ln(1/\delta)}}{\lambda^{1-\tilde{r}} \sqrt{n}} + \frac{\sqrt{\ln(1/\delta)}}{\lambda^{1-\tilde{r}} \sqrt{m}} + \lambda^r.
\end{equation}
In particular, under \Cref{ass:src}, both \Cref{ass:src-var-apdx,ass:data-in-the-closure} hold, so \eqref{eq:finite-sample-rates-data-closure} holds.
\end{coro}
\begin{proof}
    Using together the approximation error from \Cref{approx-error-source-condition} and \Cref{coro:estimation-error-apdx} yields the first result. Then, \Cref{ass:data-in-the-closure} is equivalent to $(h^\Pi, q^\Pi) = (h, q)$, which gives the second result. The final remark comes from the fact that \Cref{ass:src} is equivalent to \Cref{ass:src-var-apdx,ass:data-in-the-closure} together, as mentioned in \Cref{rem:equivalence-assumptions-source}.
\end{proof}

\begin{coro}\label{coro:simplified-rates-apdx}
    Consider \Cref{ass:src}, and let us pick $\lambda = N^{-1/2}$ if $r \leq 1/2$, and $\lambda = N^{-\frac{1}{2r+1}}$ if $r > 1/2$, where $N := \min(n,m)$. If $\delta \in (0,\frac12)$ is such that $\lambda \geq C \max \left( \frac{1}{n}\log \frac{n}{\delta}, \frac{1}{m}\log \frac{m}{\delta} \right) $, and if $\lambda < \lambda_0$ for an arbitrary upper bound $\lambda_0 > 0$, then with probability at least $1 - \delta$, we have
    \begin{align}
        \| A_\rho \uhl - (h, q) \|_\rho \quad & \lesssim \quad \sqrt{\ln(1/\delta)} N^{-r/2} & \text{if } \ r \leq 1/2  \label{eq:simplified-rates-r<1/2-apdx} \\
        \| A_\rho \uhl - (h, q) \|_\rho \quad & \lesssim \quad \sqrt{\ln(1/\delta)} N^{-\frac{r}{2r + 1}} & \text{if } \ r > 1/2,   \label{eq:simplified-rates-r>1/2-apdx}
    \end{align}
    where the hidden constants in \eqref{eq:simplified-rates-r<1/2-apdx} and \eqref{eq:simplified-rates-r>1/2-apdx} depend on $M_y, M_w, \kappa, \kappa_\cD$ and $\lambda_0$.
\end{coro}
\begin{rem*}
    Since $\lambda \geq N^{-1/2}$ and the condition on $\lambda$ is of the form $\lambda \gtrsim \frac{\log(N/\delta)}{N}$, we see that for any $\delta \in (0,1)$, the condition on $\lambda$ is always satisfied for $N$ large enough.
\end{rem*}
\begin{proof}
    As mentioned in \Cref{rem:equivalence-assumptions-source}, \Cref{ass:src} is equivalent to \Cref{ass:src-var-apdx,ass:data-in-the-closure} together. We can thus use the rates in \eqref{eq:finite-sample-rates-data-closure}. Suppose first that $r \leq 1/2$, i.e. $\tilde{r} = r$, and let us set $\lambda = N^{-1/2}$. With probability $1 - \delta$,
    \begin{align*}
    \| A_\rho \uhl - (h, q) \|_\rho \quad & \lesssim \quad \frac{\sqrt{\ln(1/\delta)}}{\lambda^{1-r} \sqrt{n}} + \frac{\sqrt{\ln(1/\delta)}}{\lambda^{1-r} \sqrt{m}} + \lambda^r\\
    & \lesssim \quad \sqrt{\ln(1/\delta)}N^{\frac{1-r}{2}} n^{-1/2} + \sqrt{\ln(1/\delta)}N^{\frac{1-r}{2}} m^{-1/2}  + N^{-r/2}\\
    & \lesssim \quad \sqrt{\ln(1/\delta)} N^{-r/2}  + N^{-r/2}. 
\end{align*}

Suppose now that $r > 1/2$, i.e. $\tilde{r} = 1/2$, and let us set $\lambda = N^{-\frac{1}{2r+1}}$. 
With probability $1 - \delta$,
    \begin{align*}
    \| A_\rho \uhl - (h, q) \|_\rho \quad & \lesssim \quad \frac{\sqrt{\ln(1/\delta)}}{\sqrt{\lambda} \sqrt{n}} + \frac{\sqrt{\ln(1/\delta)}}{\sqrt{\lambda} \sqrt{m}} + \lambda^r\\
    & \lesssim \quad \sqrt{\ln(1/\delta)}N^{\frac{1}{4r +2}} n^{-1/2} + \sqrt{\ln(1/\delta)}N^{\frac{1}{4r+2}} m^{-1/2}  + N^{-\frac{r}{2r+1}}\\
    & \lesssim \quad \sqrt{\ln(1/\delta)} N^{-\frac{r}{2r+1}}  + N^{-\frac{r}{2r+1}}.
\end{align*}

\end{proof}

\subsection{Asymptotic convergence}

We take again the notation $(h^\Pi, q^\Pi)$ from \Cref{sec:approx-error-apdx} to denote the orthogonal projection (in $L^2(\rho_X) \times L^2(\rho_Z)$) of $(h, q)$ on the closed subspace $\overline{\ran A_\rho}$.

Recall that for any fixed $\lambda > 0$, $\uhl$ already depends on $(n,m)$ (we dropped the index for lighter notation), which means that for a choice $\lambda_{n,m}$ that depends on $(n,m)$, the sequence $\widehat u_{\lambda_{n,m}}$ depends on $(n,m)$ in two different ways.

\begin{prop}\label{asymptotic:convergence}
Let $N=\min(n,m)$. Let $(\lambda_{n,m})$ be any regularization sequence such that
\[
    \lambda_{n,m}\to 0,
    \qquad 
    \frac{\log N}{\lambda_{n,m}^3 N}\to 0
    \qquad \text{as } n,m\to\infty .
\]
Then, almost surely,
\[
    \|A_\rho \widehat u_{\lambda_{n,m}} - (h^\Pi, q^\Pi)\|_\rho
    \quad \underset{n,m \rightarrow + \infty}{\longrightarrow} \quad 0.
\]
Equivalently, almost surely, for every $\epsilon>0$, there exists $N_\epsilon\in\NN$ such that
\[
    \forall (n,m)\in\NN^2, \qquad 
    \left[n,m\geq N_\epsilon
    \Longrightarrow
    \|A_\rho \widehat u_{\lambda_{n,m}} - (h^\Pi,q^\Pi)\|_\rho<\epsilon
    \right].
\]
\end{prop}

\begin{proof}
Using the triangle inequality and the decomposition of
\Cref{estimation-error-main-lemma} with $\lambda=\lambda_{n,m}$, we have
\begin{align}
    \| A_\rho \widehat u_{\lambda_{n,m}} - (h^\Pi, q^\Pi) \|_\rho
    &\leq
    \| A_\rho \widehat u_{\lambda_{n,m}} - A_\rho u_{\lambda_{n,m}} \|_\rho
    +
    \| A_\rho u_{\lambda_{n,m}} - (h^\Pi, q^\Pi) \|_\rho .
    \nonumber
\end{align}
\Cref{approx:error:apdx} shows that since
$\lambda_{n,m}\to 0$, the approximation error goes to $0$:
\[
    \| A_\rho u_{\lambda_{n,m}} - (h^\Pi, q^\Pi) \|_\rho
    \underset{n,m \rightarrow + \infty}{\longrightarrow} 0.
\]
We thus only need to prove that, almost surely,
\[
    \| A_\rho \widehat u_{\lambda_{n,m}} - A_\rho u_{\lambda_{n,m}} \|_\rho
    \underset{n,m \rightarrow + \infty}{\longrightarrow} 0 .
\]

Since we do not assume any source condition here, we are in the case
$\tilde r=0$ of \Cref{estimation-error-main-lemma}, which gives, for any
$\lambda>0$,
\begin{align*}
\| A_\rho \widehat u_\lambda - A_\rho u_\lambda \|_\rho
&\leq
\| \Sigma_\lambda^{1/2} \widehat \Sigma_\lambda^{-1/2} \|_{\op}
\left(
    \frac{1}{\lambda^{1/2}}U_n
    + \frac{1}{\lambda}V_n
    + \frac{1}{\lambda^{1/2}}W_m
    + \frac{1}{\lambda}Z_m
\right) \\
&\leq
\| \Sigma_\lambda^{1/2} \|_{\op}
\left(
    \frac{1}{\lambda}U_n
    + \frac{1}{\lambda^{3/2}}V_n
    + \frac{1}{\lambda}W_m
    + \frac{1}{\lambda^{3/2}}Z_m
\right),
\end{align*}
where we used
\[
    \| \Sigma_\lambda^{1/2} \widehat \Sigma_\lambda^{-1/2} \|_{\op}
    \leq
    \frac{\| \Sigma_\lambda^{1/2} \|_{\op}}{\sqrt{\lambda}} .
\]

Let $\tau_k=2\log k$. By point \emph{(ii)} of
\Cref{estimation-error-main-lemma}, with probability at least
$1-2e^{-\tau_k}$,
\[
    U_k \leq C_U \frac{\sqrt{\tau_k}}{\sqrt{k}} .
\]
Since $\sum_{k\geq 2} e^{-\tau_k}<\infty$, the Borel--Cantelli lemma gives
\[
    U_k = O\left(\sqrt{\frac{\log k}{k}}\right)
    \qquad \text{almost surely.}
\]
The same argument applies to $V_k,W_k,Z_k$. Hence, almost surely, for all
large enough $n,m$,
\[
    U_n,V_n = O\left(\sqrt{\frac{\log n}{n}}\right),
    \qquad
    W_m,Z_m = O\left(\sqrt{\frac{\log m}{m}}\right).
\]
Let $N=\min(n,m)$. Since $k\mapsto \log k/k$ is decreasing for large $k$,
we obtain, almost surely,
\[
    \frac{U_n}{\lambda_{n,m}}
    +
    \frac{W_m}{\lambda_{n,m}}
    =
    O\left(
        \sqrt{\frac{\log N}{\lambda_{n,m}^2 N}}
    \right)
    \to 0,
\]
because eventually $\lambda_{n,m}\leq 1$ and
\[
    \frac{\log N}{\lambda_{n,m}^2 N}
    \leq
    \frac{\log N}{\lambda_{n,m}^3 N}
    \to 0 .
\]
Similarly,
\[
    \frac{V_n}{\lambda_{n,m}^{3/2}}
    +
    \frac{Z_m}{\lambda_{n,m}^{3/2}}
    =
    O\left(
        \sqrt{\frac{\log N}{\lambda_{n,m}^3 N}}
    \right)
    \to 0 .
\]
Finally, since $\lambda_{n,m}\to 0$, we have
\[
    \| \Sigma_{\lambda_{n,m}}^{1/2} \|_{\op}
    \to
    \| \Sigma^{1/2} \|_{\op},
\]
and therefore, almost surely,
\[
    \| A_\rho \widehat u_{\lambda_{n,m}} - A_\rho u_{\lambda_{n,m}} \|_\rho
    \underset{n,m \rightarrow + \infty}{\longrightarrow} 0.
\]
Both the approximation error and the estimation error tend to $0$, hence
\[
    \|A_\rho \widehat u_{\lambda_{n,m}} - (h^\Pi, q^\Pi)\|_\rho
    \underset{n,m \rightarrow + \infty}{\longrightarrow} 0 .
\]
\end{proof}

\begin{coro}\label{coro:asymptotic-convergence-apdx}
    If \Cref{ass:diff-eval,ass:bounded-data,ass:bounded-features,ass:data-in-the-closure} hold, for any choice of $\lambda_{n,m} > 0$, for all $n,m \in \NN$, such that
    \[
    \lambda_{n,m}\to 0,
    \qquad 
    \frac{\log N}{\lambda_{n,m}^3 N}\to 0
    \qquad \text{as } n,m\to\infty ,
\]
where $N = \min(n,m)$, then, almost surely, the following convergence holds:
\[ \|A_\rho \widehat u_{\lambda_{n,m}} - (h, q)\|_\rho  \quad \underset{n,m \rightarrow + \infty}{\longrightarrow} \quad 0.\]
That is, almost surely, for any $\epsilon > 0$, there exists $N_\epsilon \in \NN$ such that
\[\forall (n,m) \in \NN^2 , \qquad \left[ \ n,m\geq N_\epsilon \ \Longrightarrow \ \|A_\rho \widehat u_{\lambda_{n,m}} - (h, q)\|_\rho < \epsilon \ \right].\]
\end{coro}
\begin{proof}
    \Cref{ass:diff-eval,ass:bounded-data,ass:bounded-features} are taken in the whole \Cref{sec:theoretical-analysis-apdx} and are usually omitted to ease the reading --- we make them explicit again for the present result. In particular such assumptions are needed in order to apply \Cref{asymptotic:convergence}. Under the additional \Cref{ass:data-in-the-closure} we have $(h, q) \in \overline{\ran A_\rho}$, thus $(h, q)$ is equal to its orthogonal projection $(h^\Pi, q^\Pi)$ on the subspace  $\overline{\ran A_\rho}$. The result then follows from \Cref{asymptotic:convergence}.
\end{proof}

\subsection{Proof of Theorem~\ref{main-theorem}}

Let us consider the setting of \Cref{sec:prob}. Under \Cref{ass:embeddings}, we can define the bounded operator
\begin{equation}
    \function{\cA_\rho}{\cF}{L^2(\rho_X) \times L^2(\rho_Z)}{u}{(u, \cD u).}
\end{equation}
We can then define $(h,q) := \cA_\rho u^*$. Our goal is to show that the hypotheses of \Cref{coro:asymptotic-convergence-apdx} hold.

\begin{prop}\label{Prop:solution+universality=data-in-the-closure-apdx}
    If \Cref{ass:embeddings,ass:universality} hold, then \Cref{ass:data-in-the-closure} holds, i.e.
    \[(h, q) \in \overline{\ran A_\rho} \subset L^2(\rho_X) \times L^2(\rho_Z),\]
    or in other words, for any $\epsilon > 0$, there exists $u \in \cH$ such that $\| A_\rho u - (h, q) \|_\rho \leq \epsilon$.
\end{prop}
\begin{proof}
We have $u^* \in \cF$. \Cref{ass:universality} implies that the inclusion
\[\function{i}{\cH}{\cF}{f}{f}\]
is well-defined and bounded, and has a dense range. In particular, $u^* \in \overline{\ran i}$. We thus have
    \[\mathcal A_\rho u^* \in \mathcal A_\rho \big( \overline{\ran i}\big)\]
    which implies since $\mathcal A_\rho$ is continuous
    \[\mathcal A_\rho u^* \in \overline{ \ran (\cA_\rho \circ i)}.\]
    Now recalling the definition \eqref{eq:def-operator-A_rho-apdx} of $A_\rho$, we see that $A_\rho = \cA_\rho \circ i$, thus we have
    \[(h, q) \in \overline{\ran A_\rho}.\]
\end{proof}

\begin{coro}\label{main-thm-apdx}
    Under \Cref{ass:diff-eval,ass:embeddings,ass:universality,ass:bounded-data,ass:bounded-features}, \Cref{coro:asymptotic-convergence-apdx} holds, i.e. for any choice of $\lambda_{n,m} > 0$, for all $n,m \in \NN$, such that
    \[
    \lambda_{n,m}\to 0,
    \qquad 
    \frac{\log N}{\lambda_{n,m}^3 N}\to 0
    \qquad \text{as } n,m\to\infty ,
\]
where $N = \min(n,m)$, then, almost surely, we have the asymptotic convergence
    \[ \|A_\rho \widehat u_{\lambda_{n,m}} - (h, q)\|_\rho  \quad \underset{n,m \rightarrow + \infty}{\longrightarrow} \quad 0.\]
   By definition of $A_\rho, (h,q)$ and the norm $\| \cdot \|_\rho$, this is equivalent to 
    \begin{equation}\label{eq:main-thm-cases-apdx}
\begin{cases}
\| \widehat u_{\lambda_{n,m}} - u^* \|_{L^2(\rho_X)} \ \underset{n,m \rightarrow \infty}{\longrightarrow} \ 0  \\
\| \cD \widehat u_{\lambda_{n,m}}  - \cD u^*\|_{L^2(\rho_Z)} \ \underset{n,m \rightarrow \infty}{\longrightarrow}  \ 0,
\end{cases}
\end{equation}
i.e. \Cref{main-theorem} is satisfied.

\end{coro}
\begin{proof}
    \Cref{coro:asymptotic-convergence-apdx} holds under \Cref{ass:diff-eval,ass:bounded-data,ass:bounded-features,ass:data-in-the-closure}, so the only missing part is \Cref{ass:data-in-the-closure}. \Cref{Prop:solution+universality=data-in-the-closure-apdx} shows that \Cref{ass:embeddings,ass:universality} imply \Cref{ass:data-in-the-closure}.
   
\end{proof}

\section{Sobolev setting}\label{sec:Sobolev-setting-apdx}

In this appendix, we focus on the setting where $\cD$ is a linear differential operator and $\cF$ is a Sobolev space. Such a mathematical setting is typical for PDEs, which are ubiquitous in scientific applications. We show that our assumptions --- formulated with generality in mind --- are satisfied in such a setting, which demonstrates their practicality. 

\subsection{Generalities}\label{sec:sobolev-setting-generalities}

We now consider a bounded Lipschitz domain $\Omega \subset \RR^d$ and set $\cX = \overline{\Omega}$. Fix $s \in \mathbb{N}$, $s \geq 1$. Let $\cF = H^s(\Omega)$ and $\cG = L^2(\Omega)$, and consider the differential operator $\cD$ of \eqref{eq:diff-op-order-s}, where $c_\alpha \in C(\overline{\Omega})$.
 The first thing to show is that $\cD$ defines a bounded operator from $\cF$ to $\cG$.

\begin{lem}\label{lem:diff-op-is-bounded-from-H^s-to-L^2}
The operator $\cD$ is well-defined and bounded from $H^s(\Omega)$ to $L^2(\Omega)$.
\end{lem}
\begin{proof}
Let \(u \in H^s(\Omega)\). By definition of \(H^s(\Omega)\), for every multi-index
\(\alpha\) with \(|\alpha|\leq s\), the weak derivative \(\partial^\alpha u\)
belongs to \(L^2(\Omega)\).

Since \(c_\alpha \in C(\overline{\Omega})\) and \(\overline{\Omega}\) is compact,
we have \(c_\alpha \in L^\infty(\Omega)\). Hence
\[
    c_\alpha \partial^\alpha u \in L^2(\Omega),
    \qquad |\alpha|\leq s,
\]
and
\[
    \|c_\alpha \partial^\alpha u\|_{L^2(\Omega)}
    \leq \|c_\alpha\|_{L^\infty(\Omega)}
        \|\partial^\alpha u\|_{L^2(\Omega)}.
\]
Since the sum defining \(\mathcal D u\) is finite, it follows that
\(\mathcal D u \in L^2(\Omega)\). Thus \(\mathcal D\) is well-defined from
\(H^s(\Omega)\) to \(L^2(\Omega)\).

Moreover, we have
\[
    \|\mathcal D u\|_{L^2(\Omega)}
    \leq
    \sum_{|\alpha|\leq s}
    \|c_\alpha \partial^\alpha u\|_{L^2(\Omega)}
    \leq
    \sum_{|\alpha|\leq s}
    \|c_\alpha\|_{L^\infty(\Omega)}
    \|\partial^\alpha u\|_{L^2(\Omega)}.
\]
Using the Cauchy-Schwarz inequality, we obtain
\[
    \|\mathcal D u\|_{L^2(\Omega)}
    \leq
    \left(
        \sum_{|\alpha|\leq s}
        \|c_\alpha\|_{L^\infty(\Omega)}^2
    \right)^{1/2}
    \left(
        \sum_{|\alpha|\leq s}
        \|\partial^\alpha u\|_{L^2(\Omega)}^2
    \right)^{1/2}.
\]
Therefore
\[
    \|\mathcal D u\|_{L^2(\Omega)}
    \leq C_{\mathcal D} \|u\|_{H^s(\Omega)},
\]
where
\[
    C_{\mathcal D}
    :=
    \left(
        \sum_{|\alpha|\leq s}
        \|c_\alpha\|_{L^\infty(\Omega)}^2
    \right)^{1/2}.
\]
Hence \(\mathcal D : H^s(\Omega) \to L^2(\Omega)\) is bounded.
\end{proof}

\begin{df}[$C^s_0$ universality]\label{def-C_0^s-universality-apdx}
    Let $U$ be an open subset of $\RR^d$, let $K \colon U \times U \to \RR$
    be a kernel, and let $\cH$ denote its associated RKHS. We say that $K$
    is \emph{$C_0^s$-universal on $U$} if for any $f \in C_0^s(U)$ and any
    $\epsilon > 0$, there exists $f_\cH \in \cH$ such that for every
    $\alpha \in \NN^d$ with $|\alpha| \leq s$,
    \[
        \sup_{x \in U} |\partial^\alpha f(x) - \partial^\alpha f_\cH(x)|
        \leq \epsilon,
    \]
    i.e.\ $\cH \cap C_0^s(U)$ is dense in $C_0^s(U)$ for the norm
    $\|f\|_{C_0^s(U)} := \max_{|\alpha|\leq s}\sup_{x\in U}|\partial^\alpha f(x)|$.
\end{df}

\begin{lem}\label{lem:C_0^s-universal-implies-sobolev-density}
    Consider $\Omega$ a bounded Lipschitz domain and $\cF := H^s(\Omega)$. If $K \in C^{2s}(\cX \times \cX)$ is the restriction to $\cX$ of a $C_0^s$-universal kernel over $\RR^d$, then \Cref{ass:universality} holds, i.e. $\cH$ is continuously and densely embedded in $H^s(\Omega)$.
\end{lem}

\begin{proof}
The embedding is well-defined and bounded: since
    $K \in C^{2s}(\cX \times \cX)$, \Cref{prop:derivatives-reproducing-prop}
    gives $\cH \hookrightarrow C^s(\cX)$ with
    $\sup_{x \in \cX}|\partial^\alpha u(x)| \leq
    \sup_{x\in\cX}\big(\partial_1^\alpha \partial_2^\alpha K(x,x)\big)^{1/2}
    \|u\|_\cH$ for $|\alpha| \leq s$, and since $\Omega$ is bounded this
    yields $\|u\|_{H^s(\Omega)} \leq c \, \|u\|_\cH$ for a constant $c$
    depending on $K$, $s$, $d$ and $|\Omega|$. It remains to prove density.
    Let us denote by $K$ both the kernel on $\cX$ and its extension to $\RR^d$. Let $\cH_0$ denote the RKHS associated to $K$ over $\RR^d$. We can see $\cH$ as a closed subspace of $\cH_0$, defined by:
    \[\cH = \overline{\Span\{K_x , x \in \cX \}} \subset \cH_0.\]
    
    Consider a function $f \in H^s(\Omega)$, and let us try to approximate it by an element of $\cH$. Since $\Omega$ is a Lipschitz domain, there exists $\tilde f \in H^s(\RR^d)$ such that $\tilde f_{|\Omega} = f$ \cite[Theorem 5]{stein1970singular}. Let $\epsilon > 0$, by density of $C^\infty_c(\RR^d)$ in $H^s(\RR^d)$, there exists $\phi \in C^\infty_c(\RR^d)$ such that $\| \phi -  \tilde f \|_{H^s} < \epsilon $. Since $K$ is $C^s_0$-universal over $\RR^d$, there now exists $f_{\cH_0} \in \cH_0$ such that $\forall |\alpha|\leq s$, $\forall x \in \RR^d$, $| \partial^\alpha f_{\cH_0}(x) - \partial^\alpha\phi(x) | < \epsilon$. We thus have
    \begin{align*}
        \sum_{|\alpha|\leq s} \int_\Omega | \partial^\alpha f_{\cH_0}(x) - \partial^\alpha\phi(x) |^2 dx \leq \binom{d+s}{d} |\Omega|\epsilon^2.
    \end{align*}
    Let us denote by $f_{\cH}$ the orthogonal projection of $f_{\cH_0}$ onto $\cH$. Since for $x \in \Omega$ we have $K_x \in \cH$, we have $f_{\cH}(x) = \langle f_{\cH} , K_x \rangle_\cH = \langle f_{\cH_0} , K_x \rangle_{\cH_0} = f_{\cH_0}(x)$. If we denote by $ \phi|_{\Omega}$ the restriction of $\phi$ to $\Omega$, we thus have
    \[\|f_{\cH} - \phi|_{\Omega} \|_{H^s} \leq  \sqrt{\binom{d+s}{d} |\Omega|}\epsilon. \]

    We can finally bound
    \begin{align*}\|f_{\cH} - f \|_{H^s} & \leq \|f_{\cH} -  \phi |_{\Omega} \|_{H^s} + \| \phi_{|\Omega} - f \|_{H^s} \\
    & \leq \|f_{\cH} -  \phi |_{\Omega} \|_{H^s} + \| \phi - \tilde f \|_{H^s} \\
    & \leq \left( \sqrt{\binom{d+s}{d} |\Omega|} + 1 \right) \epsilon. \qedhere
    \end{align*}
\end{proof}

\begin{lem}\label{lem:C^2s-on-compact-implies-bounded-features}
    If $K \in C^{2s}(\cX \times \cX)$, and $\cD$ is a linear differential operator of the form \eqref{eq:diff-op-order-s} then \Cref{ass:bounded-features} holds.
\end{lem}

\begin{proof}    
Since by assumption $\cX$ is compact, we can define
\begin{equation}\label{def-of-kappa}\kappa = \max_{x \in \cX} \|K_x \|_\cH =  \sqrt{\max_{x \in  \cX} K(x,x)}.\end{equation}
Now assume that the differential operator is of the form \eqref{eq:diff-op-order-s}, i.e. it can be written as
\begin{equation}    
 \cD = \sum_{|\alpha| \leq s} c_\alpha \partial^\alpha,\end{equation}
 for some integer $s\geq1$ and where the $c_\alpha \colon \cX \rightarrow \RR$ are continuous coefficient functions. For $x,y \in \Input$, let us write 
\[\cD_1 \cD_2 K(x,y) = \sum_{|\alpha|\leq s} \sum_{|\alpha'| \leq s} c_\alpha(x) c_{\alpha'}(y) \partial_1^\alpha \partial_2^{\alpha'} K(x,y).\] 
Since $K$ is $C^{2s}$, for any $|\alpha|\leq s$, $|\alpha'|\leq s$, we observe that $\partial_1^\alpha \partial_2^{\alpha'} K(x,y)$ is continuous and thus bounded on the compact $\cX \times \cX$. Since $c_\alpha (x)$ and $c_{\alpha'}(y)$ are bounded by assumption, we can also define
\begin{equation}\label{def-of-kappa-D}\kappa_\cD = \max_{z \in \cX} \|K^\cD_z \|_\cH = \sqrt{\max_{z \in \cX} \cD_1 \cD_2 K(z,z)}.
\end{equation}
\end{proof}

\subsection{In-domain sampling}\label{sec:in-domain-sampling-apdx}

We now turn to \Cref{ass:embeddings}. Let us first
clarify the meaning of this assumption. The spaces \(\mathcal F\) and
\(\mathcal G\) are Hilbert spaces of functions on \(\mathcal X\), possibly
defined only up to almost-everywhere equivalence, as in \(L^2\) or Sobolev
spaces. Hence, if \(\rho_X\) is singular with respect to the reference measure
defining \(\mathcal F\), the expression \([u]_{\rho_X} \in L^2(\rho_X)\) need
not be well-defined for an arbitrary \(u \in \mathcal F\). For instance, if
\(\mathcal F = L^2(\mathcal X)\) and \(\rho_X = \delta_x\), then
\([u]_{\rho_X}\) would be determined by the pointwise value \(u(x)\), which is
not defined for a general \(L^2\)-equivalence class.

We therefore define the map \(\mathcal F \to L^2(\rho_X)\) by density, starting
from continuous representatives. Let
\[
D_{\mathcal F}
:=
\mathcal F \cap \cC^0(\mathcal X),
\]
where this is understood as the subspace of elements of \(\mathcal F\) admitting
a continuous representative. For \(u \in D_{\mathcal F}\), the class
$[u]_{\rho_X}$ is well-defined, and we have $\| u \|_{L^2(\rho_X)} < + \infty$ because $\cX$ is compact, so $[u]_{\rho_X} \in L^2(\rho_X)$. If
\[
\|u\|_{L^2(\rho_X)}
\leq
C \|u\|_{\mathcal F}
\qquad
\text{for all } u \in D_{\mathcal F},
\]
and \(D_{\mathcal F}\) is dense in \(\mathcal F\), then the map
\[
u \in D_{\mathcal F}
\longmapsto
[u]_{\rho_X} \in L^2(\rho_X)
\]
extends uniquely to a bounded linear operator from \(\mathcal F\) to
\(L^2(\rho_X)\). The corresponding map \(\mathcal G \to L^2(\rho_Z)\) is
defined analogously.

In \Cref{lem:in-domain-sampling-embedding-ass} below, we prove that \Cref{ass:embeddings} holds for in-domain sampling, where both $X$ and $Z$ take values inside $\Omega$.
\begin{lem}\label{lem:in-domain-sampling-embedding-ass}
    Assume that the distributions $\rho_X$ and $\rho_Z$ are absolutely continuous with respect to the Lebesgue measure on $\Omega$, with bounded densities. Then, \Cref{ass:embeddings} holds.
\end{lem}

\begin{proof}
    Since $\rho_X$ and $\rho_Z$ have bounded densities, we can bound the $L^2(\rho_X)$ and $L^2(\rho_Z)$ norms by the standard $L^2(\Omega)$ norm. Combined with the canonical embedding $H^s(\Omega) \hookrightarrow L^2(\Omega)$, we see that \Cref{ass:embeddings} is satisfied. 
\end{proof}

\begin{prop}\label{prop:main-theorem-in-domain-sobolev}
Under the standing assumptions of \Cref{sec:sobolev-setting-generalities}, assume that $\rho_X$ and $\rho_Z$ have bounded densities with respect to
Lebesgue measure of $\Omega$, that $K\in C^{2s}(\cX\times\cX)$ is the restriction of a
$C_0^s$-universal kernel on $\RR^d$, and that
\Cref{ass:bounded-data} holds. Then all the assumptions of
\Cref{main-theorem} are satisfied.
\end{prop}
\begin{proof}
    \Cref{lem:smooth-kernel-implies-bounded-diff-eval-apdx} shows that \Cref{ass:diff-eval} holds. \Cref{lem:in-domain-sampling-embedding-ass} shows that \Cref{ass:embeddings} holds. \Cref{lem:C_0^s-universal-implies-sobolev-density} shows that \Cref{ass:universality} holds. \Cref{lem:C^2s-on-compact-implies-bounded-features} shows that \Cref{ass:bounded-features} holds. Finally, \Cref{ass:bounded-data} is a standard assumption that we take independently from the rest. All the assumptions of \Cref{main-theorem} are thus satisfied, and the theorem applies.
\end{proof}

\subsubsection{Stronger convergence in the elliptic case}\label{sec:in-domain-stronger-cvg-elliptic-apdx}

In this section, we consider an operator $\cD$ of order $s=2$ in divergence form:
\begin{equation}\label{D-divergence-form-apdx} \cD u(x) = - \sum_{i,j=1}^d \frac{\partial}{\partial x_j} \left(a^{ij}(x) \frac{\partial}{\partial x_i} u(x)\right) + \sum_{i=1}^d b^i(x) \frac{\partial}{\partial x_i} u(x) + c(x) u(x),\end{equation}
where $a^{ij} \in C^1(\overline{\Omega})$, $b^i \in C(\overline{\Omega})$ and $c \in C(\overline{\Omega})$ are coefficient functions and $a^{ij} = a^{ji}$ for all $i,j$. Note in particular that $\cD$ is of the form \eqref{eq:diff-op-order-s}.
\begin{df}[\citealp{evans2010partial}, Section 6.1]\label{def:uniform-ellipticity-apdx}
    We say that $\cD$ is uniformly elliptic if there exists a constant $\theta>0$ such that
    \[\sum_{i,j=1}^d a^{ij}(x) \xi_i\xi_j \geq \theta |\xi|^2 \] 
    for almost every $x \in \Omega$ and all $\xi \in \RR^d$.
\end{df}

In the remainder of this subsection, we assume that
$\mathcal D$ is uniformly elliptic. The following classical result on elliptic PDEs allows us to control higher-order norms of $u$ (here, $H^2(V)$ for an open set $V$ compactly embedded in $\Omega$) thanks to the norm of $\cD u$ (here, $L^2$).

\begin{prop}[\citealp{evans2010partial}, Section 6.3, Theorem 1]\label{evans-elliptic-regularity-apdx}
 Let $f \in L^2(\Omega)$ and assume that $u \in H^1(\Omega)$ is a weak solution of the PDE $\cD u = f$ on $\Omega$. Then $u \in H^2_{loc}(\Omega)$ and for each open set $V$ satisfying $\overline{V} \subset \Omega$, there exists $C_V > 0$ such that we have
    \[\| u \|_{H^{2}(V)} \leq C_V ( \| u \|_{L^2(\Omega)} + \| f \|_{L^2(\Omega)}).\]
\end{prop}

\begin{coro}\label{elliptic-regularity-domain-apdx}
Assume that $\rho_X$ and $\rho_Z$ are both absolutely continuous with respect to the Lebesgue measure on $\Omega$, with densities bounded away from $0$, so that the norms $\| \cdot \|_{L^2(\rho_X)}$ and $\| \cdot \|_{L^2(\rho_Z)}$ are stronger than $\| \cdot \|_{L^2( \Omega)}$. Then for any open $V$ such that $\overline{V} \subset \Omega$, there exists $C_V > 0$ such that for all $u \in H^2(\Omega)$, we have
\begin{equation}\label{elliptic-regularity-domain-eq-apdx}\| u \|_{H^{2}(V)} \leq C_V ( \| u \|_{L^2(\rho_X)} + \| \cD u \|_{L^2(\rho_Z)}).\end{equation}
\end{coro}
\begin{proof}
For any $u \in H^2(\Omega)$, if we denote $f := \cD u$, we have $f \in L^2(\Omega)$ which allows us to apply \Cref{evans-elliptic-regularity-apdx} with $u$ and $f$ and then bound $\| u \|_{L^2(\Omega)} $ and $\| \cD u\|_{L^2(\Omega)}$ by $\| u \|_{L^2(\rho_X)}$ and $\| \cD u \|_{L^2(\rho_Z)}$ respectively to obtain \eqref{elliptic-regularity-domain-eq-apdx}.
\end{proof}

The preceding results allow us to state the following Sobolev convergence result in the elliptic case.
\begin{coro}\label{Coro:cvg-H^2(V)-apdx}
    Assume the hypotheses of \Cref{prop:main-theorem-in-domain-sobolev} with $s=2$.
Assume moreover that $\mathcal D$ is uniformly elliptic and the densities of $\rho_X$ and $\rho_Z$ with respect to the Lebesgue measure of $\Omega$, which are bounded from above by the hypotheses of \Cref{prop:main-theorem-in-domain-sobolev}, are also bounded away from $0$, so that the norms $\| \cdot \|_{L^2(\rho_X)}$ and $\| \cdot \|_{L^2(\rho_Z)}$ are equivalent to $\| \cdot \|_{L^2( \Omega)}$. Then, for any sequence $(\lambda_{n,m})$ satisfying \eqref{eq:condition-lambda}, for every open $V$ such that $\overline{V} \subset \Omega$, we almost surely have
    \[\|\uh_{\lambda_{n,m}} - u^* \|_{H^2(V)} \ \underset{n,m \rightarrow \infty}{\longrightarrow} \ 0. \]
\end{coro}
\begin{proof}
    By \Cref{prop:main-theorem-in-domain-sobolev}, \Cref{main-theorem} applies: for any sequence $(\lambda_{n,m})$ satisfying \eqref{eq:condition-lambda}, we almost surely have
    \[ \begin{cases}
\| \widehat u_{\lambda_{n,m}} - u^* \|_{L^2(\rho_X)} \ \underset{n,m \rightarrow \infty}{\longrightarrow} \ 0  \\
\| \cD \widehat u_{\lambda_{n,m}}  - \cD u^*\|_{L^2(\rho_Z)} \ \underset{n,m \rightarrow \infty}{\longrightarrow}  \ 0.
    \end{cases}
\]
For any open set $V$ such that $\overline{V} \subset \Omega$, using \Cref{elliptic-regularity-domain-apdx} with $v_{n,m} = \widehat u_{\lambda_{n,m}} - u^* \in H^2(\Omega)$ yields the result.
\end{proof}

\subsection{Boundary sampling}\label{sec:boundary-setting-apdx}

In \Cref{lem:boundary-sampling-embedding-ass} below, we prove that \Cref{ass:embeddings} holds for the boundary sampling scenario, where $X$ takes values only on the boundary $\partial \Omega$ (while $Z$ still takes values inside $\Omega$). Such a scenario is slightly less straightforward as we cannot rely on absolute continuity (with respect to the Lebesgue measure on $\Omega$), and we instead rely on results from trace theory.
\begin{lem}\label{lem:boundary-sampling-embedding-ass}
    Assume that $\rho_X$ is absolutely continuous with respect to the Hausdorff measure of $\partial \Omega$. Assume that $\rho_Z$ is absolutely continuous with respect to the Lebesgue measure on $\Omega$. Assume that both densities are bounded. Then, \Cref{ass:embeddings} holds.
\end{lem}

\begin{proof}
    Since $\rho_Z$ has a bounded density, we can bound the $L^2(\rho_Z)$ norm by the $L^2(\Omega)$ norm, so the embedding $\cG \to L^2(\rho_Z)$ is indeed well-defined and bounded. 

    Since $\rho_X$ has a bounded density, we can bound the $L^2(\rho_X)$ norm by the standard $L^2(\partial \Omega)$ norm. Since $\Omega$ is a Lipschitz domain, we can define the trace operator 
    \[\Tr : H^1(\Omega) \to H^{1/2}(\partial \Omega),\]
    which for any $u \in H^1(\Omega) \cap \cC^0(\overline{\Omega})$, coincides with the restriction of $u$ to the boundary. By combining the trace operator with the embeddings $H^s(\Omega) \hookrightarrow H^1(\Omega)$ (before the trace operator) and $H^{1/2}(\partial \Omega) \hookrightarrow L^2(\partial \Omega)$ (after the trace operator) and finally the bounded map $L^2(\partial \Omega) \to L^2(\rho_X)$, we get that \Cref{ass:embeddings} holds.
\end{proof}

\begin{prop}\label{prop:main-theorem-boundary-sobolev}
Under the standing assumptions of
\Cref{sec:sobolev-setting-generalities}, assume that $\rho_X$ is absolutely continuous with respect to the Hausdorff measure of $\partial \Omega$, with bounded density, and assume that $\rho_Z$ is absolutely continuous with respect to the Lebesgue measure on $\Omega$, with bounded density. Assume moreover that
$K\in C^{2s}(\cX\times\cX)$ is the restriction of a
$C_0^s$-universal kernel on $\RR^d$, and that
\Cref{ass:bounded-data} holds. Then all the assumptions of
\Cref{main-theorem} are satisfied.
\end{prop}

\begin{proof}
   \Cref{lem:smooth-kernel-implies-bounded-diff-eval-apdx} shows that \Cref{ass:diff-eval} holds. \Cref{lem:boundary-sampling-embedding-ass} shows that \Cref{ass:embeddings} holds. \Cref{lem:C_0^s-universal-implies-sobolev-density} shows that \Cref{ass:universality} holds. \Cref{lem:C^2s-on-compact-implies-bounded-features} shows that \Cref{ass:bounded-features} holds. Finally, \Cref{ass:bounded-data} is a standard assumption that we take independently from the rest. All the assumptions of \Cref{main-theorem} are thus satisfied, and the theorem applies. 
\end{proof}

\subsubsection{Strong convergence in the elliptic case}\label{sec:boundary-setting-strong-cvg-apdx}

Throughout this subsubsection, we assume that $s=2$ and that
$\mathcal D$ is the operator defined in
\eqref{D-divergence-form-apdx}. We further assume that $a^{ij}, b^i, c \in C^\infty(\overline{\Omega})$, that $\cD$ is uniformly elliptic over $\Omega$ (cf. \Cref{def:uniform-ellipticity-apdx}), and that the boundary $\partial \Omega$ is a $(d-1)$-dimensional smooth manifold, $\Omega$ being locally on one side of $\partial \Omega$ (such regularity is needed to apply the regularity estimates from \citet{lions2012non}). We consider the problem
\begin{equation}\label{boundary-value-problem-apdx}
    \begin{cases}
    \cD u(x) = q(x) & x \in \Omega \\
    u(x) = h(x) & x \in \partial \Omega,
\end{cases} \end{equation}
where $(q,h) \in L^2(\Omega) \times H^{3/2}(\partial \Omega)$.
Let us assume that $0$ is not a Dirichlet eigenvalue for the operator $\mathcal{D}$ in $\Omega$, which, by the Fredholm alternative \citep[see][]{evans2010partial}, guarantees that for any $(q,h) \in L^2(\Omega) \times H^{3/2}(\partial \Omega)$, there exists a unique solution $u \in H^2(\Omega)$ to \eqref{boundary-value-problem-apdx}. This is the case for instance for the Laplacian $\cD = - \Delta$.

\begin{prop}[\citealp{lions2012non}, Theorem 7.4]\label{lions-magenes-elliptic-regularity-boundary-apdx}
    We have, for any $(q,h) \in L^2(\Omega) \times H^{3/2}(\partial \Omega)$, for any $u \in H^2(\Omega)$ solution of \eqref{boundary-value-problem-apdx},
\begin{equation}\label{lions-magenes-boundary-regularity-eq-apdx}
    \|u\|_{H^{1/2}(\Omega)} \leq c \left(  \| h\|_{L^2(\partial \Omega)} + \|q \|_{\Xi^{-3/2}(\Omega)} \right), \end{equation}
where $c > 0$ and $\| \cdot \|_{\Xi^{-3/2}(\Omega)}$ is a norm weaker than $\| \cdot \|_{L^2(\Omega)}$.
\end{prop}

\begin{rem}
    Given the regularity of $u$, $q$ and $h$, higher order estimates also hold, such as $\|u\|_{H^{2}(\Omega)} \leq c \left(  \| h \|_{H^{3/2}(\partial \Omega)} + \|q \|_{L^2(\Omega)} \right)$.
     However, although $h$ is more regular than $L^2$, we only have the $L^2$ convergence of $\Tr u$ to $h$, so the estimate \eqref{lions-magenes-boundary-regularity-eq-apdx} is the strongest we can use in our context.
\end{rem}

\begin{coro}\label{boundary-regularity-data-distributions-apdx}
        Assume that $\rho_X$ (resp. $\rho_Z$) is absolutely continuous with respect to the $(d-1)$-dimensional Hausdorff measure on $\partial \Omega$ (resp. the Lebesgue measure on $\Omega$), and assume that both densities are bounded away from $0$. Then there exists $C > 0$ such that for any $u \in H^2(\Omega)$, we have
    \begin{equation}\label{elliptic-regularity-data-distribution-apdx}
    \| u \|_{H^{1/2}(\Omega)} \leq C ( \| \Tr u \|_{L^2(\rho_X)} + \| \cD u \|_{L^2(\rho_Z)}).
    \end{equation}
\end{coro}
\begin{proof}
    Let $u \in H^2(\Omega)$. If we denote $q = \cD u$ and $h = \Tr u$, we have $q \in L^2(\Omega)$, $h \in H^{3/2}(\partial \Omega)$ and by definition, $u$ is a solution of \eqref{boundary-value-problem-apdx} for that specific choice of $(q,h)$. 
   Because the density of $\rho_X$ is bounded away from $0$, we can write 
   \[\|h\|_{L^2(\partial \Omega)} = \|\Tr u \|_{L^2(\partial \Omega)}  \leq c_1 \|\Tr u \|_{L^2(\rho_X)}.\]
    The $L^2(\Omega)$ norm being stronger than the $\Xi^{-3/2}(\Omega)$ norm and the density of $\rho_Z$ being bounded away from $0$, we can write
    \[\|q\|_{\Xi^{-3/2}(\Omega)} = \|\cD u \|_{\Xi^{-3/2}(\Omega)} \leq c_2 \|\cD u\|_{L^2(\Omega)}  \leq c_3\|\cD u\|_{L^2(\rho_Z)}.\]
     Combining these with \Cref{lions-magenes-elliptic-regularity-boundary-apdx} proves \eqref{elliptic-regularity-data-distribution-apdx}.
\end{proof}

\begin{coro}\label{Coro:cvg-H^1/2(Omega)-apdx}
    Assume the hypotheses of
\Cref{prop:main-theorem-boundary-sobolev} with $s=2$, together with the
standing assumptions of
\Cref{sec:boundary-setting-strong-cvg-apdx}. Assume moreover that the
densities of $\rho_X$ and $\rho_Z$ are bounded away from zero, so that the norms $\| \cdot \|_{L^2(\rho_X)}$ and $\|\cdot \|_{L^2(\rho_Z)}$ are respectively equivalent to $\| \cdot \|_{L^2(\partial \Omega)}$ and $\| \cdot \|_{L^2(\Omega)}$. Then, for any sequence $(\lambda_{n,m})$ satisfying \eqref{eq:condition-lambda}, we have that, almost surely
\[ \|\widehat u_{\lambda_{n,m}} - u^*\|_{H^{1/2}(\Omega)}  \quad \underset{n,m \rightarrow + \infty}{\longrightarrow} \quad 0.\]
In particular, we have
\[ \|\widehat u_{\lambda_{n,m}} - u^*\|_{L^2(\Omega)}  \quad \underset{n,m \rightarrow + \infty}{\longrightarrow} \quad 0.\]
\end{coro}
\begin{proof}
    As we have shown above, in this setting, \Cref{main-theorem} holds, i.e. for any regularizing sequence $(\lambda_{n,m})$ satisfying \eqref{eq:condition-lambda}, we have almost surely
    \[ \begin{cases}
\| \widehat u_{\lambda_{n,m}} - h \|_{L^2(\rho_X)} \ \underset{n,m \rightarrow \infty}{\longrightarrow} \ 0  \\
\| \cD \widehat u_{\lambda_{n,m}}  - q\|_{L^2(\rho_Z)} \ \underset{n,m \rightarrow \infty}{\longrightarrow}  \ 0.
    \end{cases}
\]
Considering $v_{n,m} = \widehat u_{\lambda_{n,m}} - u^*$, we can rewrite this as 
\[ \begin{cases}
\|  \Tr v_{n,m} \|_{L^2(\rho_X)} \ \underset{n,m \rightarrow \infty}{\longrightarrow} \ 0  \\
\| \cD v_{n,m}\|_{L^2(\rho_Z)} \ \underset{n,m \rightarrow \infty}{\longrightarrow}  \ 0.
    \end{cases}
\]
Using \Cref{boundary-regularity-data-distributions-apdx} with $v_{n,m}$, we immediately obtain
\[ \|v_{n,m}\|_{H^{1/2}(\Omega)}  \quad \underset{n,m \rightarrow + \infty}{\longrightarrow} \quad 0,\]
and in particular since the $H^{1/2}$ norm is stronger than the $L^2$ norm, we obtain 
\[ \|v_{n,m}\|_{L^2(\Omega)}  \quad \underset{n,m \rightarrow + \infty}{\longrightarrow} \quad 0.\]
\end{proof}

\Cref{Coro:cvg-H^1/2(Omega)-apdx} establishes $H^{1/2}$ convergence on $\Omega$. If we restrict to a smaller set $V \subset \Omega$, analogously to \Cref{Coro:cvg-H^2(V)-apdx} in \Cref{sec:in-domain-sampling-apdx}, we can prove a stronger ($H^2$) convergence on $V$.

\begin{coro}\label{coro:cvg-H^2(V)bis-apdx}
    Assume the hypotheses of
\Cref{prop:main-theorem-boundary-sobolev} with $s=2$, together with the
standing assumptions of
\Cref{sec:boundary-setting-strong-cvg-apdx}. Assume moreover that the
densities of $\rho_X$ and $\rho_Z$ are bounded away from zero, so that the norms $\| \cdot \|_{L^2(\rho_X)}$ and $\|\cdot \|_{L^2(\rho_Z)}$ are respectively equivalent to $\| \cdot \|_{L^2(\partial \Omega)}$ and $\| \cdot \|_{L^2(\Omega)}$. Then, for any regularizing sequence $(\lambda_{n,m})$ satisfying \eqref{eq:condition-lambda}, we have almost surely that for any open $V$ such that $\overline{V} \subset \Omega$, we have
    \[\|\uh_{\lambda_{n,m}} - u^* \|_{H^2(V)} \ \underset{n,m \rightarrow \infty}{\longrightarrow} \ 0. \]
\end{coro}
\begin{proof}
    Let us consider again $v_{n,m} = \widehat u_{\lambda_{n,m}} - u^*$, as in the proof of \Cref{Coro:cvg-H^1/2(Omega)-apdx}. We established that 
    \[ \begin{cases}
\|  \Tr v_{n,m} \|_{L^2(\rho_X)} \ \underset{n,m \rightarrow \infty}{\longrightarrow} \ 0  \\
\| \cD v_{n,m}\|_{L^2(\rho_Z)} \ \underset{n,m \rightarrow \infty}{\longrightarrow}  \ 0,
    \end{cases}
\]
and that
\[ \|v_{n,m}\|_{L^2(\Omega)}  \quad \underset{n,m \rightarrow + \infty}{\longrightarrow} \quad 0,\]
which allows us to use \Cref{evans-elliptic-regularity-apdx}, and the lower density bound for $\rho_Z$, to conclude that for any open $V$ such that $\overline{V} \subset \Omega$, we have
\[\|\uh_{\lambda_{n,m}} - u^* \|_{H^2(V)} \ \underset{n,m \rightarrow \infty}{\longrightarrow} \ 0. \]
\end{proof}

\subsection{Rates for the Laplacian on periodic functions}\label{sec:laplacian-periodic-fct-rates}

In this section, we consider the example of the Laplacian on the torus, i.e.~applied to Sobolev spaces of periodic functions on $[0,1]^d$. We use this example to illustrate the source condition for the rates provided in \Cref{thm:finite-sample-rates-main,coro:simplified-rate-main}. Even though it is a Sobolev setting, periodicity creates differences with the rest of \Cref{sec:Sobolev-setting-apdx}, and in particular, the results of \Cref{sec:sobolev-setting-generalities} do not apply.

Let $\Omega = (0,1)^d$, let $\cX = \overline{\Omega}=[0,1]^d$.  
Throughout this section, we identify a function $u$ on $\cX$ satisfying
periodic boundary conditions with a function on the torus
$\mathbb{T}^d = \RR^d / \ZZ^d$, and we write its Fourier expansion as
\[
    u(x) = \sum_{k \in \ZZ^d} \widetilde{u}_k \, e^{2\pi i k \cdot x},
    \qquad
    \widetilde{u}_k := \int_{[0,1]^d} u(x) \, e^{-2\pi i k \cdot x} \, dx.
\]
For $\tau \geq 0$, we define the Sobolev space of periodic functions of
order $\tau$ as
\[
    \Hper^\tau(\cX)
    := \Big\{ u \in L^2(\cX) \,:\,
    \|u\|_{\Hper^\tau}^2
    := \sum_{k \in \ZZ^d} \lambda_k^\tau \, |\widetilde{u}_k|^2 < \infty \Big\},
    \qquad
    \lambda_k := 1 + 4\pi^2 |k|^2,
\]
endowed with the inner product
$\langle u, v \rangle_{\Hper^\tau}
= \sum_{k \in \ZZ^d} \lambda_k^\tau \, \widetilde{u}_k \overline{\widetilde{v}_k}$.
The weight $\lambda_k^\tau$ is equivalent to the more common
$(1+|k|^2)^\tau$, so this choice only rescales the norm; we fix it because
$\lambda_k$ is exactly the symbol of $\mathrm{Id} - \Delta$, i.e.\
$\lambda_k \widetilde{u}_k = \widetilde{(u - \Delta u)}_k$, which
simplifies the computations below. Note that all the statements of
this section are invariant under replacing
$\|\cdot\|_{\Hper^\tau}$ by an equivalent norm, except for the
exact values of the constants $\kappa, \kappa_\cD$ in
\Cref{lem:laplacian-periodic-assumptions-check}. In particular, $\Hper^\tau(\cX)$
can be identified with the subspace of $H^\tau(\cX)$ whose elements
satisfy periodic boundary conditions, and for $\tau > d/2$ it is an RKHS,
since by Cauchy--Schwarz
$|u(x)| \leq \big(\sum_{k} \lambda_k^{-\tau}\big)^{1/2}
\|u\|_{\Hper^\tau} < \infty$ for every $x \in \cX$.

Let us consider $\cF = \Hper^2(\cX)$, $\cG = L^2(\cX)$ and $\cD = \Delta = \sum_{i=1}^d \frac{\partial^2}{(\partial x_i)^2}$ the Laplacian, which defines a continuous operator from $\cF$ to $\cG$. Let $\cH = H^\tau_{per}(\cX)$, with $\tau > d/2 + 2$. In particular, since $\tau> d/2$, $\cH$ is an RKHS. Let us consider $\rho_X = \rho_Z = \Unif([0,1]^d)$ (which corresponds to the in-domain sampling setting of \Cref{sec:Sobolev-setting}, up to the periodicity difference). In particular $L^2(\rho_X) = L^2(\rho_Z) = L^2([0,1]^d)$.

\begin{lem}\label{lem:laplacian-periodic-assumptions-check}
    \Cref{ass:diff-eval,ass:embeddings,ass:universality,ass:bounded-features} hold.
\end{lem}
\begin{proof}
We prove \Cref{ass:diff-eval,ass:bounded-features} together. For $u \in \cH$, write the Fourier expansions
\[
u(x)=\sum_{k\in\ZZ^d}\widetilde u_k e^{2\pi i k\cdot x}, \qquad  \qquad 
\Delta u(x)
=
-4\pi^2\sum_{k\in\ZZ^d}|k|^2\widetilde u_k e^{2\pi i k\cdot x}.
\]
By Cauchy--Schwarz,
\[
|u(x)|
\leq
\left(\sum_{k\in\ZZ^d}\lambda_k^{-\tau}\right)^{1/2}
\|u\|_{H^\tau_{per}},
\]
and similarly
\[
|\Delta u(x)|
\leq
4\pi^2
\left(\sum_{k\in\ZZ^d}|k|^4 \lambda_k^{-\tau}\right)^{1/2}
\|u\|_{H^\tau_{per}}.
\]
The first series is finite since $\tau>d/2$, and the second one is finite since
\[
|k|^4 \lambda_k^{-\tau}\lesssim \lambda_k^{-(\tau-2)}
\]
and $\tau-2>d/2$. Thus, since $\tau>d/2+2$, both point evaluations
\[
u\mapsto u(x),
\qquad
u\mapsto \Delta u(x)
\]
are bounded uniformly in $x$. Hence \Cref{ass:diff-eval} holds, and \Cref{ass:bounded-features} holds with
\[
\kappa
=
\left(\sum_{k\in\ZZ^d}\lambda_k^{-\tau}\right)^{1/2},
\qquad
\kappa_\cD
=
4\pi^2
\left(\sum_{k\in\ZZ^d}|k|^4 \lambda_k^{-\tau}\right)^{1/2}.
\]

The remaining assumptions are standard. Since
\[
H^2_{per}(\cX)\hookrightarrow L^2(\cX)
\quad\text{and}\quad
\cG=L^2(\cX),
\]
and since $\rho_X=\rho_Z=\Unif([0,1]^d)$, \Cref{ass:embeddings} holds. Finally,
\[
H^s_{per}(\cX)\hookrightarrow H^2_{per}(\cX)
\]
is continuous and dense, by density of smooth periodic functions in Sobolev spaces. Hence \Cref{ass:universality} holds.
\end{proof}

In particular, \Cref{ass:embeddings} allows us to define, as in \Cref{sec:rates-src-condition}, the operators
\[
\begin{aligned}
\mathcal A_\rho:\mathcal F&\to L^2(\rho_X)\times L^2(\rho_Z) \qquad \quad \text{and} \qquad \quad A_\rho=\mathcal A_\rho\circ i, \\
u&\mapsto (u,\Delta u)
\end{aligned} 
\]
where $
i:\mathcal H\to\mathcal F, u \mapsto u$ is the canonical Sobolev embedding.

We can now consider $(h,q) = \cA_\rho u^*$, denote $L = A_\rho A_\rho^*$ and $C = \cA_\rho^* \cA_\rho$. The following lemma allows us to interpret the source condition on $(h,q)$ as an equivalent source condition on the rescaled target $C^{1/2} u^*$.
\begin{lem} \label{lem-per-source-isometry}
   There exists an isometry $U : \cF \to L^2(\rho_X) \times L^2(\rho_Z)$ with range $\ran \cA_\rho$ and a bounded positive self-adjoint operator $T: \cF \to \cF$ such that $L = UTU^*$, and furthermore $(h,q) \in L^2(\rho_X) \times L^2(\rho_Z)$ satisfies a source condition with $L$ if and only if $C^{1/2} u^* \in \cF$ satisfies an equivalent source condition defined by $T$.
\end{lem}

\begin{proof}
The map $\mathcal A_\rho$ is a continuous injective linear map. Moreover, it is a topological isomorphism from $\mathcal F$ onto its image
\[
\ran(\mathcal A_\rho)
=
\{(u,\Delta u):u\in H^2_{\mathrm{per}}(\mathcal X)\}
\subset L^2(\rho_X)\times L^2(\rho_Z).
\]
Indeed, one can easily check using Fourier decomposition that 
\[\|u\|_{L^2(\cX)}^2 + \| \Delta u \|_{L^2(\cX)}^2 \asymp \|u\|_{H_{per}^2}^2. \]
This shows that $C=\cA_\rho^*\cA_\rho$ is an isomorphism.

Now use the polar decomposition of \(\cA_\rho\). We may write
\begin{equation*}\label{eq:polar-decomp}
\cA_\rho=UC^{1/2},
\end{equation*}
where
\[
U:\mathcal F\to L^2(\rho_X) \times L^2(\rho_Z)
\]
is an isometry with range $\ran \cA_\rho$. 
Consequently, recalling $A_\rho = \cA_\rho i$,
\[
L=A_\rho A_\rho^*
=
UC^{1/2}ii^*C^{1/2}U^*.
\]

 Define
\[
T:=C^{1/2}ii^*C^{1/2}:\mathcal F\to\mathcal F.
\]
Then \(T\) is bounded, self-adjoint, and positive, and by the continuous functional calculus,
\[    L^r  = (UTU^*)^r = U T^r U^*,\]
where we used the fact that $U^* U = I_\cF$.

Assume now that $ (h,q) = L^r(f,g) $, where we may assume $(f,g) = U v$, with $v \in \cF$ (if that is not the case we may simply replace $(f,g)$ by its orthogonal projection onto $\ran \cA_\rho$, which will leave its image by $L^r$ unchanged, as $U^*$ annihilates $(\ran \cA_\rho )^\perp$).

\begin{align} (h,q) = L^r(f,g) \  & \Longleftrightarrow \ \cA_\rho u^* =  U T^r U^*(f,g) \nonumber  \\
& \Longleftrightarrow \ C^{1/2} u^* =  T^r v \label{eq:lap-torus-equivalent-src}\end{align}
with $v : = U^*(f,g)$. We thus see that $(h,q)$ satisfies a source condition with $L$ if and only if $C^{1/2} u^*$ satisfies an equivalent source condition defined by $T$.
\end{proof}

\begin{lem} \label{lem-per-source-sobolev}
Let $r>0$ and set $\sigma_r:=2+2r(\tau-2)$.
Then there exists $v\in \cF$ such that
\[
C^{1/2}u^*=T^r v
\]
if and only if
\[
u^*\in H^{\sigma_r}_{\mathrm{per}}(\cX).
\]
\end{lem}

\begin{proof}
The operator $
C=\mathcal A_\rho^*\mathcal A_\rho$
is the Fourier multiplier
\[
\widetilde{(Cu)}_k
=
c_k \widetilde u_k,
\qquad
c_k:=
\frac{1+16\pi^4|k|^4}{\lambda_k^2}.
\]
Since $c_k\asymp 1$, the operators $C^{1/2}$ and $C^{-1/2}$ are bounded isomorphisms on every $H^t_{\mathrm{per}}(\cX)$.

Next, for the embedding
\[
i:\Hper^\tau(\cX)\to H^2_{\mathrm{per}}(\cX),
\]
one checks from the Fourier definitions of the inner products that
\[
\widetilde{(ii^*u)}_k
=
\lambda_k^{2-\tau}\widetilde u_k.
\]
Hence
\[
T=C^{1/2}ii^*C^{1/2}
\]
is the Fourier multiplier
\[
\widetilde{(Tu)}_k
=
c_k\lambda_k^{2-\tau}\widetilde u_k
\asymp \lambda_k^{2-\tau}\widetilde u_k.
\]
Thus
\[
\widetilde{(T^r u)}_k \asymp
\lambda_k^{-r(\tau-2)} \widetilde u_k.
\]

Set
\[
w:=C^{1/2}u^*.
\]
Using the Fourier coefficients of $T$, we see that there exists $v\in \cF=H^2_{\mathrm{per}}(\cX)$ such that
\[
w=T^r v
\]
if and only if
\[
\sum_{k\in\mathbb Z^d}
\lambda_k^{2+2r(\tau-2)}
|\widetilde w_k|^2
<\infty,
\]
which is equivalent to
\[
w\in H^{2+2r(\tau-2)}_{\mathrm{per}}(\cX).
\]
Since $C^{1/2}$ is an isomorphism on Sobolev spaces, this is equivalent to
\[
u^*\in H^{2+2r(\tau-2)}_{\mathrm{per}}(\cX).
\]
The result follows.
\end{proof}

We can now restate and prove \Cref{Ex:Laplacian-on-the-torus}
from \Cref{sec:rates-src-condition}.

\begin{prop}[\Cref{Ex:Laplacian-on-the-torus}, restated]
\label{prop-source-periodic-apdx}
    Let $u^* \in \cF$, let $(h,q) = \cA_\rho u^*$, and let $r \in (0,1]$.
    Then \Cref{ass:src} holds with exponent $r$, i.e.\
    $(h,q) \in \ran L^r$, if and only if
    $u^* \in H^{\sigma_r}_{\mathrm{per}}(\cX)$, where
    $\sigma_r := 2 + 2r(\tau-2)$. In particular, for $r = \frac{1}{2}$,
    \Cref{ass:src} is equivalent to $u^* \in \cH$.
\end{prop}

\begin{proof}
    Fix $r \in (0,1]$. By \Cref{lem-per-source-isometry}, and in particular by \eqref{eq:lap-torus-equivalent-src}, since
    $L = U T U^*$ with $U$ an isometry with range $\ran \cA_\rho$, we
    have $(h,q) \in \ran L^r$ if and only if there exists $v \in \cF$
    such that $C^{1/2} u^* = T^r v$. By
    \Cref{lem-per-source-sobolev}, such a $v$ exists if and only
    if $u^* \in H^{\sigma_r}_{\mathrm{per}}(\cX)$.

    Finally, for $r = \frac{1}{2}$ we have $\sigma_r = \tau$, so the
    condition reads $u^* \in H^{s}_{\mathrm{per}}(\cX) = \cH$, which is
    the well-specified setting.
\end{proof}

\section{Useful results}

\begin{prop}[Hoeffding inequality in separable Hilbert spaces]\label{Hoeffding-ineq-sep-hilbert-spaces}
    Take a family $\xi_1, \dots, \xi_n: \Omega \rightarrow H$
    of independent zero mean random variables such that $\|\xi_i\|_H \leq c$, then for all $\epsilon > 0$
    \[\PP \left[ \left\| \frac{1}{n} \sum_{i=1}^n \xi_i \right\|_H > \epsilon \right] \leq 2 \exp \left( - \frac{\epsilon^2 n}{4 c^2} \right)\]
    i.e. for all $\tau > 0$, with probability at least $1 - 2 e^{-\tau}$
    \[\left\| \frac{1}{n} \sum_{i=1}^n \xi_i \right\|_H \leq \frac{2 c \sqrt{\tau}}{\sqrt{n}}.\]
\end{prop}
For the proof, we refer the reader to \citet{yurinsky2006sums, pinelis1994optimum, pinelis1999correction}.

\begin{prop}[Tropp's concentration inequality, \citealp{rudi2013sample}, Theorem A.1]
    \label{Tropp-inequality}
Let $(Z_{i})_{1\leq i\le n}$ be independent copies of the random variable $Z$ with values in the space of bounded self-adjoint operators $\mathcal{S}(\mathcal{H})$ over a separable Hilbert space $\mathcal{H}$. Define $T:=\mathbb{E}[Z]$, and let there be $S\in\mathcal{S}(\mathcal{H})$ such that $\mathbb{E}[(Z-T)^{2}] \preceq S$, and a finite number $R$ such that $\|Z\|_{\op}\leq R$ almost surely. 
Define the quantities $\alpha:=\|S\|_{1}/\|S\|_{\op}$ and $\sigma^{2}:=\|S\|_{\op}$. Then, for $0 < \delta \leq \alpha$, it holds
\begin{equation}
    \mathbb{P}\left\{ \left\| \frac{1}{n}\sum_{i=1}^{n}Z_{i}-T \right\|_{\op} \leq \frac{\beta R}{n} + \sqrt{\frac{3\beta\sigma^{2}}{n}} \right\} \geq 1-\delta \,,
\end{equation}
where $\beta:=\frac{2}{3}\log\frac{4\alpha}{\delta}$.
\end{prop}
\Cref{Tropp-inequality} originally comes from \citet{tropp2012user}, but in its current form is a simple reproduction of \citet[Theorem A.1]{rudi2013sample}.
\section{Additional Experiment Information}\label{sec:exp-appendix}

\subsection{FEM Comparisons}

We describe here in more detail the Poisson PDE used for the FEM experiments.
Define a sequence of support points $\{s_i\}_{i=1}^S$ belonging to the disk in $\mathbb{R}^2$. Letting $\nu\in\{0.5, 1.5, 2.5, \infty\}$ denote the smoothness parameter, we define $f$ as
\begin{equation*}
    f(x) = \begin{cases}
        (1 - \lVert x \rVert^2) \sum_{i=1}^S K_\nu(x, s_i) & \text{if } \lVert x \rVert \leq 1 \\
        0 & \text{otherwise}
    \end{cases}
\end{equation*}
where $K_\nu$ is the Mat{\'e}rn kernel with parameter $\nu$ (note that $\nu = 0.5$ corresponds to the Laplacian kernel and $\nu=\infty$ to the Gaussian).
The Poisson PDE is then defined on a domain $\Omega = \{x: \mathbb{R}^2 \mid \lVert x \rVert < 1\}$ with boundary $\partial\Omega = \{x: \mathbb{R}^2 \mid \lVert x \rVert = 1\}$:
\begin{equation*}
    \begin{cases}
        \Delta u(x) = q(x) & x\in\Omega \\
        u(x) = h(x) & x \in \partial\Omega
    \end{cases}
\end{equation*}
with $q(x) = \Delta f(x)$ and $h(x) = f(x)$. The four Laplacians corresponding to $\nu = 0.5, 1.5, 2.5, \infty$ are shown in \cref{fig:smooth-funcs}.

\begin{figure}
    \centering
    \includegraphics[width=0.5\linewidth]{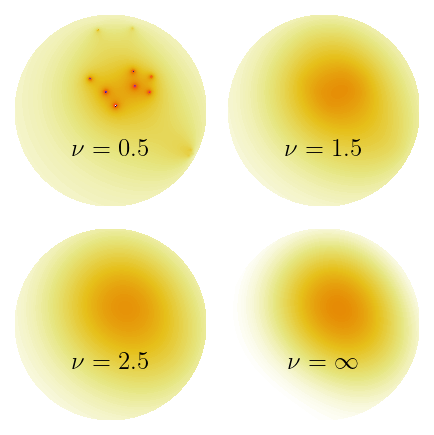}
    \caption{Four Mat{\'e}rn Laplacians with increasing smoothness.}
    \label{fig:smooth-funcs}
\end{figure}

\FloatBarrier

\bibliographystyle{plainnat}
\bibliography{biblio}

@article{chen2021solving,
  title={Solving and learning nonlinear PDEs with Gaussian processes},
  author={Chen, Yifan and Hosseini, Bamdad and Owhadi, Houman and Stuart, Andrew M},
  journal={Journal of Computational Physics},
  volume={447},
  pages={110668},
  year={2021},
  publisher={Elsevier}
}

@inproceedings{doumeche2024physics,
  title={Physics-informed machine learning as a kernel method},
  author={Doum{\`e}che, Nathan and Bach, Francis and Boyer, Claire and Biau, G{\'e}rard},
  booktitle={Proceedings of the Thirty Seventh Annual Conference on Learning Theory},
  year={2024}
}

@article{doumeche25convergence,
    author = {Nathan Doum{\`e}che and G{\'e}rard Biau and Claire Boyer},
    title = {On the convergence of {PINN}s},
    volume = {31},
    journal = {Bernoulli},
    number = {3},
    pages = {2127 -- 2151},
    year = {2025},
    doi = {10.3150/24-BEJ1799},
}

@article{batlle2025error,
  title={Error analysis of kernel/GP methods for nonlinear and parametric PDEs},
  author={Batlle, Pau and Chen, Yifan and Hosseini, Bamdad and Owhadi, Houman and Stuart, Andrew M},
  journal={Journal of Computational Physics},
  volume={520},
  pages={113488},
  year={2025},
  publisher={Elsevier}
}

@article{simon2018kernel,
  title={Kernel distribution embeddings: Universal kernels, characteristic kernels and kernel metrics on distributions},
  author={Simon-Gabriel, Carl-Johann and Sch{\"o}lkopf, Bernhard},
  journal={Journal of Machine Learning Research},
  volume={19},
  number={44},
  pages={1--29},
  year={2018}
}

@book{steinwart2008support,
author = {Steinwart, Ingo and Christmann, Andreas},
title = {Support Vector Machines},
year = {2008},
isbn = {0387772413},
publisher = {Springer Publishing Company, Incorporated},
edition = {1st}
}

@inproceedings{
li2021fourier,
title={Fourier Neural Operator for Parametric Partial Differential Equations},
author={Zongyi Li and Nikola Borislavov Kovachki and Kamyar Azizzadenesheli and Burigede liu and Kaushik Bhattacharya and Andrew Stuart and Anima Anandkumar},
booktitle={International Conference on Learning Representations},
year={2021},
url={https://openreview.net/forum?id=c8P9NQVtmnO}
}

@article{lu2021learning,
  title={Learning nonlinear operators via DeepONet based on the universal approximation theorem of operators},
  author={Lu, Lu and Jin, Pengzhan and Pang, Guofei and Zhang, Zhongqiang and Karniadakis, George Em},
  journal={Nature machine intelligence},
  volume={3},
  number={3},
  pages={218--229},
  year={2021},
  publisher={Nature Publishing Group UK London}
}

@book{lions2012non,
  title={Non-homogeneous boundary value problems and applications: Vol. 1},
  author={Lions, Jacques Louis and Magenes, Enrico},
  volume={181},
  year={2012},
  publisher={Springer Science \& Business Media}
}

@book{yurinsky2006sums,
  title={Sums and Gaussian vectors},
  author={Yurinsky, Vadim},
  year={2006},
  publisher={Springer}
}

@article{pinelis1994optimum,
  title={Optimum bounds for the distributions of martingales in Banach spaces},
  author={Pinelis, Iosif},
  journal={The Annals of Probability},
  pages={1679--1706},
  year={1994},
  publisher={JSTOR}
}

@article{pinelis1999correction,
  title={Correction:``Optimum bounds for the distributions of martingales in Banach spaces''[Ann. Probab. 22 (1994), no. 4, 1679--1706; MR 96b: 60010]},
  author={Pinelis, Iosif},
  journal={The Annals of Probability},
  volume={27},
  number={4},
  pages={2119--2119},
  year={1999},
  publisher={Institute of Mathematical Statistics}
}

@book{wendland2004scattered,
  title={Scattered data approximation},
  author={Wendland, Holger},
  volume={17},
  year={2004},
  publisher={Cambridge university press}
}

@article{narcowich2006escape,
    title={Sobolev Error Estimates and a {B}ernstein Inequality for Scattered Data Interpolation via Radial Basis Functions},
    author={Narcowich, Francis J. and Ward, Joseph D. and Wendland, Holger},
    journal={Constructive Approximation},
    volume={24},
    pages={175--186},
    issue=2,
    year=2006,
}

@article{czarnecki2017sobolev,
  title={Sobolev training for neural networks},
  author={Czarnecki, Wojciech M and Osindero, Simon and Jaderberg, Max and Swirszcz, Grzegorz and Pascanu, Razvan},
  journal={Advances in neural information processing systems},
  volume={30},
  year={2017}
}

@article{caponnetto2007optimal,
  title={Optimal rates for the regularized least-squares algorithm},
  author={Caponnetto, Andrea and De Vito, Ernesto},
  journal={Foundations of Computational Mathematics},
  volume={7},
  pages={331--368},
  year={2007},
  publisher={Springer}
}

@article{raissi2019physics,
  title={Physics-informed neural networks: A deep learning framework for solving forward and inverse problems involving nonlinear partial differential equations},
  author={Raissi, Maziar and Perdikaris, Paris and Karniadakis, George E},
  journal={Journal of Computational Physics},
  volume={378},
  pages={686--707},
  year={2019},
  publisher={Elsevier}
}

@misc{raissi2017physics,
  title={Physics informed deep learning (part i): Data-driven solutions of nonlinear partial differential equations},
  author={Raissi, Maziar and Perdikaris, Paris and Karniadakis, George Em},
  archivePrefix = {arXiv},
  eprint={1711.10561},
  year={2017}
}

@article{raissi2017gps,
  title={Machine learning of linear differential equations using Gaussian processes},
  author={Raissi, Maziar and Perdikaris, Paris and Karniadakis, George Em},
  journal={Journal of Computational Physics},
  volume={348},
  pages={683--693},
  year={2017},
  publisher={Elsevier}
}

@article{raissi2018numerical,
  title={Numerical Gaussian processes for time-dependent and nonlinear partial differential equations},
  author={Raissi, Maziar and Perdikaris, Paris and Karniadakis, George Em},
  journal={SIAM Journal on Scientific Computing},
  volume={40},
  number={1},
  pages={A172--A198},
  year={2018},
  publisher={SIAM}
}

@article{shin2020convergence,
  author       = {Shin, Yeonjong and Darbon, Jérôme and Karniadakis, George Em},
  title        = {On the Convergence of Physics Informed Neural Networks for Linear Second-Order Elliptic and Parabolic Type PDEs},
  doi          = {10.4208/cicp.oa-2020-0193},
  journal      = {Communications in Computational Physics},
  number       = {5},
  volume       = {28},
  year         = {2020},
}

@article{mishra2023estimates,
  title={Estimates on the generalization error of physics-informed neural networks for approximating PDEs},
  author={Mishra, Siddhartha and Molinaro, Roberto},
  journal={IMA Journal of Numerical Analysis},
  volume={43},
  number={1},
  pages={1--43},
  year={2023},
  publisher={Oxford University Press}
}

@article{de2024error,
  title={Error estimates for physics-informed neural networks approximating the Navier--Stokes equations},
  author={De Ryck, Tim and Jagtap, Ameya D and Mishra, Siddhartha},
  journal={IMA Journal of Numerical Analysis},
  volume={44},
  number={1},
  pages={83--119},
  year={2024},
  publisher={Oxford University Press}
}

@book{wahba1990spline,
  title={Spline models for observational data},
  author={Wahba, Grace},
  year={1990},
  publisher={SIAM}
}

@article{kimeldorf1971some,
  title={Some results on Tchebycheffian spline functions},
  author={Kimeldorf, George and Wahba, Grace},
  journal={Journal of mathematical analysis and applications},
  volume={33},
  number={1},
  pages={82--95},
  year={1971},
  publisher={Elsevier}
}

@inproceedings{fasshauer1996solving,
  title={Solving partial differential equations by collocation with radial basis functions},
  author={Fasshauer, Gregory E},
  booktitle={Proceedings of Chamonix},
  volume={1997},
  pages={1--8},
  year={1996},
}

@article{franke1998convergence,
  title={Convergence order estimates of meshless collocation methods using radial basis functions},
  author={Franke, Carsten and Schaback, Robert},
  journal={Advances in computational mathematics},
  volume={8},
  pages={381--399},
  year={1998},
  publisher={Springer}
}

@article{franke1998solving,
  title={Solving partial differential equations by collocation using radial basis functions},
  author={Franke, Carsten and Schaback, Robert},
  journal={Applied Mathematics and Computation},
  volume={93},
  number={1},
  pages={73--82},
  year={1998},
  publisher={Elsevier}
}

@article{micchelli2006universal,
  title={Universal Kernels.},
  author={Micchelli, Charles A and Xu, Yuesheng and Zhang, Haizhang},
  journal={Journal of Machine Learning Research},
  volume={7},
  number={12},
  year={2006}
}

@article{sriperumbudur2011universality,
  title={Universality, Characteristic Kernels and RKHS Embedding of Measures.},
  author={Sriperumbudur, Bharath K and Fukumizu, Kenji and Lanckriet, Gert RG},
  journal={Journal of Machine Learning Research},
  volume={12},
  number={7},
  year={2011}
}

@article{belkin2006manifold,
  title={Manifold regularization: A geometric framework for learning from labeled and unlabeled examples.},
  author={Belkin, Mikhail and Niyogi, Partha and Sindhwani, Vikas},
  journal={Journal of machine learning research},
  volume={7},
  number={11},
  year={2006}
}

@article{krishnapriyan2021characterizing,
  title={Characterizing possible failure modes in physics-informed neural networks},
  author={Krishnapriyan, Aditi S. and Gholami, Amir and Zhe, Shandian and Kirby, Robert and Mahoney, Michael W},
  journal={Advances in Neural Information Processing Systems},
  volume={34},
  year={2021}
}

@article{wang22pinns,
    title = {When and why PINNs fail to train: A neural tangent kernel perspective},
    author = {Sifan Wang and Xinling Yu and Paris Perdikaris},
    journal = {Journal of Computational Physics},
    volume = {449},
    year = {2022},
    issn = {0021-9991},
    doi = {https://doi.org/10.1016/j.jcp.2021.110768},
}

@article{dolfinx,
    author = {Baratta, Igor A. and Dean, Joseph P. and Dokken, Jørgen S. and Habera, Michal and Hale, Jack S. and Richardson, Chris N. and Rognes, Marie E. and Scroggs, Matthew W. and Sime, Nathan and Wells, Garth N.},
    doi = {10.5281/zenodo.10447666},
    journal = {preprint},
    title = {{DOLFINx: the next generation FEniCS problem solving environment}},
    year = {2023}
}

@article{fischer2020sobolev,
  title={Sobolev norm learning rates for regularized least-squares algorithms},
  author={Fischer, Simon and Steinwart, Ingo},
  journal={Journal of Machine Learning Research},
  volume={21},
  number={205},
  pages={1--38},
  year={2020}
}

@book{evans2010partial,
  title={Partial differential equations},
  author={Evans, Lawrence C},
  volume={19},
  year={2010},
  publisher={American mathematical society}
}

@article{fu2023forces,
    title={Forces are not Enough: Benchmark and Critical Evaluation for Machine Learning Force Fields with Molecular Simulations},
    author={Xiang Fu and Zhenghao Wu and Wujie Wang and Tian Xie and Sinan Keten and Rafael Gomez-Bombarelli and Tommi Jaakkola},
    journal={Transactions on Machine Learning Research},
    year={2023},
}

@misc{baptista_solving_2025,
	title = {Solving Roughly Forced Nonlinear {PDEs} via Misspecified Kernel Methods and Neural Networks},
	archivePrefix = {arXiv},
	author = {Baptista, Ricardo and Calvello, Edoardo and Darcy, Matthieu and Owhadi, Houman and Stuart, Andrew M. and Yang, Xianjin},
	year = {2025},
	eprint = {2501.17110},
}

@misc{tropp2012user,
  title={User-friendly tools for random matrices: An introduction},
  author={Tropp, Joel A},
  year={2012}
}

@article{rudi2013sample,
  title={On the sample complexity of subspace learning},
  author={Rudi, Alessandro and Canas, Guillermo D and Rosasco, Lorenzo},
  journal={Advances in Neural Information Processing Systems},
  volume={26},
  year={2013}
}

@article{hermite1878formule,
  title={Sur la formule d'interpolation de Lagrange},
  author={Hermite, M Ch and Borchardt, M},
  journal={Journal f{\"u}r die reine und angewandte Mathematik (Crelles Journal)},
  volume={1878},
  number={84},
  pages={70--79},
  year={1878},
  publisher={De Gruyter}
}

@article{birkhoff1906general,
  title={General mean value and remainder theorems with applications to mechanical differentiation and quadrature},
  author={Birkhoff, George David},
  journal={Transactions of the American Mathematical Society},
  volume={7},
  number={1},
  pages={107--136},
  year={1906},
  publisher={JSTOR}
}

@article{zongmin1992hermite,
  title={Hermite-Birkhoff interpolation of scattered data by radial basis functions},
  author={Zongmin, Wu},
  journal={Approximation Theory and its Applications},
  volume={8},
  number={2},
  pages={1--10},
  year={1992},
  publisher={Springer}
}

@article{kansa1990multiquadrics,
  title={Multiquadrics—A scattered data approximation scheme with applications to computational fluid-dynamics—II solutions to parabolic, hyperbolic and elliptic partial differential equations},
  author={Kansa, Edward J},
  journal={Computers \& mathematics with applications},
  volume={19},
  number={8-9},
  pages={147--161},
  year={1990},
  publisher={Elsevier}
}

@book{engl1996regularization,
  title={Regularization of inverse problems},
  author={Engl, Heinz Werner and Hanke, Martin and Neubauer, Andreas},
  volume={375},
  year={1996},
  publisher={Springer Science \& Business Media}
}

@article{hanke1992regularization,
  title={Regularization with differential operators: an iterative approach},
  author={Hanke, Martin},
  journal={Numerical functional analysis and optimization},
  volume={13},
  number={5-6},
  pages={523--540},
  year={1992},
  publisher={Taylor \& Francis}
}

@article{slepcev2019analysis,
  title={Analysis of p-Laplacian regularization in semisupervised learning},
  author={Slepcev, Dejan and Thorpe, Matthew},
  journal={SIAM Journal on Mathematical Analysis},
  volume={51},
  number={3},
  pages={2085--2120},
  year={2019},
  publisher={SIAM}
}

@inproceedings{zhou2005regularization,
  title={Regularization on discrete spaces},
  author={Zhou, Dengyong and Sch{\"o}lkopf, Bernhard},
  booktitle={Joint Pattern Recognition Symposium},
  pages={361--368},
  year={2005},
  organization={Springer}
}

@inproceedings{zhu2003semi,
  title={Semi-supervised learning using gaussian fields and harmonic functions},
  author={Zhu, Xiaojin and Ghahramani, Zoubin and Lafferty, John D},
  booktitle={Proceedings of the 20th International conference on Machine learning (ICML-03)},
  pages={912--919},
  year={2003}
}

@article{cai2021physics,
  title={Physics-informed neural networks (PINNs) for fluid mechanics: A review},
  author={Cai, Shengze and Mao, Zhiping and Wang, Zhicheng and Yin, Minglang and Karniadakis, George Em},
  journal={Acta Mechanica Sinica},
  volume={37},
  number={12},
  pages={1727--1738},
  year={2021},
  publisher={Springer}
}

@article{rasht2022physics,
  title={Physics-informed neural networks (PINNs) for wave propagation and full waveform inversions},
  author={Rasht-Behesht, Majid and Huber, Christian and Shukla, Khemraj and Karniadakis, George Em},
  journal={Journal of Geophysical Research: Solid Earth},
  volume={127},
  number={5},
  pages={e2021JB023120},
  year={2022},
  publisher={Wiley Online Library}
}

@article{sahli2020physics,
  title={Physics-informed neural networks for cardiac activation mapping},
  author={Sahli Costabal, Francisco and Yang, Yibo and Perdikaris, Paris and Hurtado, Daniel E and Kuhl, Ellen},
  journal={Frontiers in Physics},
  volume={8},
  pages={42},
  year={2020},
  publisher={Frontiers Media SA}
}

@article{rackauckas2020universal,
  title={Universal differential equations for scientific machine learning},
  author={Rackauckas, Christopher and Ma, Yingbo and Martensen, Julius and Warner, Collin and Zubov, Kirill and Supekar, Rohit and Skinner, Dominic and Ramadhan, Ali and Edelman, Alan},
  journal={arXiv preprint arXiv:2001.04385},
  year={2020}
}

@article{rudi2015less,
  title={Less is more: Nystr{\"o}m computational regularization},
  author={Rudi, Alessandro and Camoriano, Raffaello and Rosasco, Lorenzo},
  journal={Advances in neural information processing systems},
  volume={28},
  year={2015}
}

@article{rahimi2007random,
  title={Random features for large-scale kernel machines},
  author={Rahimi, Ali and Recht, Benjamin},
  journal={Advances in neural information processing systems},
  volume={20},
  year={2007}
}

@article{smola98splines,
title = {The connection between regularization operators and support vector kernels},
journal = {Neural Networks},
volume = {11},
number = {4},
year = {1998},
doi = {https://doi.org/10.1016/S0893-6080(98)00032-X},
author = {Alex J. Smola and Bernhard Schölkopf and Klaus-Robert Müller},
}

@article{poggio90splines,
author = {T. Poggio  and F. Girosi },
title = {Regularization Algorithms for Learning That Are Equivalent to Multilayer Networks},
journal = {Science},
volume = {247},
number = {4945},
year = {1990},
doi = {10.1126/science.247.4945.978}
}

@article{owhadi2015bayesian,
  title={Bayesian numerical homogenization},
  author={Owhadi, Houman},
  journal={Multiscale Modeling \& Simulation},
  volume={13},
  number={3},
  pages={812--828},
  year={2015},
  publisher={SIAM}
}

@article{arridge2019solving,
    title={Solving inverse problems using data-driven models},
    author={Arridge, Simon and Maass, Peter and {\"O}ktem, Ozan and Sch{\"o}nlieb, Carola-Bibiane},
    journal={Acta numerica},
    volume={28},
    year={2019},
    publisher={Cambridge University Press}
}

@article{shi2010hermite,
    title={Hermite learning with gradient data},
    author={Shi, Lei and Guo, Xin and Zhou, Ding-Xuan},
    journal={Journal of computational and applied mathematics},
    volume={233},
    number={11},
    pages={3046--3059},
    year={2010},
}

@article{cabannes2021overcoming,
    title={Overcoming the curse of dimensionality with Laplacian regularization in semi-supervised learning},
    author={Cabannes, Vivien and Pillaud-Vivien, Loucas and Bach, Francis and Rudi, Alessandro},
    journal={Advances in Neural Information Processing Systems},
    volume={34},
    pages={30439--30451},
    year={2021}
}

@article{zeinhofer2025unified,
    title={A unified framework for the error analysis of physics-informed neural networks},
    author={Zeinhofer, Marius and Masri, Rami and Mardal, Kent--Andr{\'e}},
    journal={IMA Journal of Numerical Analysis},
    volume={45},
    number={5},
    year={2025},
}

@article{azzimonti2015blood,
  title={Blood flow velocity field estimation via spatial regression with PDE penalization},
  author={Azzimonti, Laura and Sangalli, Laura M and Secchi, Piercesare and Domanin, Maurizio and Nobile, Fabio},
  journal={Journal of the American Statistical Association},
  volume={110},
  number={511},
  pages={1057--1071},
  year={2015},
  publisher={Taylor \& Francis}
}

@article{sangalli2021spatial,
  title={Spatial regression with partial differential equation regularisation},
  author={Sangalli, Laura M},
  journal={International Statistical Review},
  volume={89},
  number={3},
  pages={505--531},
  year={2021},
  publisher={Wiley Online Library}
}

@article{arnone2022some,
  title={Some first results on the consistency of spatial regression with partial differential equation regularization},
  author={Arnone, Eleonora and Kneip, Alois and Nobile, Fabio and Sangalli, Laura M},
  journal={Statistica Sinica},
  volume={32},
  number={1},
  pages={209--238},
  year={2022},
  publisher={JSTOR}
}

@article{zhou2008derivative,
  title={Derivative reproducing properties for kernel methods in learning theory},
  author={Zhou, Ding-Xuan},
  journal={Journal of computational and Applied Mathematics},
  volume={220},
  number={1-2},
  pages={456--463},
  year={2008},
  publisher={Elsevier}
}

@article{cuomo2022scientific,
  title={Scientific machine learning through physics--informed neural networks: Where we are and what’s next},
  author={Cuomo, Salvatore and Di Cola, Vincenzo Schiano and Giampaolo, Fabio and Rozza, Gianluigi and Raissi, Maziar and Piccialli, Francesco},
  journal={Journal of Scientific Computing},
  volume={92},
  number={3},
  year={2022},
}

@article{karniadakis2021physics,
  title={Physics-informed machine learning},
  author={Karniadakis, George Em and Kevrekidis, Ioannis G and Lu, Lu and Perdikaris, Paris and Wang, Sifan and Yang, Liu},
  journal={Nature Reviews Physics},
  volume={3},
  number={6},
  year={2021}
}

@article{rathore2024challenges,
  title={Challenges in training PINNs: A loss landscape perspective},
  author={Rathore, Pratik and Lei, Weimu and Frangella, Zachary and Lu, Lu and Udell, Madeleine},
  journal={arXiv preprint arXiv:2402.01868},
  year={2024}
}

@article{wang2021understanding,
  title={Understanding and mitigating gradient flow pathologies in physics-informed neural networks},
  author={Wang, Sifan and Teng, Yujun and Perdikaris, Paris},
  journal={SIAM Journal on Scientific Computing},
  volume={43},
  number={5},
  pages={A3055--A3081},
  year={2021},
  publisher={SIAM}
}

@article{doumeche2025fast,
  title={Fast kernel methods: Sobolev, physics-informed, and additive models},
  author={Doum{\`e}che, Nathan and Bach, Francis and Biau, G{\'e}rard and Boyer, Claire},
  journal={arXiv preprint arXiv:2509.02649},
  year={2025}
}

@article{doumeche2025physics,
  title={Physics-informed kernel learning},
  author={Doum{\`e}che, Nathan and Bach, Francis and Biau, G{\'e}rard and Boyer, Claire},
  journal={Journal of Machine Learning Research},
  volume={26},
  number={124},
  pages={1--39},
  year={2025}
}

@article{chen2025sparse,
  title={Sparse Cholesky factorization for solving nonlinear PDEs via Gaussian processes},
  author={Chen, Yifan and Owhadi, Houman and Sch{\"a}fer, Florian},
  journal={Mathematics of Computation},
  volume={94},
  number={353},
  pages={1235--1280},
  year={2025}
}

@article{eriksson2018scaling,
  title={Scaling Gaussian process regression with derivatives},
  author={Eriksson, David and Dong, Kun and Lee, Eric and Bindel, David and Wilson, Andrew G},
  journal={Advances in neural information processing systems},
  volume={31},
  year={2018}
}

@article{padidar2021scaling,
  title={Scaling gaussian processes with derivative information using variational inference},
  author={Padidar, Misha and Zhu, Xinran and Huang, Leo and Gardner, Jacob and Bindel, David},
  journal={Advances in Neural Information Processing Systems},
  volume={34},
  pages={6442--6453},
  year={2021}
}

@inproceedings{de2021high,
  title={High-dimensional Gaussian process inference with derivatives},
  author={De Roos, Filip and Gessner, Alexandra and Hennig, Philipp},
  booktitle={International Conference on Machine Learning},
  pages={2535--2545},
  year={2021},
  organization={PMLR}
}

@article{quarteroni25review,
    author = {Quarteroni, Alfio and Gervasio, Paola and Regazzoni, Francesco},
    title = {Combining physics-based and data-driven models: advancing the frontiers of research with scientific machine learning},
    journal = {Mathematical Models and Methods in Applied Sciences},
    volume = {35},
    number = {04},
    pages = {905-1071},
    year = {2025},
    doi = {10.1142/S0218202525500125},
}

@inproceedings{steinwart2009optimal,
  title={Optimal Rates for Regularized Least Squares Regression.},
  author={Steinwart, Ingo and Hush, Don R and Scovel, Clint},
  booktitle={COLT},
  pages={79--93},
  year={2009}
}

@inproceedings{zhang2023optimality,
  title={On the optimality of misspecified kernel ridge regression},
  author={Zhang, Haobo and Li, Yicheng and Lu, Weihao and Lin, Qian},
  booktitle={International Conference on Machine Learning},
  pages={41331--41353},
  year={2023},
  organization={PMLR}
}

@article{lin2020optimal,
  title={Optimal rates for spectral algorithms with least-squares regression over Hilbert spaces},
  author={Lin, Junhong and Rudi, Alessandro and Rosasco, Lorenzo and Cevher, Volkan},
  journal={Applied and Computational Harmonic Analysis},
  volume={48},
  number={3},
  pages={868--890},
  year={2020},
  publisher={Elsevier}
}

@article{wynne2021convergence,
  title={Convergence guarantees for Gaussian process means with misspecified likelihoods and smoothness},
  author={Wynne, George and Briol, Fran{\c{c}}ois-Xavier and Girolami, Mark},
  journal={Journal of Machine Learning Research},
  volume={22},
  number={123},
  pages={1--40},
  year={2021}
}

@article{wang2022gaussian,
  title={Gaussian process regression: Optimality, robustness, and relationship with kernel ridge regression},
  author={Wang, Wenjia and Jing, Bing-Yi},
  journal={Journal of Machine Learning Research},
  volume={23},
  number={193},
  pages={1--67},
  year={2022}
}

@article{schrader2012extended,
  title={An extended error analysis for a meshfree discretization method of Darcy's problem},
  author={Schr{\"a}der, Daniela and Wendland, Holger},
  journal={SIAM Journal on Numerical Analysis},
  volume={50},
  number={2},
  pages={838--857},
  year={2012},
  publisher={SIAM}
}

@article{de2005learning,
  title={Learning from Examples as an Inverse Problem.},
  author={De Vito, Ernesto and Rosasco, Lorenzo and Caponnetto, Andrea and De Giovannini, Umberto and Odone, Francesca and Bartlett, Peter},
  journal={Journal of Machine Learning Research},
  volume={6},
  number={5},
  year={2005}
}

@book{stein1970singular,
  title={Singular integrals and differentiability properties of functions},
  author={Stein, Elias M},
  year={1970},
  publisher={Princeton university press}
}

@article{tocano25pikans,
	author = {Toscano, Juan Diego and Oommen, Vivek and Varghese, Alan John and Zou, Zongren and Ahmadi Daryakenari, Nazanin and Wu, Chenxi and Karniadakis, George Em},
	doi = {10.1007/s44379-025-00015-1},
	journal = {Machine Learning for Computational Science and Engineering},
	number = {1},
	title = {From PINNs to PIKANs: recent advances in physics-informed machine learning},
	volume = {1},
	year = {2025},
}

@inproceedings{zhao2023pinnsformer,
  title={Pinnsformer: A transformer-based framework for physics-informed neural networks},
  author={Zhao, Zhiyuan and Ding, Xueying and Prakash, B Aditya},
  booktitle={International Conference on Learning Representations (ICLR)},
  year={2024}
}

@inproceedings{rahaman2019spectral,
  title={On the spectral bias of neural networks},
  author={Rahaman, Nasim and Baratin, Aristide and Arpit, Devansh and Draxler, Felix and Lin, Min and Hamprecht, Fred and Bengio, Yoshua and Courville, Aaron},
  booktitle={International conference on machine learning},
  year={2019},
  organization={PMLR}
}

@article{kashinath2021physics,
  title={Physics-informed machine learning: case studies for weather and climate modelling},
  author={Kashinath, Karthik and Mustafa, Mustafa and Albert, Adrian and Wu, Jean-Luc and Jiang, C and Esmaeilzadeh, Soheil and Azizzadenesheli, Kamyar and Wang, R and Chattopadhyay, Ashesh and Singh, Aakanksha and others},
  journal={Philosophical Transactions of the Royal Society A},
  volume={379},
  number={2194},
  pages={20200093},
  year={2021},
  publisher={The Royal Society Publishing}
}

@article{kissas2020machine,
  title={Machine learning in cardiovascular flows modeling: Predicting arterial blood pressure from non-invasive 4D flow MRI data using physics-informed neural networks},
  author={Kissas, Georgios and Yang, Yibo and Hwuang, Eileen and Witschey, Walter R and Detre, John A and Perdikaris, Paris},
  journal={Computer methods in applied mechanics and engineering},
  volume={358},
  pages={112623},
  year={2020},
  publisher={Elsevier}
}

@article{kovacs2022conditional,
  title={Conditional physics informed neural networks},
  author={Kovacs, Alexander and Exl, Lukas and Kornell, Alexander and Fischbacher, Johann and Hovorka, Markus and Gusenbauer, Markus and Breth, Leoni and Oezelt, Harald and Yano, Masao and Sakuma, Noritsugu and others},
  journal={Communications in Nonlinear Science and Numerical Simulation},
  volume={104},
  pages={106041},
  year={2022},
  publisher={Elsevier}
}

@article{de2005model,
  title={Model selection for regularized least-squares algorithm in learning theory},
  author={De Vito, Ernesto and Caponnetto, Andrea and Rosasco, Lorenzo},
  journal={Foundations of Computational Mathematics},
  volume={5},
  number={1},
  pages={59--85},
  year={2005},
  publisher={Springer}
}

@article{blanchard2018optimal,
  title={Optimal rates for regularization of statistical inverse learning problems},
  author={Blanchard, Gilles and M{\"u}cke, Nicole},
  journal={Foundations of Computational Mathematics},
  volume={18},
  number={4},
  pages={971--1013},
  year={2018},
  publisher={Springer}
}

@article{smale2007learning,
  title={Learning theory estimates via integral operators and their approximations},
  author={Smale, Steve and Zhou, Ding-Xuan},
  journal={Constructive approximation},
  volume={26},
  number={2},
  pages={153--172},
  year={2007},
  publisher={Springer}
}

\end{document}